\documentclass{article}

\usepackage{microtype}
\usepackage{graphicx}
\usepackage{subfigure,subcaption}
\usepackage{booktabs} % for professional tables

\usepackage{hyperref}

\usepackage[accepted]{icml2025}

\usepackage{amsmath}
\usepackage{amssymb}
\usepackage{mathtools}
\usepackage{amsthm}

\usepackage{pgf}
\usepackage{pgfplots} % often needed if your figure uses pgfplots
\usepackage{tikz}     % optional, but commonly used
\usepackage{tikz-3dplot}

\usepackage[framemethod=TikZ]{mdframed}
\newtheorem{theorem}{Theorem}[section]

\newtheorem{definition}{Definition}[section]
\newtheorem{proposition}{Proposition}[section]

\newtheorem{remark}{Remark}[section]
\newtheorem{example}{Example}[section]

\usepackage{mathrsfs}
\usepackage{latexsym}
\usepackage{crop}
\usepackage{algorithmic,algorithm}
\usepackage{multirow}
\usepackage{bm}
\usepackage{bbm}
\usepackage{enumerate}
\usepackage{url}
\usepackage{array}
\usepackage{paralist}
\usepackage{multicol}
\usepackage{diagbox}
\usepackage[most]{tcolorbox}
\makeatletter
\def\tcb@split@force@last{%
	\tcb@split@setstate@last%
	\ifdim\tcb@h@total>\tcb@h@page\relax%
	\gdef\tcb@after@lastbox{\clearpage}%
	\tcbdimto\kvtcb@bbbottom{\kvtcb@bbbottom+\tcb@h@page-\tcb@h@total}%
	\fi%
}
\makeatother
\usepackage{rotating}

\usepackage[capitalise,noabbrev]{cleveref}
\crefname{equation}{}{}
\crefname{figure}{Figure}{Figures}
\creflabelformat{equation}{\textup{(#2#1#3)}}
\crefname{assumption}{Assumption}{Assumptions}
\crefname{condition}{Condition}{Conditions}
\crefname{definition}{Definition}{Definitions}
\usepackage{arydshln}
\usepackage{enumitem}
\setlist[enumerate]{leftmargin=1.3em, itemindent=0em, itemsep=0em, topsep=0em, label = {\bfseries \arabic*.}} 
\setlist[itemize]{leftmargin=1.3em,itemindent=0em, itemsep=0em, topsep=0em} 

\usepackage{pifont}% http://ctan.org/pkg/pifont
\usepackage{accents}

\newcommand*\tageq{\refstepcounter{equation}\tag{\theequation}}

\usepackage{xspace}

\renewcommand\th{\textsuperscript{th}\xspace}
\usepackage{accents}

\usepackage{stackengine}
\stackMath
\newcommand\tsup[2][2]{%
	\def\useanchorwidth{T}%
	\ifnum#1>1%
	\stackon[-.5pt]{\tsup[\numexpr#1-1\relax]{#2}}{\scriptscriptstyle\sim}%
	\else%
	\stackon[.5pt]{#2}{\scriptscriptstyle\sim}%
	\fi%
}
\makeatletter
\newcommand{\longdash}[1][2em]{%
	\makebox[#1]{$\m@th\smash-\mkern-7mu\cleaders\hbox{$\mkern-2mu\smash-\mkern-2mu$}\hfill\mkern-7mu\smash-$}}
\makeatother
\newcommand{\omitskip}{\kern-\arraycolsep}

\newcommand{\df}{\mathrm{d}}
\newcommand{\real}{\mathbb{R}}

\DeclareMathOperator*{\argmin}{arg\,min}

\newcommand*\indic[1]{\mathbbm{1}_{\left\{#1\right\}}}

\makeatletter
\newcommand*{\transpose}{%
	{\mathpalette\@transpose{}}%
}
\newcommand*{\@transpose}[2]{%
	\raisebox{\depth}{$\m@th#1\intercal$}%
}
\makeatother

\newcommand*{\T}{{\transpose}}
\newcommand{\sB}{\mathcal{B}}
\newcommand{\sC}{\mathcal{C}}

\newcommand{\sS}{\mathcal{S}}
\newcommand{\sT}{\mathcal{T}}
\newcommand{\sL}{\mathcal{L}}
\newcommand{\sN}{\mathcal{N}}

\newcommand {\FF}  { {\mathbf{F}} }

\newcommand {\HH}  { {\mathbf{H}} }

\newcommand {\XX}  { {\mathbf{X}} }

\newcommand {\bb}  { {\bf b} }

\newcommand {\ee}  { {\bf e} }
\newcommand {\ff}  { {\bf f} }

\newcommand {\yy}  { {\bf y} }
\newcommand {\hh}  { {\bf h} }

\newcommand {\rr}  { {\bf r} }

\newcommand {\vv}  { {\bf v} }
\newcommand {\ww}  { {\bf w} }
\newcommand {\xx}  { {\bf x} }

\newcommand {\balpha} {\bm \alpha}

\newcommand {\btheta} {\bm \theta}
\newcommand {\bthetas} {\bm {\theta}^{\star}}
\newcommand {\bthetasS} {\bm {\theta}^{\star}_{\sS}}

\newcommand {\zero}  { {\bf 0} }

\newcommand {\rank}  { {\textnormal{Rank}} }

\newcommand{\hf}{\frac12}

\newcommand{\defeq}{\triangleq}
\definecolor{forestgreen}{rgb}{0.13, 0.55, 0.13}

\definecolor{amber}{rgb}{1.0, 0.75, 0.0}

\definecolor{bananayellow}{rgb}{.8, 0.6, 0}

\pgfplotsset{compat=newest}

\newcounter{comment}

\usepackage{tikz}
\usepackage{xparse}% So that we can have two optional parameters

\NewDocumentCommand\DownArrow{O{2.0ex} O{black}}{%
	\mathrel{\tikz[baseline] \draw [<-, line width=0.5pt, #2] (0,0) -- ++(0,#1);}
}

\usepackage{listings} % to inser code

\definecolor{mygreen}{rgb}{0,0.6,0}
\definecolor{mygray}{rgb}{0.5,0.5,0.5}
\definecolor{mymauve}{rgb}{0.58,0,0.82}
\definecolor{codegreen}{rgb}{0,0.6,0}
\definecolor{codegray}{rgb}{0.5,0.5,0.5}
\definecolor{codepurple}{rgb}{0.58,0,0.82}
\definecolor{backcolour}{rgb}{0.95,0.95,0.92}

\lstdefinestyle{mystyle}{
	backgroundcolor=\color{backcolour},   
	commentstyle=\color{codegreen},
	keywordstyle=\color{magenta},
	numberstyle=\tiny\color{codegray},
	stringstyle=\color{codepurple},
	basicstyle=\ttfamily\footnotesize,
	breakatwhitespace=false,         
	breaklines=true,                 
	captionpos=b,                    
	keepspaces=true,                 
	numbers=left,                    
	numbersep=5pt,                  
	showspaces=false,                
	showstringspaces=false,
	showtabs=false,                  
	tabsize=2
}
\newcommand*\dotprod[1]{\left\langle #1\right\rangle}
\newcommand*\vnorm[1]{\left\| #1\right\|}

\newcommand*\bigO[1]{\mathcal O\left( #1\right)}

\usepackage{dsfont}

\icmltitlerunning{Importance Sampling for Nonlinear Models}

\begin{document}

\twocolumn[
\icmltitle{Importance Sampling for Nonlinear Models}

% It is OKAY to include author information, even for blind
% submissions: the style file will automatically remove it for you
% unless you've provided the [accepted] option to the icml2025
% package.

% List of affiliations: The first argument should be a (short)
% identifier you will use later to specify author affiliations
% Academic affiliations should list Department, University, City, Region, Country
% Industry affiliations should list Company, City, Region, Country

% You can specify symbols, otherwise they are numbered in order.
% Ideally, you should not use this facility. Affiliations will be numbered
% in order of appearance and this is the preferred way.
\icmlsetsymbol{equal}{*}

\begin{icmlauthorlist}
\icmlauthor{Prakash P.\ Rajmohan}{EECS}
\icmlauthor{Fred Roosta}{SMP,CIRES}

\end{icmlauthorlist}

\icmlaffiliation{CIRES}{ARC Training Centre for Information Resilience (CIRES), Brisbane, Australia}
\icmlaffiliation{SMP}{School of Mathematics and Physics, University of Queensland, Brisbane, Australia.}
\icmlaffiliation{EECS}{School of Electrical Engineering and Computer Science, University of Queensland, Brisbane, Australia.}

\icmlcorrespondingauthor{Prakash P.\ Rajmohan}{p.palanivelurajmohan@uq.net.au}
\icmlcorrespondingauthor{Fred Roosta}{fred.roosta@uq.edu.au}

% You may provide any keywords that you
% find helpful for describing your paper; these are used to populate
% the ``keywords'' metadata in the PDF but will not be shown in the document
\icmlkeywords{Importance Sampling, Nonlinear Adjoint, Leverage Scores, Active Learning}

\vskip 0.3in
]

% this must go after the closing bracket ] following \twocolumn[ ...

% This command actually creates the footnote in the first column
% listing the affiliations and the copyright notice.
% The command takes one argument, which is text to display at the start of the footnote.
% The \icmlEqualContribution command is standard text for equal contribution.
% Remove it (just {}) if you do not need this facility.

\printAffiliationsAndNotice{}  % leave blank if no need to mention equal contribution
%\printAffiliationsAndNotice{\icmlEqualContribution} % otherwise use the standard text.

\begin{abstract}
\label{abstract}
While norm-based and leverage-score-based methods have been extensively studied for identifying ``important'' data points in linear models, analogous tools for nonlinear models remain significantly underdeveloped. By introducing the concept of the adjoint operator of a nonlinear map, we address this gap and generalize norm-based and leverage-score-based importance sampling to nonlinear settings.
We demonstrate that sampling based on these generalized notions of norm and leverage scores provides approximation guarantees for the underlying nonlinear mapping, similar to linear subspace embeddings.
As direct applications, these nonlinear scores not only reduce the computational complexity of training nonlinear models by enabling efficient sampling over large datasets but also offer a novel mechanism for model explainability and outlier detection. Our contributions are supported by both theoretical analyses and experimental results across a variety of supervised learning scenarios.
\end{abstract}
\section{Introduction}
\label{Introduction}
The process of training in machine learning (ML) typically boils down to solving an optimization problem of the form
\begin{align}
\label{eq:loss}
\min_{\btheta \in \real^{p}} \left\{\sL(\btheta) = \sum_{i=1}^{n} \ell(f_{i}(\btheta))\right\},
\end{align}
where $\ell$ is a loss function, and $f_{i}:\real^{p} \to \real$ is a potentially nonlinear mapping, parametrized by $\btheta$, that represents the ML model evaluated at the $i\th$ training data point. Throughout this paper, we make the umbrella assumption that $n \geq p$, i.e., we operate in the underparameterized setting. 

For example, in the simple linear regression, $f_{i}(\btheta) = \dotprod{\btheta,\xx_i}-y_i$, where $ (\xx_{i}, y_{i}) $ is the $i\th$ input-output pair. Using the squared loss $\ell(t) = t^{2}$ gives rise to the familiar linear least-squares problem $\sL(\btheta) = \vnorm{\XX \btheta - \yy}^{2}$, where $\XX \in \real^{n \times d}$ is the input data matrix whose $i\th$ row is $\xx_{i}$, and $\yy \in \real^{n}$ is the output vector whose $i\th$ component is $y_i$. 

%Prakash Modification:
% The growth of machine learning (ML) over the past decade has been largely driven by the explosion in the availability of data. However, this exponential increase in data volume presents significant computational challenges, particularly in solving \cref{eq:loss} and developing diagnostic tools for post-training analysis. These challenges have been extensively studied and addressed in simpler linear settings. In this context, randomized numerical linear algebra (RandNLA) has emerged as a powerful paradigm for approximating underlying matrices and accelerating computations. 
% TODO: ``extensively studied''
The growth of ML over the past decade has been largely driven by the explosion in the availability of data. However, this exponential increase in data volume presents significant computational challenges, particularly in solving \cref{eq:loss} and developing diagnostic tools for post-training analysis. These challenges have been extensively studied and addressed in simpler linear settings. In this context, randomized numerical linear algebra (RandNLA) has emerged as a powerful paradigm for approximating underlying matrices and accelerating computations. 
RandNLA has led to significant advancements in algorithms for fundamental matrix problems, including matrix multiplication, least-squares problems, least-absolute deviation, and low-rank matrix approximation, among others. For further reading, see the lecture notes and surveys on this topic, such as \citet{mahoney2011randomized,woodruff2014sketching,drineas2018lectures,martinsson2020randomized,murray2023randomized,derezinski2024recent}. 

Arguably, linear least-squares problems have been among the most well-studied problems in RandNLA \cite{sarlos2006improved,nelson2013osnap,meng2013low,clarkson2017low}. Various randomized approximation techniques, ranging from data-independent oblivious methods (e.g., sketching and projections) to data-dependent non-oblivious sampling methods (e.g., non-uniform leverage score or row-norm sampling), have emerged as powerful tools to speed up computations directly \cite{woodruff2014sketching} or to construct preconditioners for downstream linear algebra subroutines \cite{avron2010blendenpik}. 
While these tools have proven remarkably successful in accelerating computations for linear least-squares problems, extending them to more general problems of the form \cref{eq:loss} remains challenging. Efforts to push these boundaries into the nonlinear realm often remain ad hoc and limited in scope \cite{gajjar2021subspace,erdelyi2020fourier,avron2019universal,avron2017random}. A key gap lies in the lack of a systematic framework to capture and embed the ``nonlinear component'' of the objective function while preserving the critical approximation properties well-established in linear embeddings.

% Much remains unknown about how to construct theoretically grounded importance sampling schemes for broadly used nonlinear models, such as deep neural networks. 

In this paper, we aim to bridge this gap to some extent. Specifically, we focus on non-uniform sampling in \cref{eq:loss}. Letting $\btheta^{\star}$ and $\btheta^{\star}_{\sS}$ denote the optimal parameters obtained from training the model over the full dataset and a, potentially non-uniformly, sampled subset of data, respectively, our goal is to ensure that, for any small $\varepsilon$, the samples are selected such that
\begin{align}
\label{eq:goal}
\sL(\btheta^{\star}_{\sS}) \leq \sL(\btheta^{\star}) + \bigO{\varepsilon}.
\end{align}
Recent work by \cite{gajjar2023active, gajjar2024agnostic} has made progress toward this goal by proposing a one-shot active learning strategy using leverage scores of the data for certain classes of single-neuron predictors. Here, we take a step further by introducing the concept of the adjoint operator of a nonlinear map. This allows us to generalize norm-based and leverage-score-based importance sampling to nonlinear settings, thereby obtaining approximation guarantees of the form \cref{eq:goal} in many settings. 

\textbf{Contributions}. Our contributions are as follows:

\begin{enumerate}
    \item By introducing the nonlinear adjoint operator, we provide a systematic framework for constructing importance sampling scores for \cref{eq:loss}. While numerical integration can generally approximate these scores, we show that for specific nonlinear models, they can be derived directly.
    
    \item We further show that, under certain assumptions, importance sampling based on these scores achieves \cref{eq:goal}. To our knowledge, this is the first work to extend such guarantees to NNs.
    
    \item To validate the theoretical results, we present experiments for several supervised learning tasks. We show that, beyond reducing computational costs, our framework can be used post hoc for diagnostics, such as identifying important samples and detecting anomalies.
\end{enumerate}

\textbf{Notation.} Vectors and matrices are denoted by bold lowercase and bold uppercase letters, respectively. The $i^{th}$ row of a data matrix $\XX \in \mathbb{R}^{n \times d}$ is denoted by $\xx_i \in \mathbb{R}^d$. The pseudoinverse of $\XX$ is denoted by $\XX^\dagger$, and the $i\th$ standard basis vector in $\mathbb{R}^n$ is denoted by $\ee_i$. The spectral norm and Frobenius norm of a matrix are denoted by $\|\cdot\|_{2}$ and $\|\cdot\|_{\mathrm{F}}$, respectively. The inner product between two vectors is denoted by $\langle \cdot, \cdot \rangle$. We use $\mathcal{O}(\cdot)$ for Big-O complexity and $\widetilde{\mathcal{O}}(\cdot)$ to omit logarithmic factors. The sub-sampled data matrix is represented by $\XX_{\sS} \in \mathbb{R}^{s \times d}$, where $s$ is the number of sub-sampled points.

\section{Background and Related Work}

\paragraph{Importance Sampling in Linear Models.}

In linear settings, importance sampling has been widely used for various linear algebra tasks such as matrix multiplication \cite{drineas2006fast}, least-squares regression \cite{drineas2006sampling}, and low-rank approximation \cite{cohen2017input}. It approximates large data matrices by prioritizing data points with high ``information'', reducing computational \cite{clarkson2017low, nelson2013osnap} and storage costs \cite{iwen2021lower, meng2013low} while ensuring strong guarantees for downstream tasks.% such as matrix sketching, subspace embeddings, and least-squares regression \cite{tropp2017practical, woodruff2014sketching, mahoney2011randomized}.

% Prakash modification: Already defined rows and X
% Formally, let $\XX \in \mathbb{R}^{n \times d}$, $n \gg d$, be a data matrix with row vectors $\xx_i \in \real^d$, $i = 1, \dots, n$. A nonnegative weight $q_i \ge 0$ is assigned to each row, calculated through some importance sampling scheme. The probability of sampling the $i\th$ row is then given by $\tau_{i} = {q_i}/{\sum_{j=1}^n q_j}$. 
Formally, let $\XX \in \mathbb{R}^{n \times d}$, $n \gg d$, be a data matrix with row vectors $\xx_i \in \real^d$. Each row is assigned a nonnegative weight $q_i \ge 0$ through an importance sampling scheme, with sampling probability $\tau_{i} = {q_i}/{\sum_{j=1}^n q_j}$. 
Consider selecting $s \ll n$ rows of $\XX$ independently at random, with replacement, according to $\{\tau_{i}\}$. Let $\XX_{\sS} \in \real^{s \times d}$ be the sub-sampled matrix, where the $i\th$ row is ${\xx_{i}}/{\sqrt{s\,\tau_{i}}}$. 
If $\{\tau_{i}\}$ is constructed appropriately, $\XX_{\sS}$ preserves key spectral and geometric properties of $\XX$ with high probability, acting as a lower-dimensional approximation. This is captured by a \textit{linear subspace embedding} guarantee: for all $\vv \in \real^d$,
\begin{align}
(1 - \varepsilon)\,\|\XX \vv\|_2^2 \leq \|\XX_{\sS} \vv\|_2^2 \leq 
(1 + \varepsilon)\,\|\XX \vv\|_2^2. \label{eq:approx_lin}
\end{align}
% Two canonical approaches to constructing $\{\tau_{i}\}$ are based on the \textit{norms} and the \textit{leverage scores} of the rows of $\XX$.
% %
% In row-norm sampling, we typically have $q_i = \|\mathbf{x}_i\|_2^2$, i.e., rows that contribute more to the overall \(\ell_2\)-energy \(\|\mathbf{X}\|_{\mathrm{F}}^2 = \sum_{i=1}^n \|\mathbf{x}_i\|_2^2\) are sampled with higher probability. The leverage score of each row $\xx_i$ measures the influence of $\xx_i$ on the row space of $\mathbf{X}$, and is defined as
% \begin{align*}
% q_i = \dotprod{\ee_{i}, \XX\XX^{\dagger}\ee_{i}} = \min_{\balpha \in \real^{n}} \vnorm{\balpha}_{2}^{2} \text{ subject to } \XX^{\T} \balpha = \xx_{i}.
% \end{align*}
% In other words, the $i\th$ leverage score can be viewed as capturing the ``importance'' of $\xx_i$ in forming the linear subspace spanned by the data points. 
Two common approaches for constructing $\{\tau_{i}\}$ are based on the \textit{norms} and \textit{leverage scores} of the rows of $\XX$. In row-norm sampling, we set $q_i = \|\mathbf{x}_i\|_2^2$, meaning rows that contribute more to the overall $\ell_2$-energy, $\|\mathbf{X}\|_{\mathrm{F}}^2 = \sum_{i=1}^n \|\mathbf{x}_i\|_2^2$, are sampled more frequently. Leverage scores quantify the influence of each row $\xx_i$ on the row space of $\XX$ and are defined as  
\begin{align*}
q_i = \dotprod{\ee_{i}, \XX\XX^{\dagger}\ee_{i}} = \min_{\balpha \in \real^{n}} \vnorm{\balpha}_{2}^{2} \text{ subject to } \XX^{\T} \balpha = \xx_{i}.
\end{align*} 
Thus, the $i\th$ leverage score captures the ``importance'' of $\xx_i$ in spanning the data subspace.

These importance sampling schemes are particularly effective when the data exhibits high coherence \cite{paschou2007pca,mahoney2009cur,gittens2013revisiting,fanhigh,eshragh2022lsar}. It has been well established that, for any $\varepsilon \in (0,1)$, as long as the sample size $s$ is sufficiently large--specifically, $s \in \bigO{{d \,\ln d}/{\varepsilon^2}}$--the approximation in \cref{eq:approx_lin} holds with high probability \cite{martinsson2020randomized}. These bounds are tight (up to logarithmic factors) \cite{woodruff2014sketching,chen2019active}.

% The linear subspace embedding guarantee of the form \cref{eq:approx_lin} is the key property underlying the development of approximations for linear least-squares problems \cite{woodruff2014sketching}. However, going beyond least-squares loss, simple linear subspace embedding may not suffice to achieve the desired guarantees. Instead, various tools such as alternative importance sampling scores (e.g., Lewis weights \cite{apers2024computing,johnson2001finite,bourgain1989approximation,cohen2015lp}) or the concept of coresets \cite{feldman2020introduction,mirzasoleiman2020coresets,lucic2018training,har2004coresets} are often employed. For more general loss functions $\ell$, some authors have studied approximation guarantees similar to \cref{eq:goal}, extending these results—often in an ad hoc manner—to specific problems such as logistic regression \cite{curtin2019coresets,huggins2016coresets,tolochinksy2022generic}, linear predictor models with hinge-like loss \cite{mai2021coresets}, $\ell_{p}$ regression \cite{chen2021query,musco2022active}, and kernel regression \cite{erdelyi2020fourier}.
The linear subspace embedding guarantee in \cref{eq:approx_lin} forms the foundation for approximating linear least-squares problems \cite{woodruff2014sketching}. However, for loss functions beyond least-squares, simple linear subspace embedding may not provide the desired guarantees. In such cases, alternative tools like importance sampling scores (e.g., Lewis weights \cite{apers2024computing,johnson2001finite,bourgain1989approximation,cohen2015lp}) or coresets \cite{feldman2020core,mirzasoleiman2020coresets,lucic2018training,har2004coresets} are commonly used. % feldman2020core and feldman2020introduction are the same

In particular for coresets, \citet{langberg2010universal} introduced a sensitivity sampling framework that provides foundational coreset guarantees for broad classes of objectives. This approach (akin to importance sampling) assigns each data point a sampling probability proportional to its worst-case influence (sensitivity) on the objective function, yielding a coreset that achieves a $(1 \pm \varepsilon)$ approximation for all queries. Subsequent work has refined this idea for specific loss functions; for instance, \citet{tremblay2019determinantal} combine sensitivity scores with determinantal point processes (DPP) to reduce redundant draws and thus promote diversity among the selected points. For logistic regression, \citet{munteanu2018coresets} prove that although no sublinear-size coreset exists in the worst case, a sensitivity-based scheme can yield the first provably sublinear approximation coreset when the dataset's complexity measure $\mu(X)$ is bounded. These results, position importance‑sampling coresets as the natural extension of linear embeddings when moving beyond least‑squares. % \prakash{In particular, coreset methods using sensitivity sampling schemes provide foundational guarantees for general function approximations \cite{langberg2010universal}. These methods, same as importance sampling schemes, assigns to each point a probability mass proportional to each point’s maximal influence on the objective, and yields universal $(1 \pm \varepsilon)$ approximatons. Subsequent work refined this idea for specific losses, namely, \cite{tremblay2019determinantal} combines sensitivity scores with determinantal point processes (DPP) to mitigate redundant draws, hence  promoting diversity among selected points. On the other hand, for logistic regression \cite{munteanu2018coresets} shows that, despite a worst‑case lower bound against sublinear coresets, a sensitivity based scheme produces provably summaries whenever the data complexity measure $\mu(X)$ is bounded. These results, position importance‑sampling coresets as the natural extension of linear embeddings when moving beyond least‑squares.} % A dominant line of coreset research revolves around sensitivity sampling, introduced by \cite{langberg2010universal}, which assigns to each point a probability mass proportional to its worst‑case influence on the objective and yields universal $(1 \pm \varepsilon)$ integral approximators. Subsequent work refined this idea for specific losses, \cite{tremblay2019determinantal} couples sensitivity scores with DPP diversity to mitigate redundant draws, while \cite{munteanu2018coresets} shows that, despite a worst‑case lower bound against sublinear coresets, a sensitivity‑based scheme produces provably small summaries whenever the data complexity measure $\mu(X)$ is bounded. These results, together with sensitivity‑aware Lewis‑weight sampling, position importance‑sampling coresets as the natural extension of linear embeddings when moving beyond least‑squares.

% In particular, coreset methods using sensitivity sampling schemes have been extensively studied, providing foundational guarantees for general function approximations \cite{langberg2010universal}. Sensitivity sampling assigns probabilities proportional to each point’s maximal influence across model parameters, forming the basis for strong approximation bounds. For logistic regression, \cite{munteanu2018coresets} demonstrated conditions under which sensitivity-based coresets overcome inherent complexity barriers, establishing the first provably sublinear coresets in restricted scenarios. Complementarily \cite{tremblay2019determinantal} explored determinantal point processes (DPPs) to enhance coresets by promoting diversity among selected points, improving empirical performance without sacrificing theoretical guarantees.
For more general loss functions $\ell$, some works have extended approximation guarantees similar to \cref{eq:goal} to specific problems, such as logistic regression \cite{samadian2020unconditional,huggins2016coresets,tolochinksy2022generic}, linear predictor models with hinge-like loss \cite{mai2021coresets}, $\ell_{p}$ regression \cite{chen2021query,musco2022active}, and kernel regression \cite{erdelyi2020fourier}. % curtin2019coresets is the same as samadian2020unconditional

\paragraph{Importance Sampling in Nonlinear Models.}
Importance sampling techniques have long played a crucial role in modern large-scale machine learning tasks, serving as tools to identify the most important samples and accelerating optimization procedures \cite{katharopoulos2018not,stich2017safe,nabian2021efficient,canevet2016importance,liu2017black,meng2022count,xu2016sub,liu2024winner}. However, their use in obtaining approximation guarantees of the form \cref{eq:goal} has been very limited. 

To our knowledge, \citet{gajjar2023active, gajjar2024agnostic} are among the first to consider non-linear function families and analyze a leverage-score-based sampling scheme (also referred to as a one-shot active learning sampling scheme) in the context of single-index (or ``single neuron'') models, where $\ell$ is the squared loss and $f_{i}(\btheta) = \phi\bigl(\langle \btheta, \xx_{i} \rangle \bigr)$ for some scalar non-linearity $\phi\colon \real \to \real$. For their derivations, the authors employ intricate tools from high-dimensional probability, as the mapping $\xx \mapsto \phi(\langle \ww, \xx \rangle)$ introduces non-linearity that invalidates the straightforward matrix Chernoff arguments typically used in linear settings.

\citet{gajjar2023active} showed that if $\phi$ is $L$-Lipschitz, one can collect $\widetilde{\mathcal{O}}(d^{2}/\varepsilon^{4})$ labels, sampled with probabilities proportional to the linear leverage scores of $\XX$, to guarantee a solution \(\bthetasS\) satisfying, with high probability, the bound  
\begin{align}
\label{eq:goal_bad}
\sL(\bthetasS) \leq C \cdot \sL(\bthetas) + \bigO{\varepsilon},  
\end{align}  
for some sufficiently large constant \(C > 0\). They also argued that such additive error bounds are necessary, as achieving a pure multiplicative error bound would require exponentially many samples in $d$.  

By using intricate tools from high-dimensional probability, subsequent work by \citet{gajjar2024agnostic} improved the sample complexity to \(\widetilde{\mathcal{O}}(d/\varepsilon^{2})\),  matching the linear benchmark up to polylogarithmic terms. Furthermore, they extended the analysis to the setting where $\phi$ is Lipschitz but unspecified by employing a meticulous ``sampling-aware'' discretization of the function class.

While the results in \citet{gajjar2023active, gajjar2024agnostic} represent significant contributions toward achieving \cref{eq:goal} for nonlinear models, they have a notable drawback: the guarantees in \citet{gajjar2023active, gajjar2024agnostic} take the form \cref{eq:goal_bad}, which involves an undesirable constant $C \gg 1$ (exceeding 1,000 in \citet{gajjar2024agnostic}). Such large constants undermine the practical utility of the approximation guarantee in \cref{eq:goal_bad}. We show that, under certain assumptions, it is possible to obtain the more desirable guarantee \cref{eq:goal} (where $C = 1$).

\section{Nonlinear Approximation: Our Approach}
\label{sec:imp_samp_concepts}
% Prakash modification:
% The underlying structure that enables the use of the subspace embedding property \cref{eq:approx_lin} in approximating linear least-squares problems is the linear \textit{inner product}. Specifically, for $f(\btheta) = \langle \btheta, \xx \rangle - y$, the least-squares model can be expressed as 
The core structure enabling the use of the subspace embedding property \cref{eq:approx_lin} in approximating linear least-squares problems is the linear \textit{inner product}. Specifically, for $f(\btheta) = \langle \btheta, \xx \rangle - y$, the least-squares model can be expressed as 
\begin{align*}
    \sum_{i=1}^{n} f^{2}_{i}(\btheta) = \sum_{i=1}^{n} \left(\langle \btheta, \xx_{i} \rangle - y_{i}\right)^{2} = \langle \widehat{\XX}\widehat{\btheta}, \widehat{\XX}\widehat{\btheta} \rangle,
\end{align*}
where $\widehat{\XX} = [\XX \mid -\yy] \in \real^{n \times (d+1)}$ and $\widehat{\btheta} = [\btheta; 1] \in \real^{d+1}$. This structure allows the subspace embedding property \cref{eq:approx_lin} to be applied, enabling the approximation of the original least-squares objective using a carefully chosen subsample of the data. 

Using the notion of the \textit{nonlinear adjoint operator}, we now derive an analogous ``inner-product-like'' representation for nonlinear mappings. This representation captures the nonlinearity in a novel way, enabling the derivation of nonlinear embeddings similar to \cref{eq:approx_lin}.

\subsection{Nonlinear Adjoint Operator}
Consider a mapping $ f(.):\real^{p} \to \real $ that is continuously differentiable. By the integral form of the mean value theorem,
\begin{align*}
	f(\btheta) &= f(\zero) + \int_{0}^{1}  \dotprod{\frac{\partial}{\partial \btheta} f(t\btheta) ,\btheta} \df  t \\
    &= f(\zero) + \dotprod{\btheta, \displaystyle \int_{0}^{1} \frac{\partial}{\partial \btheta} f(t\btheta) \df t}.
\end{align*} 
% \begin{definition}[Nonlinear Adjoint Operator]
% \label{def:naop}
% The nonlinear adjoint operator of $f:\real^{p} \to \real$ is the vector values map $ \ff^{\star}: \real^{p}  \to \real^{p} $ defined as 
% \begin{align}
% 	\label{eq:f*}
% 	\ff^{\star}(\btheta) \defeq \int_{0}^{1} \frac{\partial}{\partial \btheta} f(t\btheta) \df t.
% \end{align} 
% \end{definition}
The above relation is an equation involving functions of $\btheta$. The second term includes an inner product between the parameter $\btheta$ and a vector-valued map, reminiscent of the role of dual operators in Banach spaces. This observation motivates the following definition, which extends beyond continuously differentiable functions.
\begin{definition}[Adjoint Operator]
\label{def:naop}
The adjoint operator of $f:\real^{p} \to \real$ is the vector valued map $ \ff^{\star}: \real^{p}  \to \real^{p} $, 
\begin{align}
	\label{eq:f*}
	\ff^{\star}(\btheta) \defeq \int_{0}^{1} \frac{\partial}{\partial \btheta} f(t\btheta) \df t,
\end{align} 
for all $\btheta \in \real^{p}$ such that the function $\phi:[0,1] \to \real$ defined as $\phi(t) \defeq f(t\btheta)$ is absolutely continuous on $[0,1]$.
\end{definition}
Note that the absolute continuity condition relaxes the requirement for differentiability to that of being differentiable almost everywhere.

With the above definition, we can more compactly write
\begin{subequations}
\label{eq:f*_hat}
\begin{align}
    f(\btheta) &= 
%     \dotprod{\begin{bmatrix}
% 			\btheta \\ 1
% 		\end{bmatrix}, \begin{bmatrix}
% 			\displaystyle \int_{0}^{1} \frac{\partial}{\partial \btheta} f(t\btheta) \df t \\ \\
% 			f(\zero)
% 	\end{bmatrix}} \\ 
%     &= \dotprod{\begin{bmatrix}
% 	\btheta \\ 1
% \end{bmatrix}, \begin{bmatrix}
% \ff^{\star}(\btheta) \\ \\
% f(\zero)
% \end{bmatrix}}  = 
\dotprod{\widehat{\btheta}, \widehat{\ff}^{\star}(\btheta)}, \label{eq:f*_hat_dotproduct}
\end{align} 
where 
\begin{align}
	{\widehat{\ff}}^{\star}(\btheta) \defeq \begin{bmatrix}
		\ff^{\star}(\btheta) \\ 
		f(\zero)
	\end{bmatrix}, \quad \text{and} \quad \widehat{\btheta} \defeq \begin{bmatrix}
		\btheta \\ 1
	\end{bmatrix}. \label{eq:f*_hat_f* and theta}
\end{align}
\end{subequations}
The term nonlinear adjoint operator is motivated by cases where $ f(\zero) = 0 $. In such cases, the above expression simplifies to $f(\btheta) = \dotprod{\btheta, \ff^{\star}(\btheta)}$, which can be viewed as a nonlinear analogue of the Riesz representation theorem in linear settings; see, for example,  \citet{buryvskova1998some,scherpen2002nonlinear}. %Nevertheless, it has been a challenge to extend the ``inner-product'' like representation of the linear operator to a more generalized nonlinear form or non-Hilbert settings. A very few works such as \cite{buryvskova1998some} extends to nonlinear models in the Banach Space and recently \cite{jafari2024impressive} used the Fréchet derivatives for defining adjoints of differentiable nonlinear operators.

In the simplest case of linear regression, i.e., $f(\btheta) = \langle \btheta, \xx \rangle - y$, we have $\ff^{\star}(\btheta) = \xx$, that is, the input data constitutes the space of adjoint operators. Beyond simple linear settings, the explicit calculation of \cref{eq:f*} involves evaluating an integral. Naturally, this integral can be approximated using numerical methods, such as quadrature schemes. 
Fortunately, for many machine learning models, the following property enables the direct computation of \cref{eq:f*}.
\begin{proposition}
    \label{prop:f*_composit}
    Let $ f = g \circ h = g(h)$ where $ g:\real \to \real $ and $h:\real^{p} \to \real $. Also, assume that $ h $ is positively homogeneous of degree $ \alpha \in \real $, i.e., $ h(t \btheta) = t^{\alpha} h(\btheta)  $ for any $ t > 0 $. Then 
        \begin{align*}
            \ff^{\star}(\btheta) = \left\{\begin{array}{ll} 
            \displaystyle \left(\frac{g(h(\btheta)) - g(0)}{\alpha \left(h(\btheta)\right)} \right)\frac{\partial}{\partial \btheta} h( \btheta), & \text{if } h(\btheta) \neq 0, \\ \\
            \displaystyle \left(\frac{g^{\prime}(0)}{\alpha} \right)\frac{\partial}{\partial \btheta} h(\btheta), & \text{if } h(\btheta) = 0.
        \end{array} \right.
	\end{align*}
\end{proposition}
\begin{proof}
The proof can be found in \cref{sec:appendix:proofs}. %\let\qed\relax
\end{proof}

\begin{example}[Generalized Linear Predictors]
\label{ex:glm}
Consider generalized linear predictor models, often called ``single index'' models, where $f(\btheta) = \phi(\langle \btheta, \xx \rangle)$ for some function $\phi: \real \to \real$ and data point $\xx \in \real^{d}$. These models are foundational in studying complex nonlinear predictors in high dimensions \cite{hristache2001direct, hardle2004nonparametric, kakade2011efficient, bietti2022learning}. By limiting networks to a single hidden unit, they provide a controllable instance of neural network architecture while retaining simplicity \cite{pmlr-v65-goel17a}. In scientific machine learning, such models underpin efficient surrogate methods for parametric partial differential equations, where the cost of labeled queries can be prohibitive \cite{cohen2015approximation, o2022derivative}.

Note that here $f(\btheta) = g(h(\btheta))$ where $g(t) = \phi(t)$ and $h(\btheta) = \dotprod{\btheta,\xx}$. Since the positive homogeneity degree of $h$ is $ \alpha = 1 $, we have 
% \vspace{-2mm}
\begin{align}
    \label{eq:glm_f*}
    \ff^{\star}(\btheta) %&= \int_{0}^{1} \frac{\partial}{\partial \btheta} f(t \btheta, 	\xx) \df t = \left(\int_{0}^{1} \phi^{\prime}(\dotprod{t\btheta,\xx}) \df t\right)  \xx 
    = \left(\frac{\phi(\dotprod{\btheta,\xx}) - \phi(0)}{\dotprod{\btheta,\xx}}\right) \xx.
\end{align}
We also get $\,\widehat{\ff}^{\star}(\btheta) = \begin{bmatrix}
        \ff^{\star}(\btheta) \\ \phi(0)
    \end{bmatrix}$.
% \begin{align*}
%     \widehat{\ff}^{\star}(\btheta) = \begin{bmatrix}
%         \ff^{\star}(\btheta) \\ \phi(0)
%     \end{bmatrix}.
% \end{align*}
A notable example is the logistic function where $ \phi(t) = 1/(1+\exp(-t))$, where   
% \vspace{-2mm}
\begin{align*}
    \ff^{\star}(\btheta) %&= \int_{0}^{1} \frac{\partial}{\partial \btheta} f(t \btheta, 	\xx) \df t = - \left(\int_{0}^{1} \frac{e^{-\dotprod{\btheta,\xx}t}}{\left(1 + e^{-\dotprod{\btheta,\xx}t}\right)^{2}} \df t\right)  \xx
    = - \left(\frac{\tanh(\dotprod{\btheta,\xx}/2)}{2 \dotprod{\btheta,\xx}}\right) \xx.
\end{align*}
\end{example}
        
\begin{example}[ReLU Neural Networks]
\label{ex:nn}
Suppose $ r(\btheta) = \phi(\psi(\btheta))$ where $\phi:\real \to \real$, and $ \psi(\btheta) \defeq a \cdot  \max\{\dotprod{\bb,\xx},0\})$ for $\xx \in \real^{d}$ and $ \btheta = [a,\bb] \in \real^{d + 1}$. Here, the positive homogeneity degree of $\psi$ is $ \alpha = 2 $. Hence, 
% \vspace{-2mm}
\begin{align}
    \label{eq:relu}
    &\rr^{\star}(\btheta) = \int_{0}^{1} \frac{\partial}{\partial \btheta} \phi(t a \cdot   \max\left\{\dotprod{t \bb,\xx},0\right\}) \; \df t \\%= \bm{\phi}^{\star}(a \cdot   \max\left\{\dotprod{\bb,\xx},0\right\}) \\
    &= \frac{\phi(a \cdot   \max\left\{\dotprod{\bb,\xx},0\right\}) - \phi(0)}{2 a \cdot \max\left\{\dotprod{\bb,\xx},0\right\}} \begin{bmatrix}
        \max\left\{\dotprod{\bb,\xx},0\right\} \\ \\
        a \cdot \xx \cdot \indic{\dotprod{\bb,\xx} > 0}
    \end{bmatrix}. \nonumber
\end{align}
In the typical case where $ \phi(z) = z $, this simplifies to 
\begin{align*}
    \rr^{\star}(\btheta)= \hf \begin{bmatrix}
        \max\left\{\dotprod{\bb,\xx},0\right\} \\
        a \cdot  \xx \indic{\dotprod{\bb,\xx} > 0}
    \end{bmatrix}. 
\end{align*}
Now, consider a two-layer neural network with $m$ hidden neurons, ReLU activation, and a single output:
\begin{align*}
    f(\btheta) = \sum_{j=1}^{m} r(\btheta_{j}) = \sum_{j=1}^{m} \phi(a_{j} \cdot   \max\left\{\dotprod{\bb_{j},\xx},0\right\}),
\end{align*}
where $ \btheta_{j} = [a_{j},\bb_{j}] $ and $ \btheta = [\btheta_{1},\ldots,\btheta_m] $. 
%An example with $ d=6 $, $m=4$ is depicted in \cref{fig:nn_single} in \cref{sec:appendix:figure_nn}. 
 We have 
\begin{align*}
    \ff^{\star}(\btheta) &= \sum_{j=1}^{m} \int_{0}^{1} \frac{\partial}{\partial \btheta} \phi(t a_{j} \cdot   \max\left\{\dotprod{t \bb_{j},\xx},0\right\}) \df t \\
    & = \sum_{j=1}^{m} \ee_{j} \otimes \rr^{\star}(\btheta_{j}) \\
    &= \begin{bmatrix}
        [\rr^{\star}(\btheta_{1})]^{\T} & [\rr^{\star}(\btheta_{2})]^{\T} & \ldots & [\rr^{\star}(\btheta_{m})]^{\T},
    \end{bmatrix}^{\T}%\\
    % &= \begin{bmatrix}
    %     \rr^{\star}(\btheta_{1}) \\ \rr^{\star}(\btheta_{2}) \\ \vdots \\\rr^{\star}(\btheta_{m})
    % \end{bmatrix},
\end{align*}
where $ \otimes $ denotes the Kronecker product and $ \rr^{\star}(\btheta_{j}) $ is given in \cref{eq:relu}. 
Finally, we get $\, \widehat{\ff}^{\star}(\btheta) = \begin{bmatrix}
        \ff^{\star}(\btheta) \\ m \cdot \phi(0)
    \end{bmatrix}$.
% \begin{align*}
%     \widehat{\ff}^{\star}(\btheta) = \begin{bmatrix}
%         \ff^{\star}(\btheta) \\ m \cdot \phi(0)
%     \end{bmatrix}.
% \end{align*}
\end{example}

\subsection{Parameter-dependent Approximations}
\label{sec:sampling_nls}
The ``inner-product-like'' representation of a nonlinear function in \cref{eq:f*_hat} enables the development of sampling strategies for approximating general loss functions beyond linear least-squares and extends the concept of linear subspace embeddings to nonlinear settings. We focus on nonlinear least-squares problems, where $\ell(t) = t^2$, and \cref{eq:loss} is expressed as $\sL(\btheta) = \sum_{i=1}^{n} \left( f_{i}(\btheta) \right)^{2}$, with extensions to more general nonlinear losses discussed in \cref{sec:appendix:general}.

\begin{definition}[Nonlinear Dual Matrix]
\label{def:dual_X}
    Given the nonlinear maps $f_{i}: \real^{p} \to \real$, $i=1,\ldots,n$, the nonlinear dual matrix operator, $ \FF^{\star}:\real^{p} \to \real^{n \times p}$, is defined as
\begin{align*}
	\FF^{\star}(\btheta) \defeq \begin{bmatrix}
		\ff^{\star}_{1}(\btheta) & \ff^{\star}_{2}(\btheta) & \ldots &  \ff^{\star}_{n}(\btheta)
	\end{bmatrix}^{\T}.
\end{align*}
Using \cref{eq:f*_hat_f* and theta}, we also define
\begin{align*}
	\widehat{\FF}^{\star}(\btheta) \defeq \begin{bmatrix}
		\widehat{\ff}^{\star}_{1}(\btheta) & \widehat{\ff}^{\star}_{2}(\btheta) & \ldots &  \widehat{\ff}^{\star}_{n}(\btheta)
	\end{bmatrix}^{\T}.
\end{align*}
\end{definition}
With \cref{eq:f*_hat,def:dual_X}, we can now write
\begin{align*}
	\sL(\btheta) \hspace{-0.5mm} = \hspace{-0.5mm} \sum_{i=1}^{n} \left( f_{i}(\btheta) \right)^{2} \hspace{-0.5mm} = \hspace{-0.5mm} \sum_{i=1}^{n} \langle \widehat{\btheta}, \widehat{\ff}^{\star}_{i}(\btheta) \rangle^{2} \hspace{-0.5mm} = \hspace{-0.5mm} \vnorm{\widehat{\FF}^{\star}(\btheta) \widehat{\btheta}}^{2}. \tageq\label{eq:nls}
 %    \\
	% &= \vnorm{\widehat{\FF}^{\star}(\btheta) \widehat{\btheta}}^{2} = \vnorm{\FF^{\star}(\btheta) \btheta + \ff(\zero)}^{2}. \tageq\label{eq:nls}
\end{align*}
% \begin{align*}
% 	&\sL(\btheta) = \sum_{i=1}^{n} \left( f_{i}(\btheta) \right)^{2} = \sum_{i=1}^{n} \dotprod{\widehat{\btheta}, \widehat{\ff}^{\star}_{i}(\btheta)}^{2} \\
% 	% &= \sum_{i=1}^{n} \dotprod{\widehat{\btheta}, \left( \widehat{\ff}^{\star}_{i}(\btheta)\left[\widehat{\ff}^{\star}_{i}(\btheta)\right]^{\T}\right) \widehat{\btheta}} \\
%     &= \dotprod{\widehat{\btheta}, \left(\left[\widehat{\FF}^{\star}(\btheta)\right]^{\T}\widehat{\FF}^{\star}(\btheta)\right) \widehat{\btheta}} = \vnorm{\widehat{\FF}^{\star}(\btheta) \widehat{\btheta}}^{2}. \tageq\label{eq:nls}
%  %    \\
% 	% &= \vnorm{\widehat{\FF}^{\star}(\btheta) \widehat{\btheta}}^{2} = \vnorm{\FF^{\star}(\btheta) \btheta + \ff(\zero)}^{2}. \tageq\label{eq:nls}
% \end{align*}
This formulation resembles the ordinary least-squares setup, with the nonlinear dual matrix $\widehat{\FF}^{\star}(\btheta)$ replacing the original input data matrix. This motivates the following notions of importance sampling scores, which generalize standard leverage scores or row-norms to those based on the rows of the dual matrix $\widehat{\FF}^{\star}(\btheta)$.

\begin{definition}[Nonlinear Leverage Scores]
\label{def:lev_score}
Nonlinear leverage score of $f_{i}$ is defined as 
\begin{align*}
	\tau_{i}(\btheta) %&= \dotprod{\widehat{\ff}_{i}^{\star}_{i}(\btheta),\left(\left[\widehat{\FF}^{\star}(\btheta)\right]^{\T}\widehat{\FF}^{\star}(\btheta)\right)^{-1}\widehat{\ff}_{i}^{\star}_{i}(\btheta)} \\
    &\defeq \frac{\dotprod{\ee_{i},\widehat{\FF}^{\star}(\btheta) \left[\widehat{\FF}^{\star}(\btheta)\right]^{\dagger} \ee_{i}}}{\rank\left(\widehat{\FF}^{\star}(\btheta)\right)}. 
	% & = \min_{\bomega \in \real^{n}}\; \vnorm{\bomega}^{2} \;\;  \text{s.t.}\;\; \left[\widehat{\FF}^{\star}(\btheta)\right]^{\T} \bomega = \widehat{\ff}_{i}^{\star}_{i}(\btheta).
\end{align*}	
\end{definition}
\begin{definition}[Nonlinear Norm Scores]
\label{def:row_norm}
Nonlinear norm score of $f_{i}$  is defined as 
\begin{align*}
	\tau_{i}(\btheta) &\defeq \frac{\vnorm{\widehat{\ff}^{\star}_{i}(\btheta)}_{2}^{2}}{\vnorm{\widehat{\FF}^{\star}(\btheta)}^{2}_{\textnormal{F}}}.
\end{align*}	
\end{definition}
It is easy to verify that, for both \cref{def:lev_score,def:row_norm}, the values $\{\tau_{i}(\btheta)\}_{i=1}^{n}$ form a probability distribution, i.e., $\tau_{i}(\btheta) \geq 0$ and $\sum_{i} \tau_{i}(\btheta) = 1$.

\begin{remark}
     \cref{def:lev_score,def:row_norm} generalize their linear counterparts. Specifically, for linear models where $f(\btheta) = \langle \btheta, \xx \rangle$, these definitions reduce to the standard linear leverage and row norm scores. To the best of our knowledge, this represents the first systematic extension of these concepts from linear to nonlinear settings.
\end{remark} 

% Let $\sS $ the index set of the $s$ samples by randomly sampling, with replacement, $s \leq n$ data points, with probabilities determined by either of the sampling distributions in \cref{def:lev_score,def:row_norm}. Note that $\sS$  depends on the choice of $\btheta$ (this will be addressed later in \cref{sec:estimating_scores}). The loss on the sampled data is then defined as 
% \begin{align*} 
% \sL_{\sS}(\btheta) \defeq \sum_{i \in \sS} \frac{\left( f_{i}(\btheta) \right)^{2}}{s \tau_{i}(\btheta)}, 
% \end{align*}
% which constitutes an unbiased estimator of the true loss. Similar to the derivation of \cref{eq:nls}, it follows that $\sL_{\sS}(\btheta) = \|\widehat{\FF}^{\star}_{\sS}(\btheta) \widehat{\btheta}\|^{2}$ where $\widehat{\FF}^{\star}_{\sS}(\btheta) \in \mathbb{R}^{s \times p}$ represents a subset of $\widehat{\FF}^{\star}(\btheta)$, consisting of rows indexed by $\sS$, with the $i\th$ row scaled by $\sqrt{s \tau_{i}(\btheta)}$. 

\begin{remark}
\label{sec:adjoint_def_remark}
While \cref{ex:glm} and \cref{ex:nn} illustrate explicit computations of the adjoint operator via \cref{prop:f*_composit}, our approach conceptually extends to broader classes of nonlinear functions through its general \cref{def:naop}. When an explicit adjoint form is available, importance scores (row-norm or leverage) are computed directly from the nonlinear dual matrix via standard linear algebra routines (e.g., QR or SVD factorizations). In scenarios where the adjoint lacks a closed-form expression, we can approximate it numerically using \cref{def:naop}; the resulting dual matrix is then factorized similarly, with QR or SVD methods.
\end{remark}

Let $\sS$ be the index set of $s$ samples, obtained by randomly sampling $s \leq n$ data points with replacement, using probabilities defined by one of the sampling distributions in \cref{def:lev_score,def:row_norm}. Note that $\sS$ depends on the choice of $\btheta$ (this is addressed later in \cref{sec:estimating_scores}). The loss on the sampled data is defined as  
\begin{align*}  
\sL_{\sS}(\btheta) &\defeq \sum_{i \in \sS} \frac{\left( f_{i}(\btheta) \right)^{2}}{s \tau_{i}(\btheta)},  
\end{align*}  
which is an unbiased estimator of the true loss. Similar to the derivation of \cref{eq:nls}, we have  
\begin{align*}  
\sL_{\sS}(\btheta) &= \|\widehat{\FF}^{\star}_{\sS}(\btheta) \widehat{\btheta}\|^{2},  
\end{align*}  
where $\widehat{\FF}^{\star}_{\sS}(\btheta) \in \mathbb{R}^{s \times p}$ is a subset of $\widehat{\FF}^{\star}(\btheta)$, consisting of rows indexed by $\sS$, with the $i$th row scaled by $1/\sqrt{s \tau_{i}(\btheta)}$.

The representation in \cref{eq:nls}, which related the output of a nonlinear function to the norm of a matrix-vector product, enables the use of existing results on approximate matrix multiplication and linear subspace embedding to derive an approximation bound on $\sL_{\sS}(\btheta)$; see, for example, \citet{woodruff2014sketching, drineas2018lectures, martinsson2020randomized}. Specifically, for a fixed $\btheta$ and any of the sampling distributions mentioned, if the sample size is $\bigO{p \log(p/\delta)/\varepsilon^{2}}$, then with probability at least $1-\delta$, for all $\vv \in \real^{p+1}$, $(1-\varepsilon) \| \widehat{\FF}^{\star}(\btheta) \vv\|^{2} \! \leq \!  
    \|\widehat{\FF}^{\star}_{\sS}(\btheta) \vv\|^{2} \! \leq \! (1+\varepsilon) \| \widehat{\FF}^{\star}(\btheta) \vv\|^{2}$.  
% $\vv \in \real^{p+1}$, $(1-\varepsilon) \| \widehat{\FF}^{\star}(\btheta) \vv^{2} \! \leq \!  
%     \|\widehat{\FF}^{\star}_{\sS}(\btheta) \vv\|^{2} \! \leq \! (1+\varepsilon) \| \widehat{\FF}^{\star}(\btheta) \vv\|^{2}$.  
% PrakashNew{The above had a norm missing closing bracket}
% \begin{align*}
%     (1-\varepsilon) \vnorm{\widehat{\FF}^{\star}(\btheta) \vv}^{2} \! \leq \!  
%     \vnorm{\widehat{\FF}^{\star}_{\sS}(\btheta) \vv}^{2} \! \leq \! (1+\varepsilon) \vnorm{\widehat{\FF}^{\star}(\btheta) \vv}^{2}.
% \end{align*}
% \begin{align*}
%     (1-\varepsilon) \left[\widehat{\FF}^{\star}(\btheta)\right]^{\T}\widehat{\FF}^{\star}(\btheta) &\preceq 
%     \left[\widehat{\FF}^{\star}_{\sS}(\btheta)\right]^{\T}\widehat{\FF}^{\star}_{\sS}(\btheta) \\
%     &\preceq (1 + \varepsilon) \left[\widehat{\FF}^{\star}(\btheta)\right]^{\T}\widehat{\FF}^{\star}(\btheta),
% \end{align*}
which in turn implies, with probability at least $1-\delta$,
\begin{align}
    \label{eq:approx}
	(1-\varepsilon) \sL(\btheta) \leq \sL_{\sS}(\btheta) \leq (1+\varepsilon)\sL(\btheta).
\end{align}

\subsection{Parameter-independent Approximation}
% The relation \cref{eq:approx} and the optimality of $\btheta^{\star}_{\sS}$ with respect to $\sL_{\sS}(.)$ allow us to take initial steps towards \cref{eq:goal} as
% \begin{align}
%     \label{eq:param_indep_01}
%     \sL_{\sS}(\btheta^{\star}_{\sS}) \leq \sL_{\sS}(\btheta^{\star}) \leq  (1+\varepsilon)\sL(\btheta^{\star}).
% \end{align}
% However, the key missing component is that the left-hand side of \cref{eq:goal} involves the loss evaluated over the full dataset, whereas the left-hand side of the inequality above is restricted to the loss over the sampled subset. The chain of inequalities to obtain \cref{eq:goal} would be complete if we could find a sensible lower bound of $\sL_{\sS}(\btheta^{\star}_{\sS})$ in terms of $\sL(\btheta^{\star}_{\sS})$. Even so, \cref{eq:param_indep_01} requires evaluating the importance sampling scores defined in \cref{def:lev_score,def:row_norm} at $\bthetas$, which is inherently unknown. Therefore, to achieve the goal of obtaining \cref{eq:goal}, we need to address two critical challenges:
% \begin{enumerate}
% \item Approximating the nonlinear scores defined in \cref{def:lev_score,def:row_norm} with quantities that are independent of the parameter $\btheta$.
% \item Establishing a meaningful lower bound for $\sL_{\sS}(\btheta^{\star}_{\sS})$ in terms of $\sL(\btheta^{\star}_{\sS})$.
% \end{enumerate}

% Our next tasks involve addressing these challenges. 
The relation \cref{eq:approx} and the optimality of $\btheta^{\star}_{\sS}$ with respect to $\sL_{\sS}(.)$ allow us to take initial steps towards \cref{eq:goal} as
\begin{align}
    \label{eq:param_indep_01}
    \sL_{\sS}(\btheta^{\star}_{\sS}) \leq \sL_{\sS}(\btheta^{\star}) \leq  (1+\varepsilon)\sL(\btheta^{\star}).
\end{align}
However, the key missing component is that the left-hand side of \cref{eq:goal} involves the loss evaluated over the full dataset, whereas the left-hand side of \cref{eq:param_indep_01} is restricted to the loss over the sampled subset. The chain of inequalities to obtain \cref{eq:goal} would be complete if we could find a sensible lower bound for $\sL_{\sS}(\btheta^{\star}_{\sS})$ in terms of $\sL(\btheta^{\star}_{\sS})$. Additionally, \cref{eq:param_indep_01} requires evaluating the importance sampling scores defined in \cref{def:lev_score,def:row_norm} at $\bthetas$, which is inherently unknown. Therefore, to achieve \cref{eq:goal}, we must address two critical challenges:

\begin{enumerate}
    \item Approximating the nonlinear scores defined in \cref{def:lev_score,def:row_norm} with quantities that are independent of the parameter $\btheta$.
    \item Establishing a meaningful lower bound for $\sL_{\sS}(\btheta^{\star}_{\sS})$ in terms of $\sL(\btheta^{\star}_{\sS})$.
\end{enumerate}

Our next tasks involve addressing these challenges.

\subsubsection{Estimating Sampling Scores}
\label{sec:estimating_scores}
% The relation \cref{eq:param_indep_01} necessitates that \cref{eq:approx} holds for $\bthetas$, which involves calculating the scores defined in \cref{def:lev_score,def:row_norm} for $\bthetas$. However, this is inherently infeasible, as $\bthetas$ is unknown.  
% %
% A possible remedy lies in approximating these nonlinear scores with ones that do not depend on $\btheta$. 

% The idea of sampling with near-optimal probabilities has been extensively studied in the context of randomized numerical linear algebra. More precisely, as highlighted in \citet{drineas2018lectures,woodruff2014sketching}, even when leverage scores or row norms are estimated inaccurately by a factor $0 < \beta \leq 1$ expressed as $\beta \tau_{i} \leq \hat{\tau}_{i}$ for $i=1,2,\ldots,n$, it is still possible to achieve the subspace embedding property \cref{eq:approx_lin} with a sample size of at least $\bigO{p \log(p/\delta)/(\beta \varepsilon^{2})}$. In this light, if the nonlinear scores can be approximated such that $\beta \tau_{i}(\btheta) \leq   \tau_{i}$ for some $\beta \in (0,1]$, we can sample according to  $\tau_{i}$ and still obtain a guarantee of the form \cref{eq:approx} for any $\btheta$. 

% Within the context of the specific models considered in \cref{ex:glm,ex:nn}, we demonstrate that leveraging the particular structure of the nonlinear adjoint operators enables the estimation of the nonlinear scores.

The relation in \cref{eq:param_indep_01} requires calculating the scores from \cref{def:lev_score,def:row_norm} for $\bthetas$, which is infeasible since $\bthetas$ is unknown. A solution is to approximate these nonlinear scores with ones independent of $\btheta$.  

Sampling with near-optimal probabilities has been studied in RandNLA \citep{drineas2018lectures,woodruff2014sketching}, where it is known that if leverage or row-norm scores, $\{\tau_i\}$, are approximated by $\{\hat{\tau}_i\}$ such that $\beta \tau_i \leq \hat{\tau}_i$, the subspace embedding property \cref{eq:approx_lin} holds with sample size $\mathcal{O}(p \log(p/\delta)/(\beta \varepsilon^2))$. If nonlinear scores can be similarly approximated in a manner independent of $\btheta$, we can sample according to the approximate scores and still obtain a guarantee of the form \cref{eq:approx} for any $\bthetas$.  

Within the context of the specific models considered in \cref{ex:glm,ex:nn}, we demonstrate that leveraging the structure of the nonlinear adjoint operators enables the estimation of the nonlinear scores.

\begin{example}[Generalized Linear Predictors]
\label{ex:glm_estimate}
% Consider the class generalized linear predictor models from \cref{ex:glm} with  any activation function $\phi$ for which there exists $0 < l < u < \infty$ such that $l \leq {\phi^{2}(t)}/{t^{2}} \leq u$ for all $t\in \sT$ for some set of interest $\sT$. An example is that of a Swish-type activation function where
% \begin{align*}
% \phi(t) = t \cdot \left( \sqrt{c_1} + (\sqrt{c_2} - \sqrt{c_1}) \frac{1}{1 + e^{-\zeta t}} \right),
% \end{align*}
% for some $0 \leq c_1 < c_2$ and $\zeta \in \real$. %Note that with $c_2=1$, as $c_1 \to 0$, the above function approaches the typical swish function $t/ (1 + e^{-\zeta t})$. 
% Suppose $c_1 > 0$. Then, we can take $l = \sqrt{c_1}$ and $u = \sqrt{c_2}$. For $c_1 = 0$, i.e., the typical Swish function, we can still define $0 < l \defeq \displaystyle \min_{t \in \mathcal{T}} {\phi^{2}(t)}/{t^{2}} $ provided that $\sT$ is a bounded set.
Consider the class of generalized linear predictor models from \cref{ex:glm} with any activation function $\phi$ such that there exist constants $0 < l < u < \infty$ for which $l \leq {\phi^{2}(t)}/{t^{2}} \leq u$ for all $t\in \sT$, where $\mathcal{T}$ is some set of interest. An example is the Swish-type activation function given by
\begin{align*}
\phi(t) &= t \cdot \left( \sqrt{c_1} + (\sqrt{c_2} - \sqrt{c_1}) \frac{1}{1 + e^{-\zeta t}} \right),
\end{align*}
for some $0 \leq c_1 < c_2$ and $\zeta \in \mathbb{R}$. If $c_1 > 0$, we can take $l = \sqrt{c_1}$ and $u = \sqrt{c_2}$. For $c_1 = 0$, i.e., the typical Swish function, we can still define $0 < l \defeq \displaystyle \min_{t \in \mathcal{T}} {\phi^{2}(t)}/{t^{2}} $ provided that $\mathcal{T}$ is a bounded set.

Consider leverage score sampling according to \cref{def:lev_score}. 
%Since $\phi(0) = 0$, we have $
%\langle \ee_{i},\widehat{\FF}^{\star}(\btheta) \left[\widehat{\FF}^{\star}(\btheta)\right]^{\dagger} \ee_{i} %\rangle = \langle \ee_{i},\FF^{\star}(\btheta) \left[\FF^{\star}(\btheta)\right]^{\dagger} \ee_{i} \rangle$. 
Suppose $\XX \in \real^{n \times d}$ and $\FF^{\star}(\btheta) \in \real^{n \times d}$ are both full column rank. Since 
\begin{align*}
\left[\FF^{\star}(\btheta)\right]^{\T} \FF^{\star}(\btheta) = \sum_i \left( \frac{\phi(\langle \btheta, \xx_i \rangle)}{\langle \btheta, x \rangle} \right)^2 \xx_i \xx_i^T,
\end{align*}
it follows that, as long as $c_1 > 0$ or $\btheta \in \sC$ for some bounded set $\sC$,  we have $\zero \prec l \cdot \XX^{\T} \XX \preceq \left[\FF^{\star}(\btheta)\right]^{\T} \FF^{\star}(\btheta) \preceq u \cdot \XX^{\T} \XX$.
% \begin{align*}
% \zero \prec l \cdot \XX^{\T} \XX \preceq \left[\FF^{\star}(\btheta)\right]^{\T} \FF^{\star}(\btheta) \preceq u \cdot \XX^{\T} \XX.
% \end{align*}
Hence,
\begin{align*}
    d \cdot \tau_{i}(\btheta) &= \dotprod{\ff^{\star}(\btheta), \left( \left[\FF^{\star}(\btheta)\right]^{\T} \FF^{\star}(\btheta)\right)^{-1} \ff^{\star}(\btheta)} \\
    &\leq \frac{1}{c_{1}}\dotprod{\ff^{\star}(\btheta), \left( \XX^{\T} \XX\right)^{-1} \ff^{\star}(\btheta)} \\
    &\leq \frac{u}{l}\dotprod{\xx, \left( \XX^{\T} \XX\right)^{-1} \xx} = \frac{d \cdot u}{l} \tau_{i},
\end{align*}
where $\tau_{i}$ is the linear leverage score. This implies $\beta = l/u$. 
\end{example}

\begin{example}[ReLU Neural Networks]
\label{ex:nn_estimate}
%Going beyond single index models, We now consider the application of our framework to sampling more complex neural network models. To the best of our knowledge, this is the first attempt to extend importance sampling strategies with rigorous approximation guarantees of the form given in \cref{eq:goal}.
Revising the NN model in \cref{ex:nn}, we consider row norm sampling according to  \cref{def:row_norm}. Suppose $\phi$ is such that such that $c_1 \leq  (\phi(t)-\phi(0))^{2}/t^{2} \leq c_2$ for some $0 < c_1 \leq c_2 < \infty$ and for all $t\in \sT$ for some set of interest $\sT$, e.g., Swish-type or linear output layer. Recall that $ \btheta = [\btheta_{1},\ldots,\btheta_m] $ where $ \btheta_{j} = [a_{j},\bb_{j}] $. Also denote  $ \btheta^{\star} = [\btheta_{1}^{\star},\ldots,\btheta_{m}^{\star}] $ where $\btheta^{\star}_{j} = [a_{j}^{\star},\bb_{j}^{\star}]$. Suppose $a^{\star}_{j} \neq 0$ for all $j \in {1, 2, \ldots, m}$. 
%If $a^{\star}_{j} = 0$ for some $j$, we can simply eliminate all connections leading to $a^{\star}_{j}$ and consider a network with $m-1$ hidden units. 
Let $l > 0$ be such that $\min_{j} (a^{\star}_{j})^{2} \geq  l$ and  $0 < u < \infty $ be such that $\sum_{j=1}^{m} \left( \vnorm{\bb^{\star}_{j}}^{2} + (a^{\star}_j)^2 \right) \leq u$. Define the set 
\begin{align*}
\sC \defeq \left\{ [a_{1},\bb_{1},\ldots,a_{m},\bb_{m}] \mid  \min_{j  = 1,\ldots,m} \; (a_{j})^{2} \geq  l, \right.\\
\left.  \sum_{j=1}^{m} \left( \vnorm{\bb_{j}}^{2} + (a_j)^2 \right) \leq u  \right\}.
\end{align*}
Clearly, by construction, $\bthetas \in \sC $. After some algebraic manipulations (see \cref{sec:appendix:nn}), for any $\xx_{i}$ and any $\btheta \in \sC$, we have 
\begin{align*}
c_1 l  \vnorm{\xx_{i}}^{2} &\leq \vnorm{\ff^{\star}_{i}(\btheta)}^{2} \leq c_2 u \vnorm{\xx_{i}}^{2},
\end{align*}
where $\ff^{\star}_{i}(\btheta)$ is the adjoint operator of 
\begin{align*}
    f_{i}(\btheta) = \sum_{j=1}^{m} \phi_{i}(a_{j} \cdot   \max\left\{\dotprod{\bb_{j},\xx_{i}},0\right\}).
\end{align*}
This implies that 
\begin{align*}
\tau_{i}(\btheta) = \frac{\vnorm{\widehat{\ff}^{\star}_{i}(\btheta)}_{2}^{2}}{\vnorm{\widehat{\FF}^{\star}(\btheta)}^{2}_{\textnormal{F}}} \leq \left(\frac{\max\{c_2 u,1\}  }{\min\{c_1 l,1\}}\right) \tau_{i},
\end{align*}
where $\tau_{i}$ is the norm score for the $i\th$ row of
\begin{align*}
    \widehat{\XX} = \begin{bmatrix}
        \xx_{1} &  \xx_{2} & \ldots & \xx_{n} \\
        m\phi_{1}(0) & m\phi_{2}(0) & \ldots & m\phi_{n}(0)
    \end{bmatrix}^{\T} \in \real^{n \times  (d+1)}.
\end{align*}
% \begin{align*}
%     \widehat{\XX} = \begin{bmatrix}
%         \xx^{\T}_{1} & m\phi_{1}(0) \\ \xx^{\T}_{2} & m\phi_{2}(0)  \\ \vdots & \vdots \\ \xx^{\T}_{n} & m\phi_{n}(0) 
%     \end{bmatrix} \in \real^{n \times  (d+1)}.
% \end{align*}
This implies $\beta = {\min\{c_1 l,1\}}/{\max\{c_2 u,1\}}$.

\end{example}

\subsubsection{Lower bounding $\sL_{\sS}(\btheta^{\star}_{\sS})$}
\label{sec:lower_bound}
% Let $\btheta^{\star}$ denote a solution to \cref{eq:loss}, and consider a compact ball with a radius $R$, chosen large enough to contain $\btheta^{\star}$. Let $\sB^{\star}_{R}$ denote this ball. For any $\varepsilon \in (0,1)$, we pick a discrete subset $\sN_{\varepsilon} \subseteq \sB^{\star}_{R}$ such that, for every $\btheta \in \sB^{\star}_{R}$, there is at least one $\btheta^{\prime} \in \sN_{\varepsilon}$ such that,
% \begin{align*}
% \|\btheta - \btheta^{\prime}\|_{2} \;\leq\; \varepsilon R.
% \end{align*}
% This construction is often referred to as $\varepsilon$-net and reduces an uncountable continuous set to a finite covering set \cite{woodruff2014sketching,vershynin2018high}. Standard volume and covering number arguments provide an upper bound on the cardinality of the set $|\sN_{\varepsilon}|$.  Recall that the volume of $\sB^{\star}_{R}$ is
% \begin{align*}
% \text{Vol}(\sB^{\star}_{R}) \;=\; \frac{\pi^{p/2} R^{p}}{\Gamma\bigl({p}/{2}+1\bigr)},
% \end{align*}
% where $\Gamma$ is Euler's gamma function. A bound on $|\sN_{\varepsilon}|$ can be obtained as the number of small balls of radius $\varepsilon R/2$ that cover the larger ball of of radius $(1+{\varepsilon}/{2})R$, which is given by as the ratio of their respective volume, 
% \begin{align*}
% |\sN_{\varepsilon}| \geq \frac{{((1+ {\varepsilon}/{2}) R)}^p}{({\varepsilon R}/{2})^{p}} \;=\; ({1 + {2}/{\varepsilon}})^p \in \Theta({1}/{\varepsilon}^p).
% \end{align*}
Let $\btheta^{\star}$ denote a solution to \cref{eq:loss}, and consider a ball with radius $R$, chosen large enough to contain $\btheta^{\star}$, denoted by $\sB^{\star}_{R}$. For any $\varepsilon \in (0,1)$, we pick a discrete subset $\sN_{\varepsilon} \subseteq \sB^{\star}_{R}$ such that, for every $\btheta \in \sB^{\star}_{R}$, there exists at least one $\btheta^{\prime} \in \sN_{\varepsilon}$ satisfying $\|\btheta - \btheta^{\prime}\|_{2} \leq \varepsilon R$.
% \begin{align*}
% \|\btheta - \btheta^{\prime}\|_{2} \leq \varepsilon R.
% \end{align*}
The size of this set can be shown to be $|\sN_{\varepsilon}| \in \Omega\left({1}/{\varepsilon^p}\right)$. This construction is well-known and is commonly referred to as an $\varepsilon$-net; see \cref{sec:appendix:e_net}. 
% Standard volume and covering number arguments provide an upper bound on the cardinality of the set $|\sN_{\varepsilon}|$. Recall that the volume of $\sB^{\star}_{R}$ is $\text{Vol}(\sB^{\star}_{R}) = {\pi^{p/2} R^{p}}/{\Gamma\left({p}/{2}+1\right)}$,
% % \begin{align*}
% % \text{Vol}(\sB^{\star}_{R}) = \frac{\pi^{p/2} R^{p}}{\Gamma\left(\frac{p}{2}+1\right)},
% % \end{align*}
% where $\Gamma$ is Euler's gamma function. A bound on $|\sN_{\varepsilon}|$ can be obtained as the number of small balls with radius $\varepsilon R/2$ that cover the larger ball of radius $(1 + \frac{\varepsilon}{2})R$, which is given by the ratio of their respective volumes:
% \begin{align*}
% |\sN_{\varepsilon}| \geq \frac{{\left(1 + \frac{\varepsilon}{2}\right)}^p R^p}{{\left(\frac{\varepsilon R}{2}\right)}^p} = \left(1 + \frac{2}{\varepsilon}\right)^p \in \Theta\left(\frac{1}{\varepsilon^p}\right).
% \end{align*}

This $\varepsilon$-net construction enables a union bound argument to control the probability of \cref{eq:approx} for all $\btheta \in \sN_{\varepsilon}$. Suppose $\sS$ can be sampled according to a uniform estimate of the nonlinear scores, as described in \cref{sec:estimating_scores}, so that $\sS$ no longer depends on the choice of $\btheta$. If, for a given $\btheta$, \cref{eq:approx} fails to hold with a probability of at most $\delta^{\prime}$, then a union bound ensures that the overall failure probability of \cref{eq:approx} for all $\btheta \in \sN_{\varepsilon}$ cannot exceed $|\sN_{\varepsilon}|\delta^{\prime}$. 
By choosing $\delta^{\prime} = \delta/|\sN_{\varepsilon}| \leq \delta \varepsilon^{p}$, we ensure \cref{eq:approx} holds with a success probability of at least $1 - \delta$ for all $\btheta \in \sN_{\varepsilon}$.

Suppose $\sL_{\sS}(.)$ is continuous on $\sB^{\star}_{R}$, which is almost always guaranteed for many loss functions in ML. The compactness of the ball implies that $\sL_{\sS}(.)$ is also Lipchitz continuous with respect to $\btheta$ on $\sB^{\star}_{R}$, i.e., for any set of samples $\sS$, there exists a constant $0 \leq L(f,\XX,\sS,R) < \infty$ such that for any $\btheta, \btheta^{\prime} \in \sB^{\star}_{R}$, $|\sL_{\sS}(\btheta) - \sL_{\sS}(\btheta^{\prime})| \;\leq\; L(f,\XX, \sS,R) \|\btheta - \btheta^{\prime}\|_{2}$.
% \begin{align*}
% |\sL_{\sS}(\btheta) - \sL_{\sS}(\btheta^{\prime})| \;\leq\; L(f,\XX, \sS,R) \|\btheta - \btheta^{\prime}\|_{2}.
% \end{align*}
Note that the Lipschitz continuity constant may depend on $\XX$, $\sS$, $f$, and the radius of the ball $\sB^{\star}_{R}$. Define $L(f, \XX, R) \defeq \max_{\sS} L(f,\XX, \sS, R)$.
% \begin{align*}
%     L(f, \XX, R) \defeq \max_{\sS} L(f,\XX, \sS, R). 
% \end{align*}
Note that this is a maximization over a finite collection of numbers\footnote{The total number of possible unordered samples of any size from $n$ object with replacement is given by $\sum_{s=1}^{n}\binom{n+s-1}{s}$.} and hence $L(f, \XX, R) < \infty$.

Now, suppose $\sC \subseteq \sB_{R}^{\star}$ with $\bthetas \in \sC$, and let
\begin{align}
    \label{eq:bthetasS}
    \bthetasS \in \argmin_{\btheta \in \sC} \sL_{\sS}(\btheta).
\end{align}
Also, let $\btheta_{0} \in \sN_{\varepsilon}$ be such that $\|\bthetasS - \btheta_{0}\|_{2} \;\leq\; \varepsilon R$. Without loss of generality, we can assume $\btheta_{0} \in \sN_{\varepsilon} \cap \sC$ as otherwise we can simply increase the size of the net inside $\sC$ by $\bigO{|\sN_{\varepsilon}|}$ . From Lipschitz continuity and \cref{eq:approx}, we get 
\begin{align*}
&\sL(\bthetasS) \leq \sL(\btheta_{0}) + \varepsilon R \cdot L(f, \XX, R) \\
&\hspace{-1mm}\leq \frac{1}{1-\varepsilon} \sL_{\sS}(\btheta_{0}) + \varepsilon R \cdot L(f, \XX, R)\\
&\hspace{-1mm}\leq \frac{1}{1-\varepsilon} \Big(\sL_{\sS}(\bthetasS) \hspace{-0.5mm} + \hspace{-0.5mm} \varepsilon R \cdot L(f, \XX, R) \Big)   + \varepsilon R \cdot L(f, \XX, R),
\end{align*}
which gives the lower bound
\begin{align*}
\sL_{\sS}(\bthetasS) &\geq (1-\varepsilon)\sL(\bthetasS) -   \varepsilon (2-\varepsilon) R \cdot L(f, \XX, R).
\end{align*}
Combining this with \cref{eq:param_indep_01}, we get 
\begin{align*}
    \sL(\btheta^{\star}_{\sS}) \hspace{-0.5mm}\leq\hspace{-0.5mm} \sL(\btheta^{\star}) \hspace{-0.5mm}+\hspace{-0.5mm} \frac{\varepsilon}{1-\varepsilon} \Big( \sL(\bthetas)  \hspace{-0.5mm}+\hspace{-0.5mm} (2-\varepsilon) R \cdot L(f, \XX, R)\Big). 
\end{align*}

% \begin{theorem}[Main Result]
%     \label{thm:param_indep}
%     Suppose $\bthetas \in \sC$ is a solution to \cref{eq:loss}. Furthermore, assume that for any given $\btheta \in \sC$ and any $\delta^{\prime} \in (0,1)$, we can sample $\sS$, potentially depending on $\btheta$, such that \cref{eq:approx} holds with the probability at least $1-\delta^{\prime}$.  Suppose $\delta^{\prime} \leq \delta \varepsilon^{p}$ for some $\varepsilon \in (0,1)$ and $\delta \in (0,1)$ and let $\bthetasS$ be defined as in \cref{eq:bthetasS}. Then \cref{eq:goal} holds with probability at least $1-\delta$.
% \end{theorem}

The culmination of the above discussions and derivations leads to the main result of this paper.
\begin{theorem}[Main Result]
    \label{thm:param_indep_02}
    Suppose $\bthetas \in \sC$ is a solution to \cref{eq:loss} with squared loss $\ell(t) = t^2$ and  for some $0 < \beta \leq 1$, we have $ \beta\tau_{i}(\btheta) \leq   \tau_{i}$ for $i=1,2,\ldots,n$ and for all $ \btheta \in \sC$. Consider random sampling according to $\tau_{i}$ with a sample $\sS$ of size at least $\bigO{\left( p \log(p/\delta) + p^{2} \log(p/\varepsilon)\right)/(\beta \varepsilon^{2}})$ for some $\varepsilon \in (0,1)$ and $\delta \in (0,1)$. Let $\bthetasS$ be defined as in \cref{eq:bthetasS}. Then \cref{eq:goal} holds with probability at least $1-\delta$. 
\end{theorem}
\begin{remark}
    \cref{thm:param_indep_02} applies to more generally than single index models. However, when adapted to \cref{ex:glm_estimate}, it provides a result similar to \citet{gajjar2024agnostic}. While both works have the same dependence on $\varepsilon$, our result has quadratic dependence on the dimension, while theirs is linear. Additionally, our approach involves a constrained optimization subproblem \cref{eq:bthetasS}, %requiring the choice of  $\sC$ such that $\bthetas \in \sC$, 
    whereas the subproblems of \citet{gajjar2024agnostic} are unconstrained. In contrast, \citet{gajjar2024agnostic} provided a guarantee of the form \cref{eq:goal_bad} with a potentially large constant $C \gg 1$, while \cref{thm:param_indep_02} provides a more desirable guarantee of the form \cref{eq:goal}.
    %In addition, in addition to requiring Lipschits activation function, we also require monotonic and more recent ones if any, in terms of \cref{thm:param_indep_02} and the underlying assumptions...we have a worse sampling complexity $p^2/\varepsilon^2$ for us vs.\ $p/\varepsilon^{2}$ for \cref{gajjar2024agnostic}, but but we do not need to know Lipschitz constant for the algorithm unlike them but we need to know $R$ in $\sB_{R}^{\star}$. We both need to know Lipschitz constant of $\phi$ for sample size.....they have potentially a bad constant $C$ that multiplies ``OPT'' whereas we do not have that constant...anything else?
\end{remark}
% After establishing that the nonlinearity can be captured with the linear data matrix with an additional misestimation factor error. The below algorithm outlines the solution to the sub-sampled problem with $m$ labeled samples.
% \begin{algorithm}[H]
% \caption{Sampling with Nonlinear Adjoint}
% \label{alg:leverage}
% \begin{algorithmic}[1]
% \INPUT Matrix $\XX \in \mathbb{R}^{n \times d}$, target vector $y \in \mathbb{R}^n$, number of samples $m$.
% \OUTPUT Approximate solution $\btheta_s$ to $\min_{\btheta \in \mathcal{C}} \sL(\btheta, \XX_s)$, where $\XX_s$ is a subsampled version of $\XX$.
% \STATE Compute sampling scores $\tau_i^\XX$ for each row $i=1,\ldots,n$ of $\XX$.
% \STATE Compute the normalized sampling probabilities 
% \begin{align*}
% \tau_{i} = \frac{\tau_i^\XX}{\sum_{j=1}^n \tau_j^\XX}, \quad \text{for } i=1,\ldots,n.
% \end{align*}
% \STATE Construct a sampling matrix $\bSS \in \mathbb{R}^{m \times n}$ according to the above probabilities (each of the $m$ rows of $\bSS$ has exactly one non-zero entry, chosen with probability $\tau_{i}$).
% \STATE Form the subsampled matrix $\XX_s = \bSS\,\XX$.
% \STATE Solve the constrained minimization problem:
% \begin{align*}
% \btheta_s = \arg\min_{\btheta \in \mathcal{D}} \| \sL(\btheta, \XX_s) \|_2^2.
% \end{align*}
% \STATE \textbf{return} $\btheta_s$
% \end{algorithmic}
% \end{algorithm}
% \vspace{-3mm}
\section{Experiments}
\label{sec:experiments}
%We present two complementary sets of experiments focused primarily on supervised learning tasks to demonstrate the effectiveness and versatility of the proposed nonlinear importance scores. First, we consider image classification using four standard datasets such as \texttt{SVHN} (Street View House Numbers) \cite{netzer2011reading}, FER-2013 (Facial Expression Recognition) \cite{goodfellow2013challenges}, NotMNIST \cite{bulatov2011notmnist}, and Quick, Draw (QD) \cite{ha2018neural}, aiming to show that the proposed sampling scores acts as a novel diagnotics tools for complex black-box models to find the most ``important'' points through, complementing other methods in this direction such as \citet{koh2017understanding, shrikumar2017learning} and also serving as an outlier and anomaly detector by isolating points that deviate significantly from the nominal data manifold \cite{schirrmeister2020understanding, qiu2022latent}. Second, we use benchmarking regression datasets such as California Housing Prices \cite{pace1997sparse} and Medical Insurance dataset \cite{lantz2019machine}, with an objective to show that our importance sampling approach outperforms traditional linear sampling strategies and uniform sampling. Throughout these experiments, we employ both single-index models and neural networks to ensure that our findings hold across different model complexities and architectures. Further experimental details are given in \cref{sec:appendix:exp_details}. The code used for our experiments is available \href{https://anonymous.4open.science/}{here}.
We conduct a series of experiments to demonstrate the effectiveness and versatility of the proposed nonlinear importance scores. 
First, we evaluate our approach on several benchmarking regression datasets, including California Housing Prices \cite{pace1997sparse}, Medical Insurance dataset \cite{lantz2019machine}, and Diamonds dataset \cite{wickham2009ggplot2}. % and the Mammographic Mass dataset \cite{mammographic_mass_161}. 
As a proof of concept, we demonstrate that our importance sampling method outperforms traditional linear and uniform sampling strategies by achieving lower training error using fewer samples. 
Second, we consider image classification using four standard datasets: \texttt{SVHN} (Street View House Numbers) \cite{netzer2011reading}, \texttt{FER-2013} (Facial Expression Recognition) \cite{goodfellow2013challenges}, \texttt{NOTMNIST} \cite{bulatov2011notmnist}, and \texttt{QD} (Quick, Draw) \cite{ha2018neural}. Here, we demonstrate that our proposed nonlinear leverage scores in \cref{def:lev_score} serve as a powerful diagnostic tool, identifying the most ``important'' points in training complex models and detecting anomalies by isolating outliers.
We employ both single-index models and neural networks to validate our findings across varying model complexities and architectures. Additional experimental details are provided in \cref{sec:appendix:exp_details}. The code is available \href{https://github.com/prprakash02/Importance-Sampling-for-Nonlinear-Models}{here}.

% \vspace{-1.5mm}
\begin{figure}[!ht]
    \centering
    \subfigure[California Housing Dataset]{
\includegraphics[width=.72\linewidth]{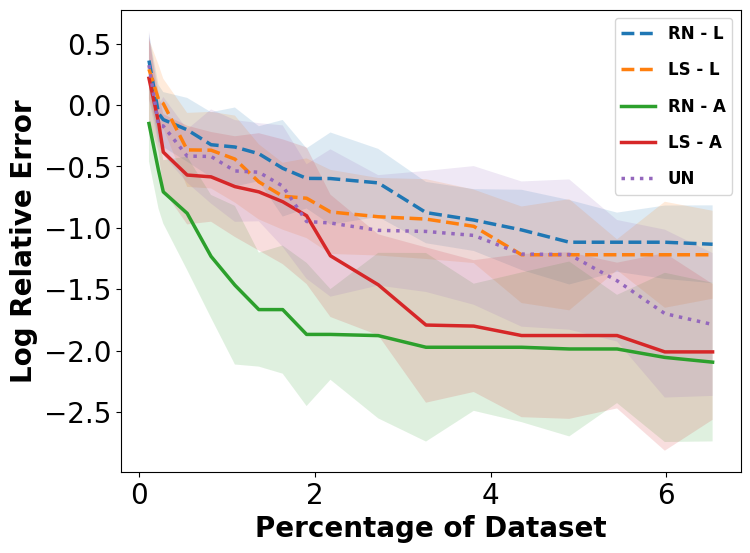}
    }
    
    %\vspace{-1mm}
    \subfigure[Medical Insurance Dataset]{
        \includegraphics[width=.72\linewidth]{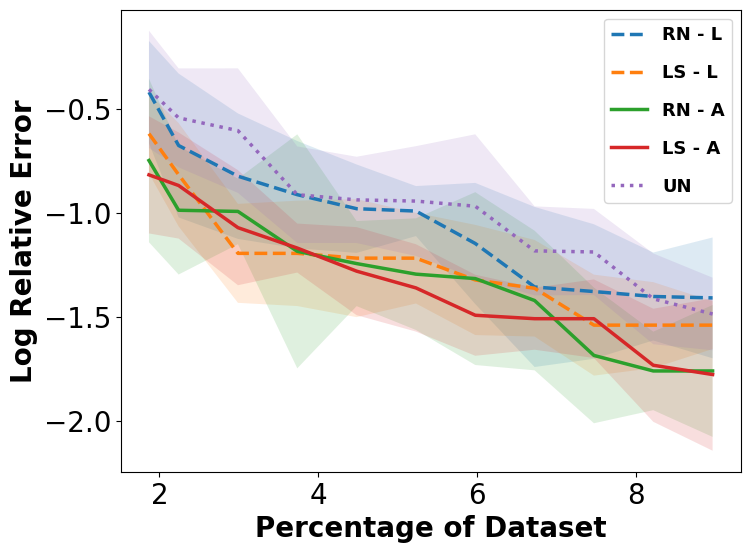}
    }
    
    %\vspace{-1mm}
    \subfigure[Diamonds dataset]{
        \includegraphics[width=.72\linewidth]{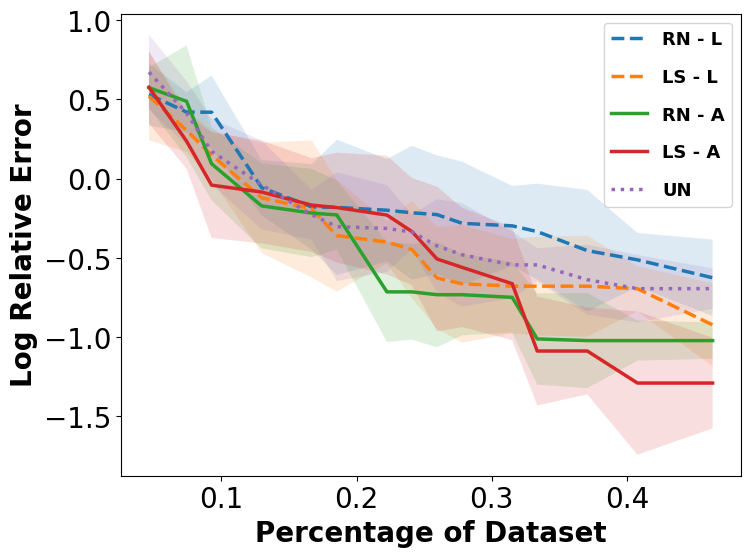}
    }
    
    % \vspace{-1mm}
    % \subfigure[Mammographic Mass dataset]{
    %     \includegraphics[width=.72\linewidth]{images/mam_main_1.png}
    % }
    
    % \vspace{-1mm}
    % \subfigure[Mammographic Mass dataset]{
    %     \includegraphics[width=.72\linewidth]{images/mam_main.png}
    % }

    % \subfigure[Mammographic Mass dataset]{
    %     \includegraphics[width=.72\linewidth]{images/mamo_rel_err.png}
    % }

    % \subfigure[Mammographic Mass dataset (Accuracy not calculated as \%)]{
    %     \includegraphics[width=.72\linewidth]{images/mamo_rel_acc.png}
    % }
    
    % \subfigure[Mammographic Mass dataset (Accuracy calculated as \%)]{
    %     \includegraphics[width=.72\linewidth]{images/mamo_log_rel_acc.png}
    % }
    
    \caption{
    %Comparison of various sampling strategies. The Y-axis represents $\log\left[(\sL(\bthetas) - \sL(\bthetasS))/\sL(\bthetas)\right]$ as a function of the sample size (as a percentage of the total data), where $\bthetasS$ is the optimal parameter obtained by training on the sampled data. ``RN'', ``LS'', and ``UN'' denote Row Norm, Leverage Scores, and Uniform Sampling, respectively, with suffixes ``L'' and ``A'' indicating linear and adjoint-based nonlinear variants. Nonlinear importance scores consistently outperform both linear scores and uniform sampling.
    Comparison of sampling strategies. The Y-axis shows $\log\left[( \sL(\bthetasS) - \sL(\bthetas))/\sL(\bthetas)\right]$ against sample size (as a percentage of total data), where $\bthetasS$ is the optimal parameter from training on sampled data. ``RN'', ``LS'', and ``UN'' denote Row Norm, Leverage Scores, and Uniform Sampling, respectively, with ``L'' and ``A'' indicating linear and adjoint-based nonlinear variants. Nonlinear importance scores consistently outperform all other alternatives.
    % Comparison of sampling strategies.  In subfigures (a)-(c), The Y-axis shows $\log\left[(\sL(\bthetas) - \sL(\bthetasS))/\sL(\bthetas)\right]$ against sample size (as a percentage of total data) where $\bthetasS$ is the optimal parameter from training on sampled data, and in subfigure (d) the accuracies were calculated on the full dataset with the optimal parameter from training on sampled data. ``RN'', ``LS'', and ``UN'' denote Row Norm, Leverage Scores, and Uniform Sampling, respectively, with ``L'' and ``A'' indicating linear and adjoint-based nonlinear variants. Nonlinear importance scores consistently outperform all other alternatives.
    \label{fig:numerical_datasets}}
\end{figure}

\begin{figure*}[!ht]
    \centering

    %----------------- Row 1: \texttt{SVHN} (1 vs 0) and (1 vs 7) -----------------%
    \subfigure[\texttt{SVHN} High (1 vs 0)]{
        \includegraphics[width=0.22\linewidth]{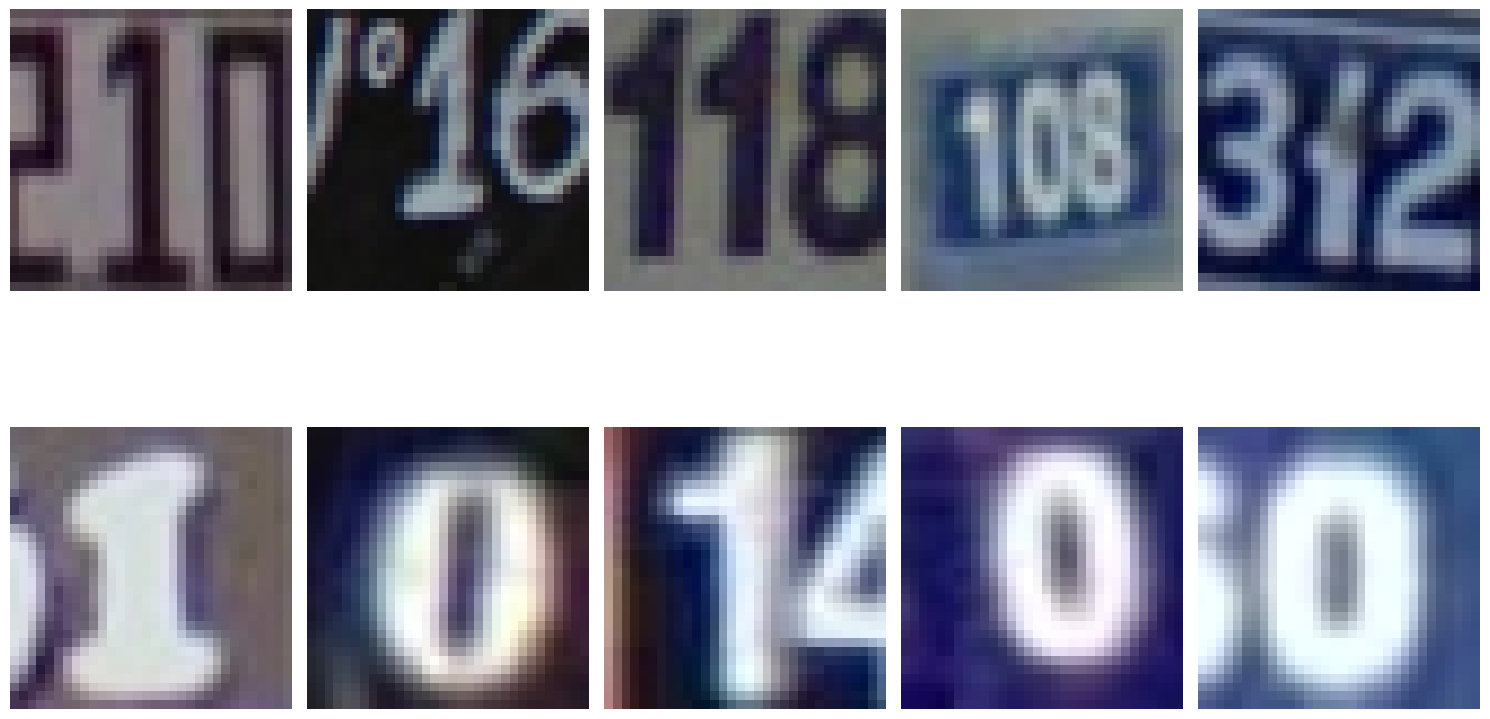}
    }
    \subfigure[\texttt{SVHN} Low  (1 vs 0)]{
        \includegraphics[width=0.22\linewidth]{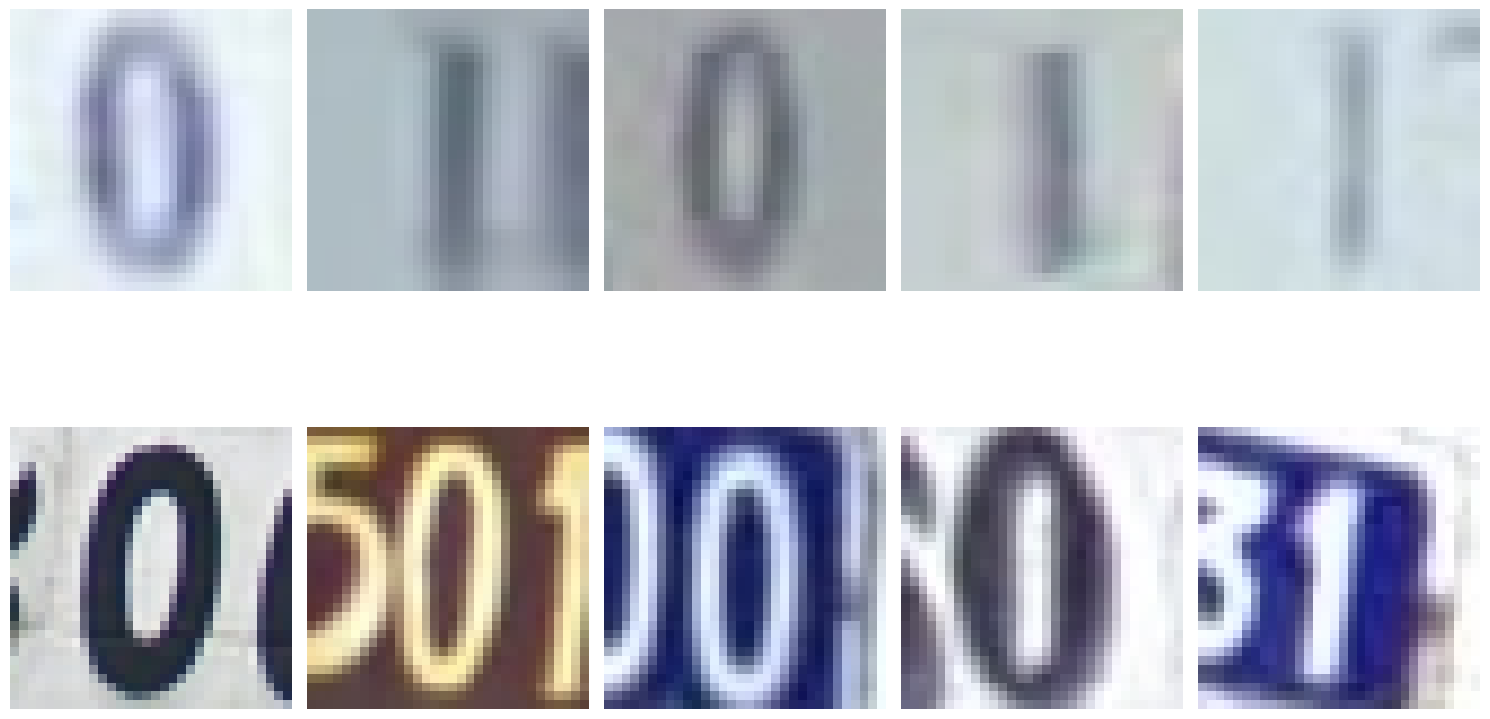}
    }
    \subfigure[\texttt{SVHN} High (1 vs 7)]{
        \includegraphics[width=0.22\linewidth]{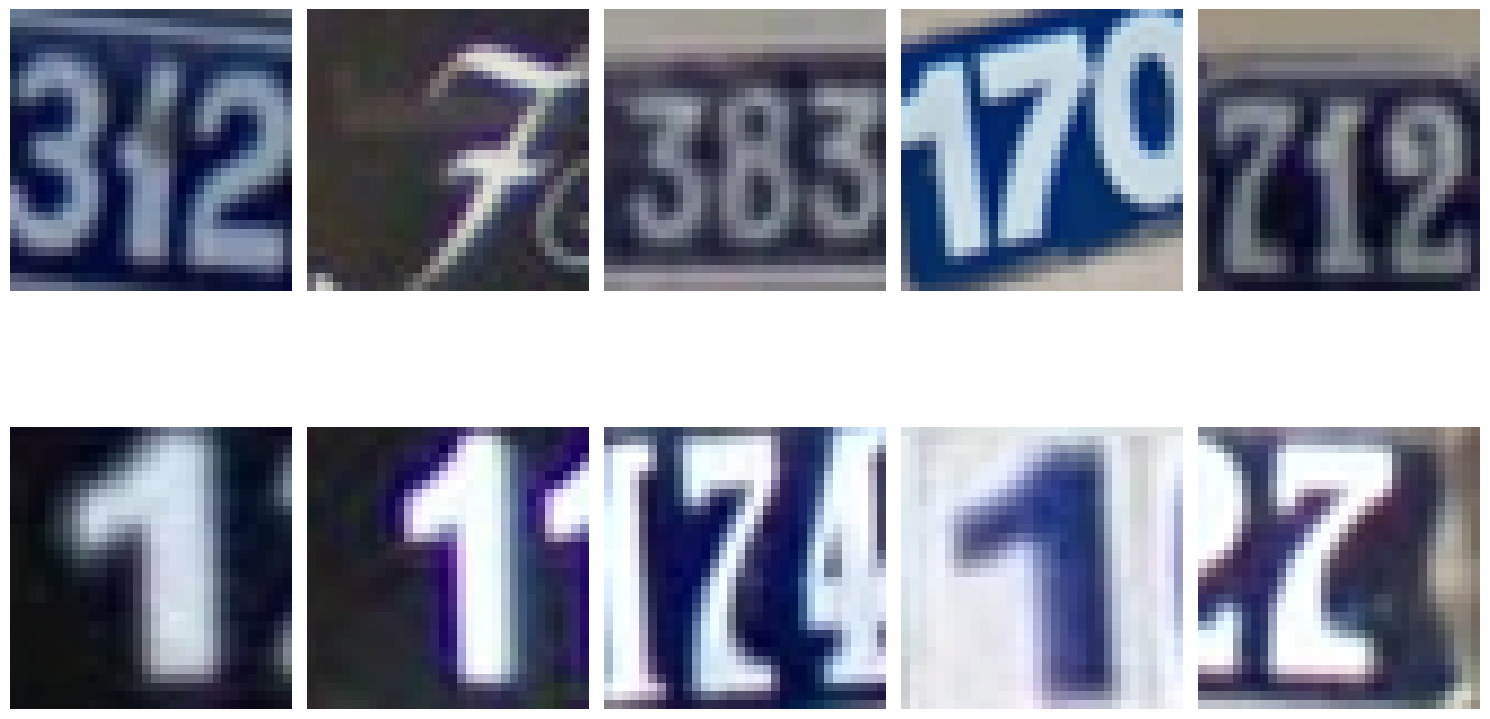}
    }
    \subfigure[\texttt{SVHN} Low   (1 vs 7)]{
        \includegraphics[width=0.22\linewidth]{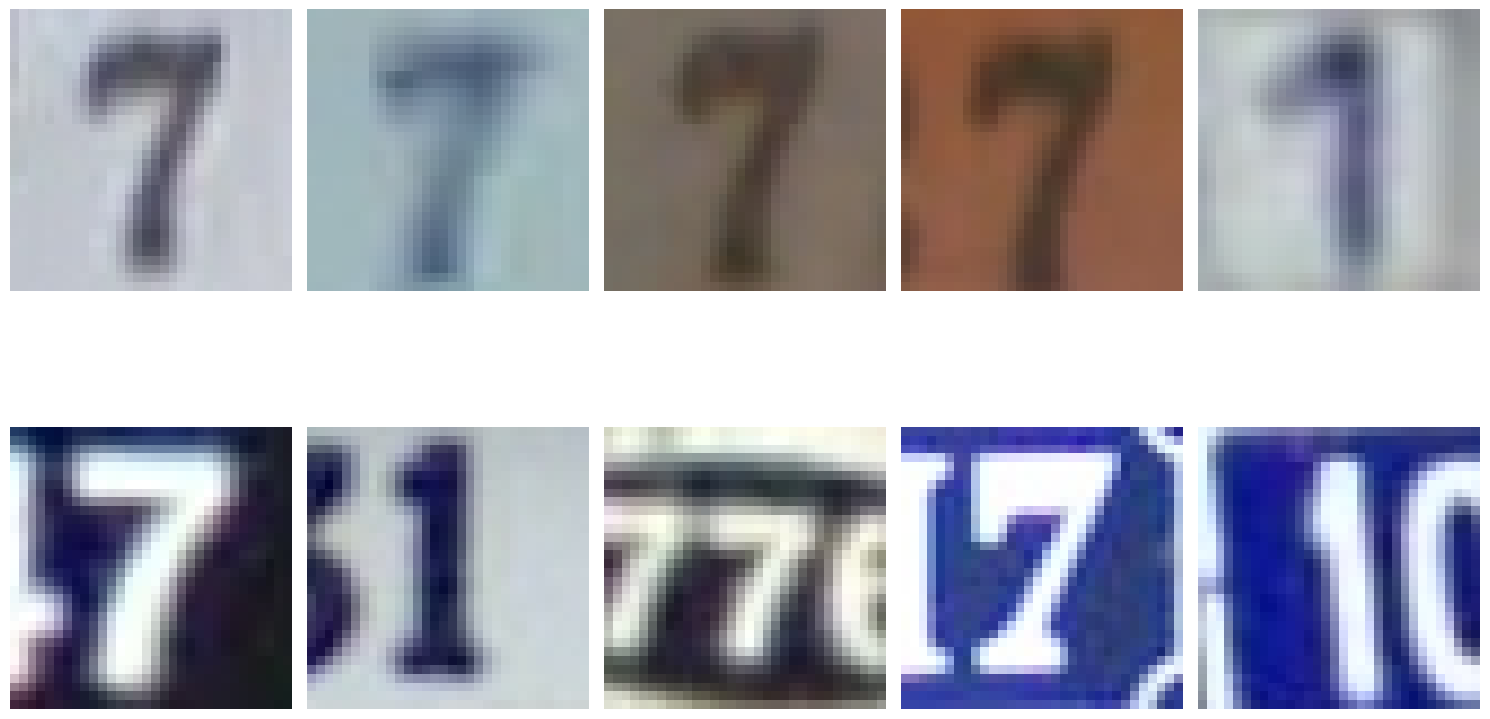}
    }

    % \vspace{1em}
    
    %----------------- Row 2: \texttt{NOTMNIST} (A vs B) and (B vs D) -------------%
    \subfigure[\texttt{NOTMNIST} High (A vs B)]{
        \includegraphics[width=0.22\linewidth]{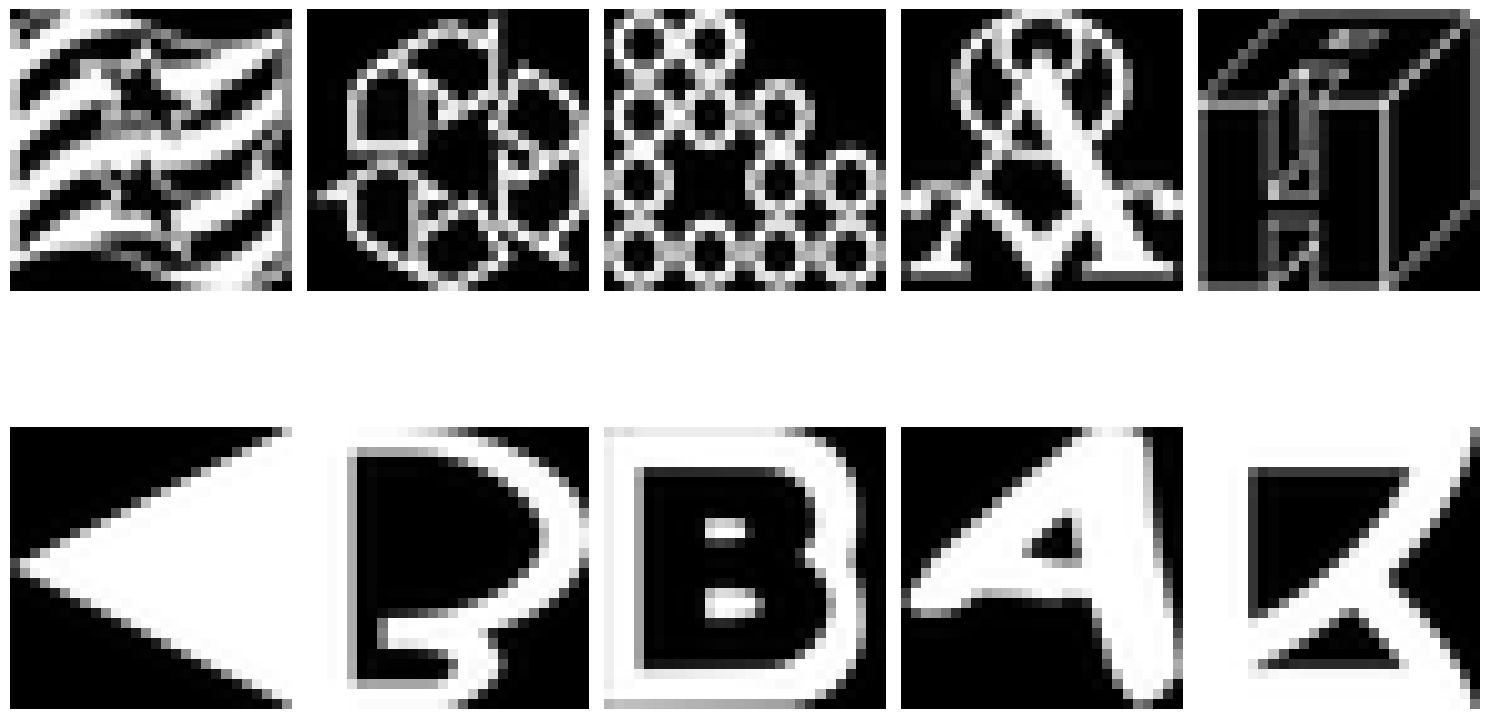}
    }
    \subfigure[\texttt{NOTMNIST} Low   (A vs B)]{
        \includegraphics[width=0.22\linewidth]{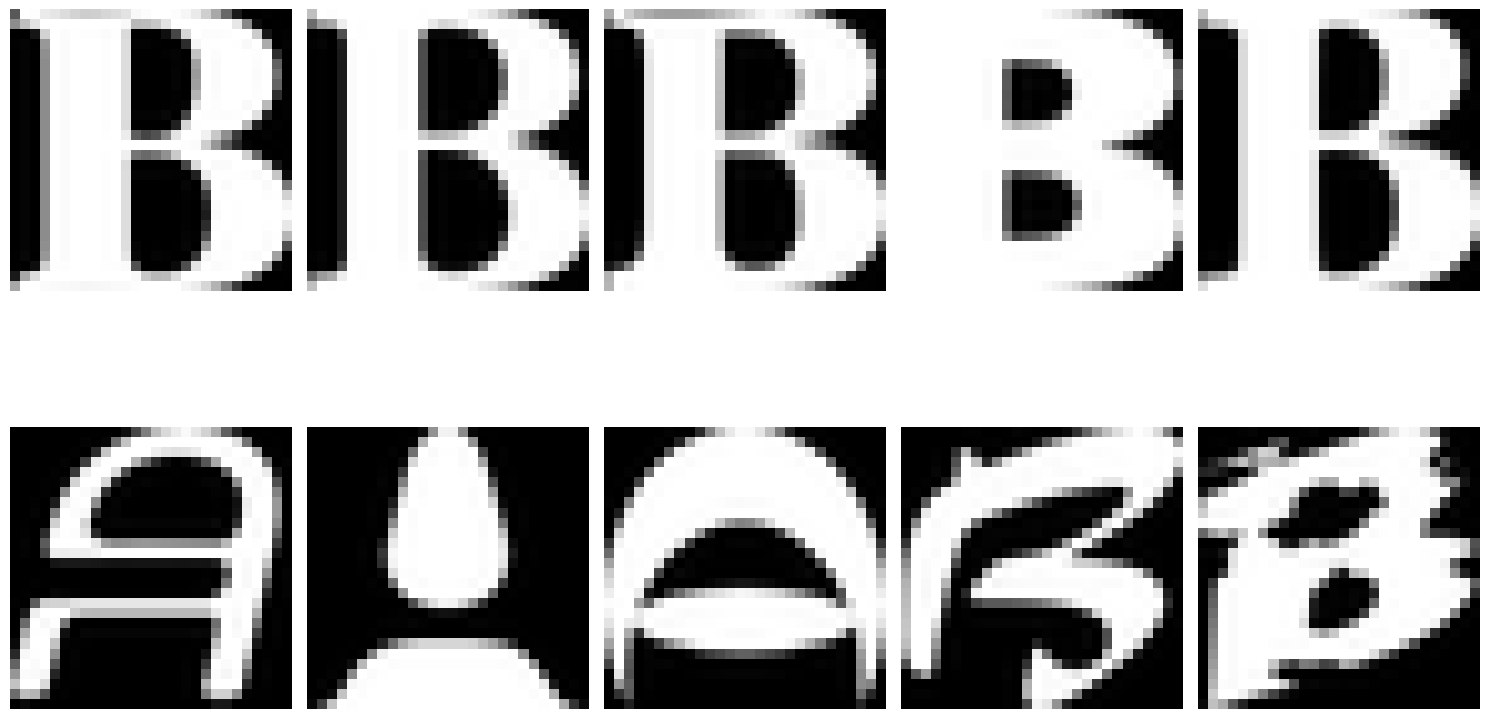}
    }
    \subfigure[\texttt{NOTMNIST} High (B vs D)]{
        \includegraphics[width=0.22\linewidth]{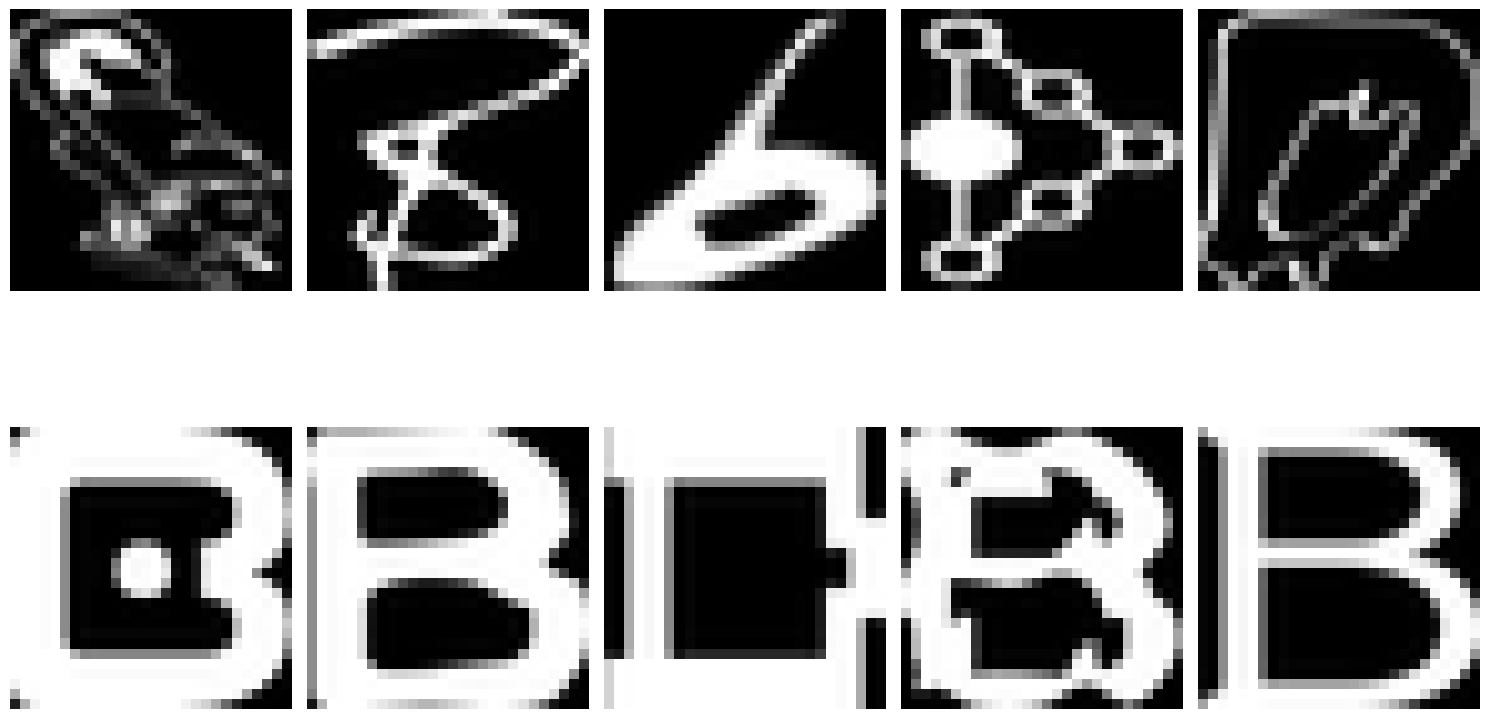}
    }
    \subfigure[\texttt{NOTMNIST} Low   (B vs D)]{
        \includegraphics[width=0.22\linewidth]{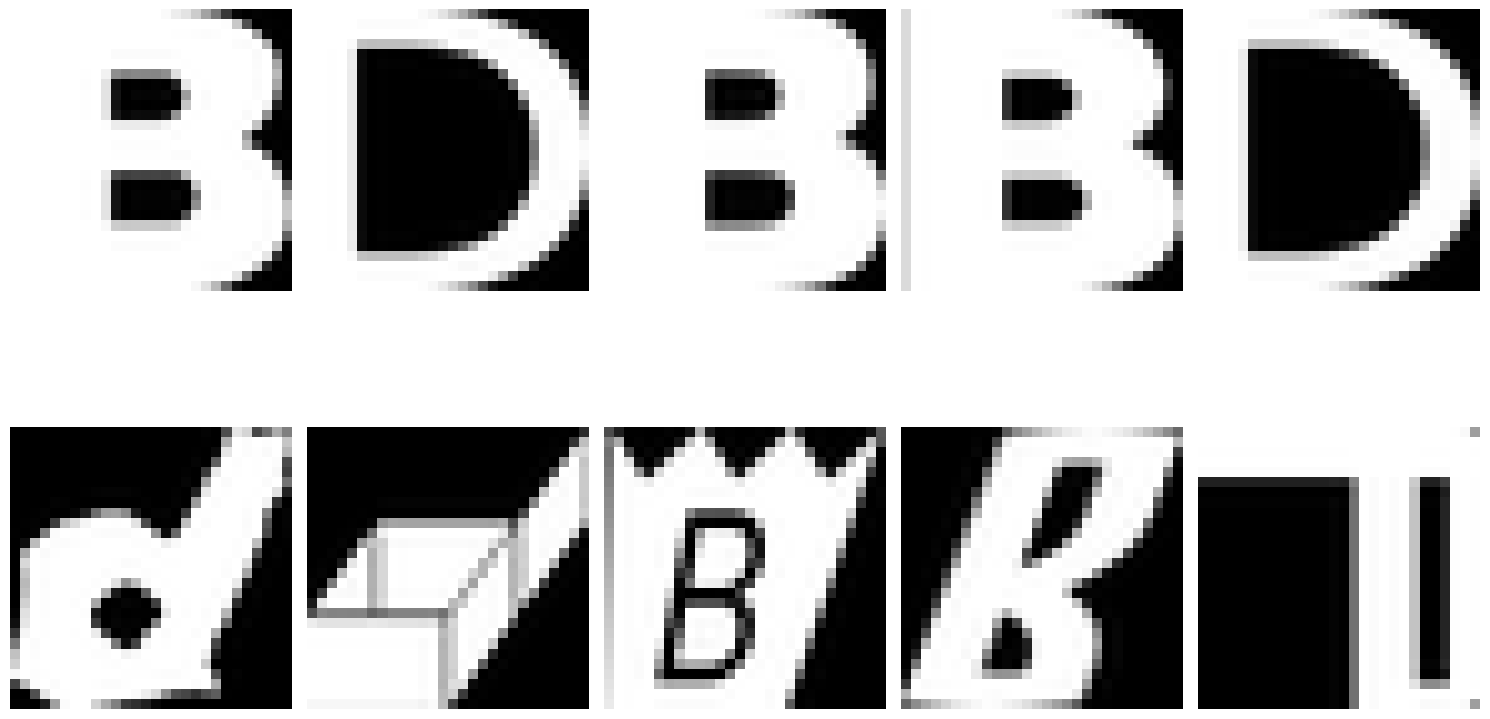}
    }

    % \vspace{1em}

    %----------------- Row 3: \texttt{Quick Draw} and \texttt{FER} ------------------------%
    \subfigure[\texttt{QD} High]{
        \includegraphics[width=0.22\linewidth]{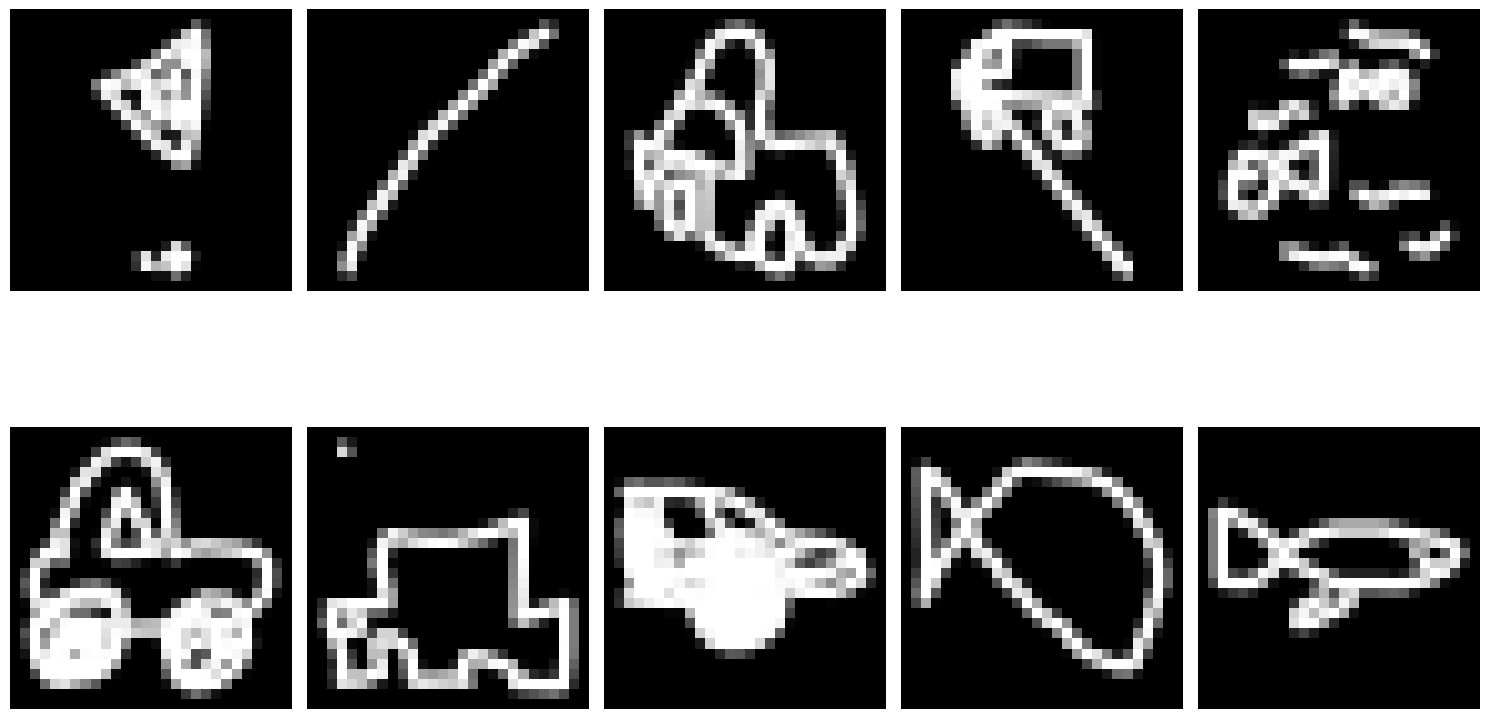}
    }
    \subfigure[\texttt{QD} Low  ]{
        \includegraphics[width=0.22\linewidth]{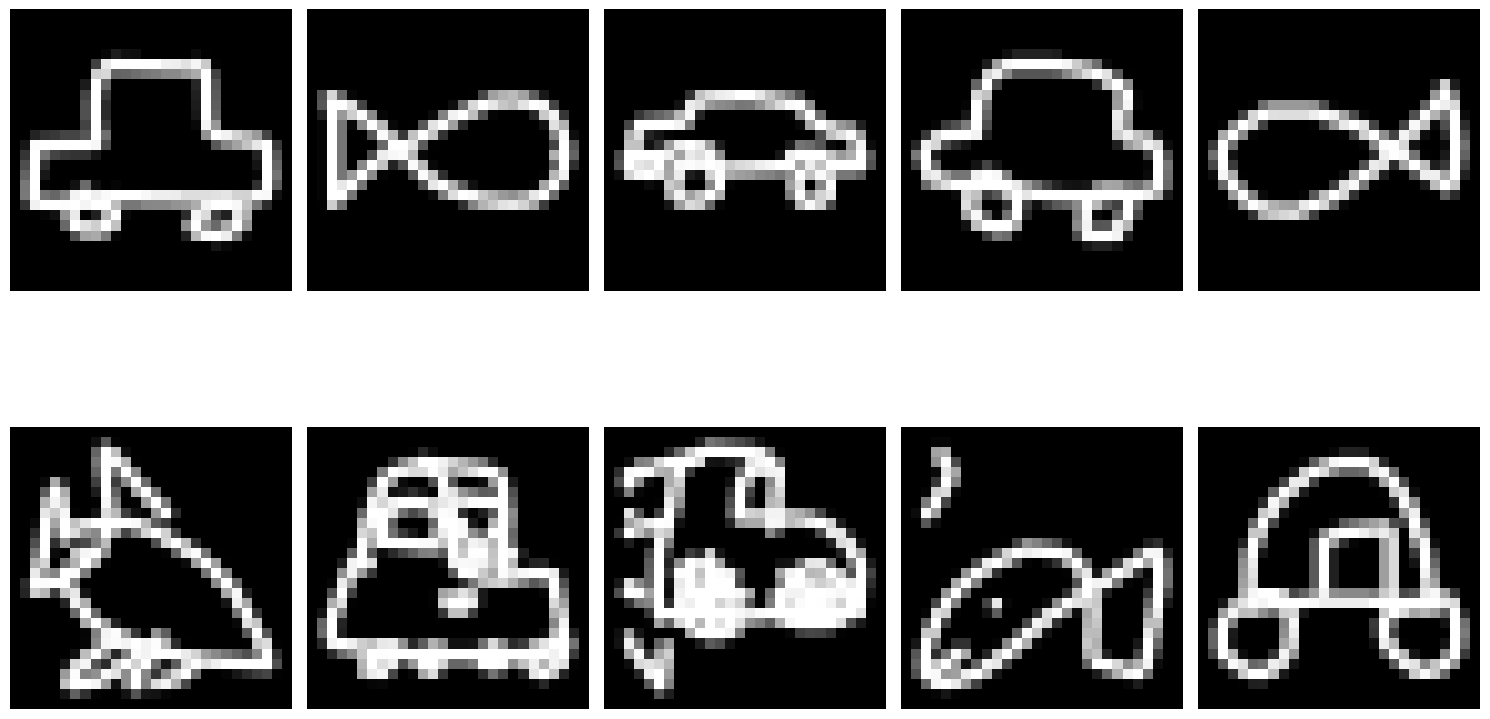}
    }
    \subfigure[\texttt{FER} High]{
        \includegraphics[width=0.22\linewidth]{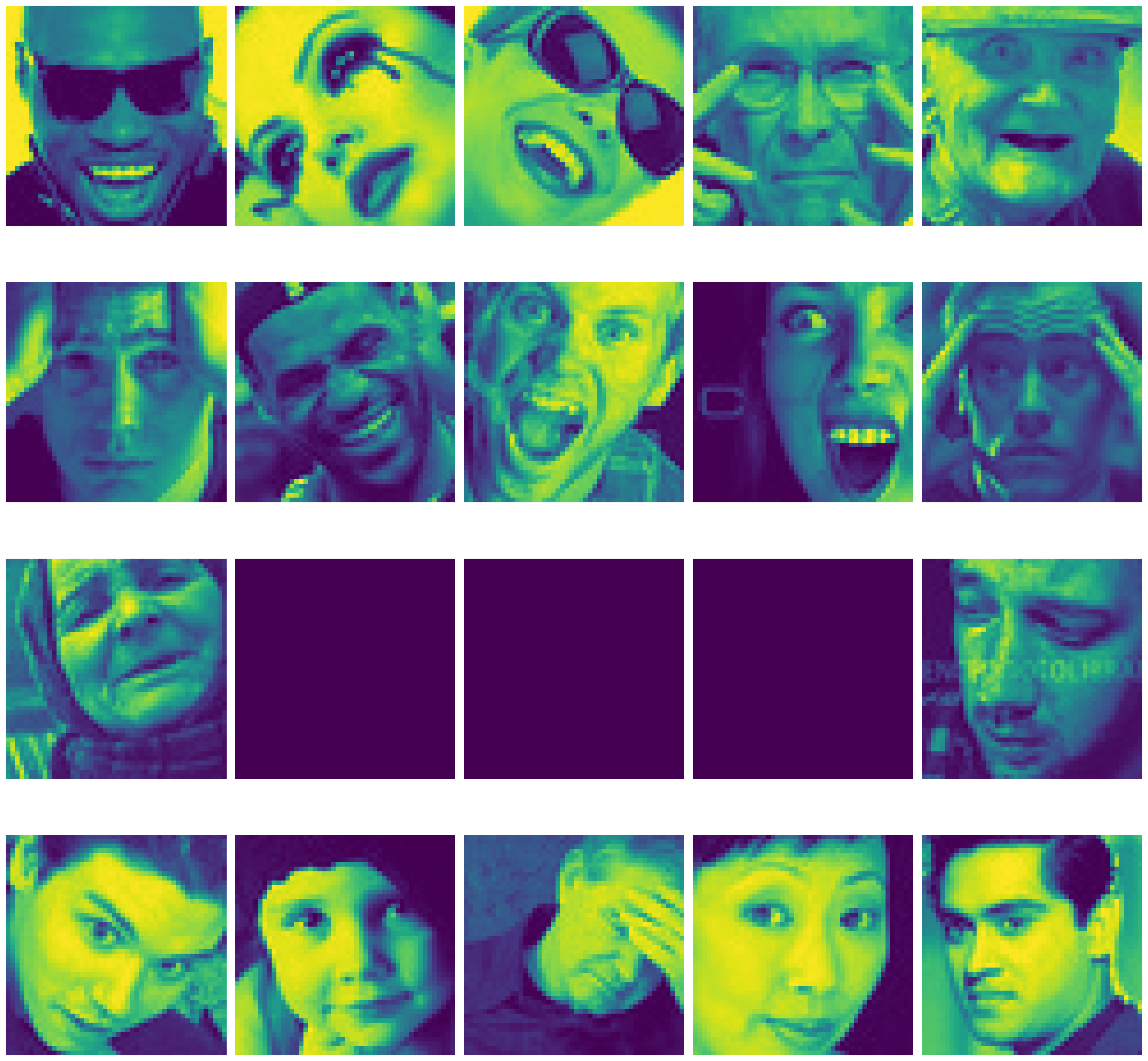}
    }
    \subfigure[\texttt{FER} Low  ]{
        \includegraphics[width=0.22\linewidth]{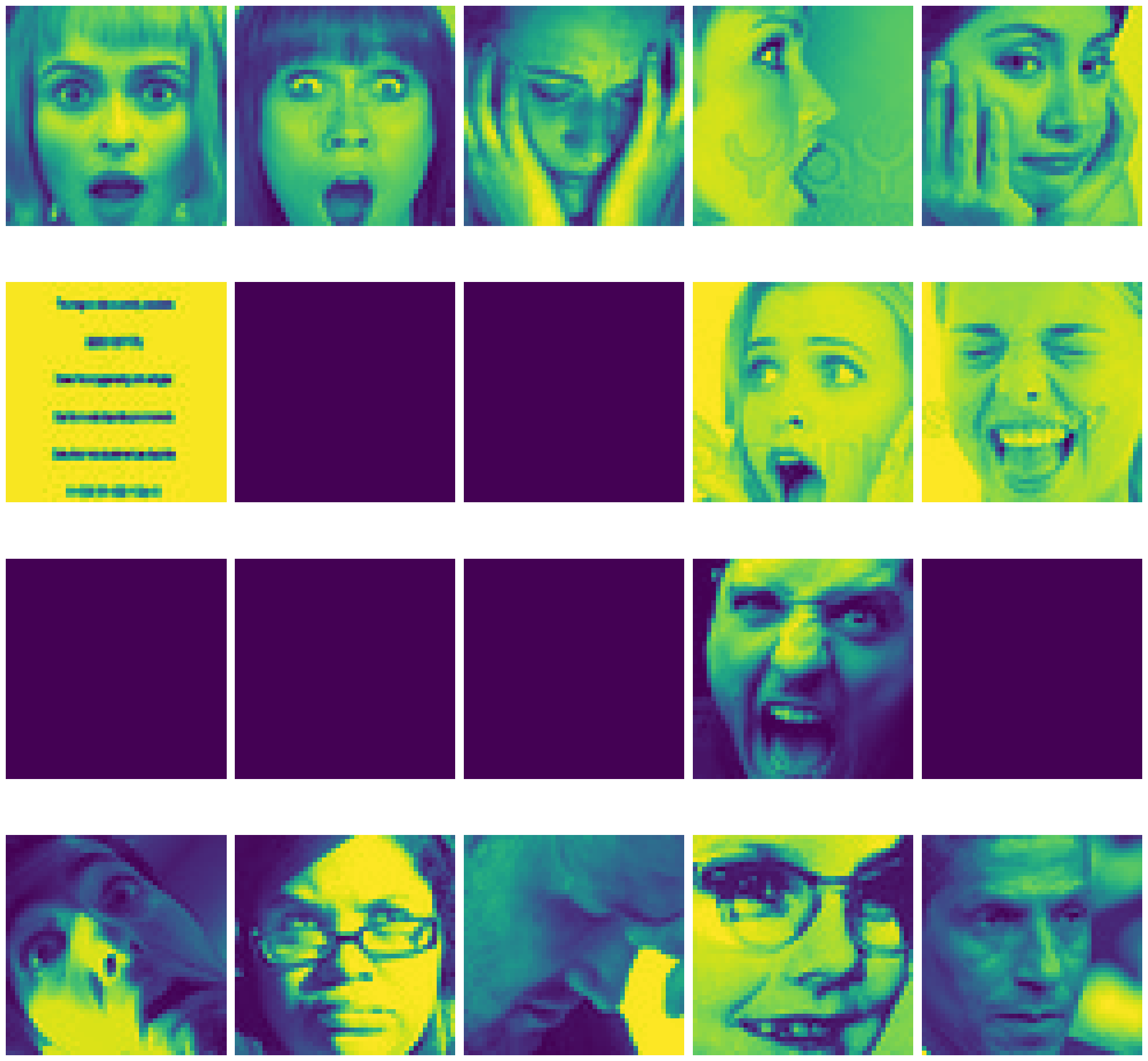}
    }
    % \caption{Comparisons of images with high and low linear/nonlinear leverage scores across \texttt{SVHN}, \texttt{NOTMNIST}, \texttt{QD}, and \texttt{FER} datasets. ``High'' and ``Low.'' refer, respectively, to the images with the highest and lowest scores. In subfigures (a)-(j), the top row represents the images selected using nonlinear leverage score in \cref{def:lev_score}, while the bottom row represents those using the linear leverage score of the original dataset. 
    % It can clearly be seen  that samples with higher nonlinear leverage scores contain distinct patterns and are harder to classify, while those with lower scores are straightforward. In contrast, standard linear leverage scores select more less meaningful/insightful samples. Notably, our method method identifies mislabeled and noisy samples with high scores, which represent outliers. 
    % % In \texttt{FER}, the top rows represent model trained with $\approx$75\% accuracy in training data, followed by $\approx$65\%, $\approx$55\% accuracies and the last row represents linear leverage scores.}
    % In \texttt{FER}, the top three rows show models at $\approx$75\%, $\approx$65\%, $\approx$55\% training accuracy, while the bottom row shows linear leverage scores.
    % The under-trained models highlight blank images or extreme expressions, whereas fully trained models detect subtler emotions and facial characteristics such as accessories, tears, and aging.
    % \label{fig:all_datasets}}
    \caption{
    Comparisons of high and low linear/nonlinear leverage scores across multiple datasets. ``High'' and ``Low'' refer to images with the highest and lowest scores, respectively. In subfigures (a)-(j), the top row shows images selected using nonlinear leverage scores (\cref{def:lev_score}), while the bottom row uses linear leverage scores. Samples with higher nonlinear scores contain distinct patterns and are harder to classify, while lower scores correspond to straightforward samples. In contrast, linear scores select less insightful samples. Notably, our method identifies mislabeled and noisy samples (outliers) with high scores. In \texttt{FER}, the top three rows represent the scores calculated at  $\approx$75\% (final), $\approx$65\%, and $\approx$55\% (initial) training accuracy, while the bottom row uses linear scores. Under-trained models highlight extreme expressions, while fully trained models detect subtler emotions and facial features like accessories, tears, and aging.
    \label{fig:all_datasets}}
\end{figure*}

\vspace{-1mm}
\noindent
\textbf{Regression Tasks.}  % \paragraph{Quantitative Experiments.}
We compare the efficacy of three sampling methods: uniform sampling, linear importance sampling (based on leverage scores and row-norms of the original dataset), and nonlinear importance sampling as defined in \cref{def:lev_score,def:row_norm}. While computing nonlinear importance scores directly is impractical in real-world scenarios, these experiments serve to validate that such scores effectively identify critical samples. 
As shown in \cref{fig:numerical_datasets}, nonlinear importance sampling more effectively identifies impactful samples, achieving lower MSE with fewer training points compared to uniform and linear sampling. The Y‑axis for regression tasks reports relative error on a logarithmic scale, so even small vertical shifts represent substantial absolute improvements.
% \vspace{-3mm}

\noindent
\textbf{Classification Tasks.} % \paragraph{Qualitative Experiments.}
We convert the datasets into binary classification tasks, comparing visually similar (``like'') and dissimilar (``unlike'') classes, in \texttt{SVHN} and \texttt{NotMNIST}. ``like'' classes are those that are harder to distinguish without prior knowledge, while ``unlike'' classes are more easily separable. In \texttt{SVHN}, we examine two scenarios: distinguishing the clearly distinct digits ``1'' vs.\ ``0'' and the more visually similar ``1'' vs.\ ``7.'' Similarly, in \texttt{NotMNIST}, we consider letter pairs (A vs.\ B) and (B vs.\ D).  
We also analyze more complex, noisy datasets like \texttt{QD}, where we distinguish between ``Fish'' and ``Car.'' For \texttt{FER}, which groups expressions by valence (positive vs. negative) based on \cite{mollahosseini2017affectnet}, we track the evolution of important samples as the model approaches optimal parameters ($\approx$55\% at initialization, $\approx$65\%, and $\approx$75\% accuracy at termination).  

\cref{fig:all_datasets} presents images with the highest (most unique) and lowest (easiest to classify) nonlinear leverage scores for each binary task. 
The results clearly show that samples with higher nonlinear leverage scores contain distinct patterns and are harder to classify, while those with lower scores are straightforward. In contrast, standard linear leverage scores select more random and less meaningful/insightful samples. Notably, our method identifies mislabeled and noisy samples with high scores, which represent the outliers. In \texttt{FER}, under-trained models highlight blank or extreme valence, whereas trained models detect subtler emotions and facial characteristics such as accessories, tears, and aging. To demonstrate that nonlinear scores also yield robust quantitative evidence for our classification experiments, we provide additional numerical results in \cref{sec:appendix:qualitative_numerical_exp}.
Additional images supporting \cref{fig:all_datasets} for all datasets are given in \cref{sec:appendix:exp_results}.

\section{Conclusions and Further Thoughts}
We introduced a unifying framework that extends importance sampling from linear to more general nonlinear models, through the notion of the adjoint operator of a nonlinear map. This perspective yields sampling schemes with approximation guarantees analogous to linear subspace embeddings, yet it is applicable to a wide range of models. Our theoretical analysis shows that these generalized scores offer strong performance bounds and our experiments demonstrate concrete benefits in reduced training costs, improved diagnostics, and outlier detection.  %By bridging randomized linear algebra techniques with nonlinear parametrization, our work broadens the scope of importance sampling, offering a principled and efficient pathway for tackling increasingly complex machine learning tasks at scale. %We believe these findings open new avenues in large-scale, high-dimensional learning tasks.

While we did not explicitly address active learning or transfer learning scenarios in this paper, our nonlinear importance sampling approach can naturally be utilized in such contexts. In particular, the method of subsampling based on nonlinear scores can be viewed as a one-shot active learning strategy, where selecting the most informative samples significantly reduces labeling costs. Similarly, these nonlinear scores could also be beneficial in analyzing transfer learning, by effectively identifying samples from a source domain that best represent the characteristics of a target domain.

Although \cref{sec:appendix:general} outlines preliminary steps toward extending the theoretical guarantees beyond squared loss objectives, fully generalizing them remains an avenue for future work. By estimating nonlinear importance scores without the model parameters, we substantially reduce the classic “chicken-and-egg” coupling between sampling and estimation. The subsequent constrained optimization still introduces a mild dependence on the unknown solution, leaving open challenges for fully parameter-agnostic importance sampling in nonlinear settings.

As shown, in many cases the computational cost of approximating our nonlinear importance scores is comparable to that of existing sampling techniques in the linear regime, as both rely on similar fundamental operations. In such cases, the only additional overhead typically stems from computing the nonlinear dual matrix. Thus, our framework in many cases does not introduce any fundamentally new computational bottleneck beyond those already present in linear importance sampling methods. As a result, the practical runtime overhead is largely governed by implementation-level factors, such as the choice of optimization algorithm and hyperparameter settings.

\newpage
\section*{Acknowledgments}
We sincerely thank Cameron Musco for their valuable feedback and discussions. This research was partially supported by the Australian Research Council through an Industrial Transformation
Training Centre for Information Resilience (IC200100022).

\section*{Impact Statement}
This paper presents work whose goal is to advance the field of 
Machine Learning. There are many potential societal consequences 
of our work, none which we feel must be specifically highlighted here.

% In the unusual situation where you want a paper to appear in the
% references without citing it in the main text, use \nocite
\nocite{langley00}

\bibliography{references}

\begin{thebibliography}{72}
\providecommand{\natexlab}[1]{#1}
\providecommand{\url}[1]{\texttt{#1}}
\expandafter\ifx\csname urlstyle\endcsname\relax
  \providecommand{\doi}[1]{doi: #1}\else
  \providecommand{\doi}{doi: \begingroup \urlstyle{rm}\Url}\fi

\bibitem[Apers et~al.(2024)Apers, Gribling, and Sidford]{apers2024computing}
Apers, S., Gribling, S., and Sidford, A.
\newblock On computing approximate lewis weights.
\newblock \emph{arXiv preprint arXiv:2404.02881}, 2024.

\bibitem[Avron et~al.(2010)Avron, Maymounkov, and Toledo]{avron2010blendenpik}
Avron, H., Maymounkov, P., and Toledo, S.
\newblock Blendenpik: Supercharging lapack's least-squares solver.
\newblock \emph{SIAM Journal on Scientific Computing}, 32\penalty0 (3):\penalty0 1217--1236, 2010.

\bibitem[Avron et~al.(2017)Avron, Kapralov, Musco, Musco, Velingker, and Zandieh]{avron2017random}
Avron, H., Kapralov, M., Musco, C., Musco, C., Velingker, A., and Zandieh, A.
\newblock Random fourier features for kernel ridge regression: Approximation bounds and statistical guarantees.
\newblock In \emph{International conference on machine learning}, pp.\  253--262. PMLR, 2017.

\bibitem[Avron et~al.(2019)Avron, Kapralov, Musco, Musco, Velingker, and Zandieh]{avron2019universal}
Avron, H., Kapralov, M., Musco, C., Musco, C., Velingker, A., and Zandieh, A.
\newblock A universal sampling method for reconstructing signals with simple fourier transforms.
\newblock In \emph{Proceedings of the 51st Annual ACM SIGACT Symposium on Theory of Computing}, pp.\  1051--1063, 2019.

\bibitem[Bietti et~al.(2022)Bietti, Bruna, Sanford, and Song]{bietti2022learning}
Bietti, A., Bruna, J., Sanford, C., and Song, M.~J.
\newblock Learning single-index models with shallow neural networks.
\newblock \emph{Advances in Neural Information Processing Systems}, 35:\penalty0 9768--9783, 2022.

\bibitem[Bourgain et~al.(1989)Bourgain, Lindenstrauss, and Milman]{bourgain1989approximation}
Bourgain, J., Lindenstrauss, J., and Milman, V.
\newblock Approximation of zonoids by zonotopes.
\newblock \emph{Acta Mathematica}, 162:\penalty0 73 – 141, 1989.

\bibitem[Bulatov(2011)]{bulatov2011notmnist}
Bulatov, Y.
\newblock Notmnist dataset.
\newblock \emph{Google (Books/OCR), Tech. Rep.[Online]. Available: http://yaroslavvb. blogspot. it/2011/09/notmnist-dataset. html}, 2:\penalty0 4, 2011.

\bibitem[Bur{\'y}{\v{s}}kov{\'a}(1998)]{buryvskova1998some}
Bur{\'y}{\v{s}}kov{\'a}, V.
\newblock Some properties of nonlinear adjoint operators.
\newblock \emph{The Rocky Mountain journal of mathematics}, 28\penalty0 (1):\penalty0 41--59, 1998.

\bibitem[Can{\'e}vet et~al.(2016)Can{\'e}vet, Jose, and Fleuret]{canevet2016importance}
Can{\'e}vet, O., Jose, C., and Fleuret, F.
\newblock Importance sampling tree for large-scale empirical expectation.
\newblock In \emph{International Conference on Machine Learning}, pp.\  1454--1462. PMLR, 2016.

\bibitem[Chen \& Derezinski(2021)Chen and Derezinski]{chen2021query}
Chen, X. and Derezinski, M.
\newblock Query complexity of least absolute deviation regression via robust uniform convergence.
\newblock In \emph{Conference on Learning Theory}, pp.\  1144--1179. PMLR, 2021.

\bibitem[Chen \& Price(2019)Chen and Price]{chen2019active}
Chen, X. and Price, E.
\newblock Active regression via linear-sample sparsification.
\newblock In \emph{Conference on Learning Theory}, pp.\  663--695. PMLR, 2019.

\bibitem[Clarkson \& Woodruff(2017)Clarkson and Woodruff]{clarkson2017low}
Clarkson, K.~L. and Woodruff, D.~P.
\newblock Low-rank approximation and regression in input sparsity time.
\newblock \emph{Journal of the ACM (JACM)}, 63\penalty0 (6):\penalty0 1--45, 2017.

\bibitem[Cohen \& DeVore(2015)Cohen and DeVore]{cohen2015approximation}
Cohen, A. and DeVore, R.
\newblock Approximation of high-dimensional parametric pdes.
\newblock \emph{Acta Numerica}, 24:\penalty0 1--159, 2015.

\bibitem[Cohen \& Peng(2015)Cohen and Peng]{cohen2015lp}
Cohen, M.~B. and Peng, R.
\newblock Lp row sampling by lewis weights.
\newblock In \emph{Proceedings of the forty-seventh annual ACM symposium on Theory of computing}, pp.\  183--192, 2015.

\bibitem[Cohen et~al.(2017)Cohen, Musco, and Musco]{cohen2017input}
Cohen, M.~B., Musco, C., and Musco, C.
\newblock Input sparsity time low-rank approximation via ridge leverage score sampling.
\newblock In \emph{Proceedings of the Twenty-Eighth Annual ACM-SIAM Symposium on Discrete Algorithms}, pp.\  1758--1777. SIAM, 2017.

\bibitem[Derezi{\'n}ski \& Mahoney(2024)Derezi{\'n}ski and Mahoney]{derezinski2024recent}
Derezi{\'n}ski, M. and Mahoney, M.~W.
\newblock Recent and upcoming developments in randomized numerical linear algebra for machine learning.
\newblock In \emph{Proceedings of the 30th ACM SIGKDD Conference on Knowledge Discovery and Data Mining}, pp.\  6470--6479, 2024.

\bibitem[Drineas \& Mahoney(2018)Drineas and Mahoney]{drineas2018lectures}
Drineas, P. and Mahoney, M.~W.
\newblock Lectures on randomized numerical linear algebra.
\newblock \emph{The Mathematics of Data}, 25\penalty0 (1), 2018.

\bibitem[Drineas et~al.(2006{\natexlab{a}})Drineas, Kannan, and Mahoney]{drineas2006fast}
Drineas, P., Kannan, R., and Mahoney, M.~W.
\newblock Fast monte carlo algorithms for matrices i: Approximating matrix multiplication.
\newblock \emph{SIAM Journal on Computing}, 36\penalty0 (1):\penalty0 132--157, 2006{\natexlab{a}}.

\bibitem[Drineas et~al.(2006{\natexlab{b}})Drineas, Mahoney, and Muthukrishnan]{drineas2006sampling}
Drineas, P., Mahoney, M.~W., and Muthukrishnan, S.
\newblock Sampling algorithms for $\ell_{2}$ regression and applications.
\newblock In \emph{Proceedings of the seventeenth annual ACM-SIAM symposium on Discrete algorithm}, pp.\  1127--1136, 2006{\natexlab{b}}.

\bibitem[Erd{\'e}lyi et~al.(2020)Erd{\'e}lyi, Musco, and Musco]{erdelyi2020fourier}
Erd{\'e}lyi, T., Musco, C., and Musco, C.
\newblock Fourier sparse leverage scores and approximate kernel learning.
\newblock \emph{Advances in Neural Information Processing Systems}, 33:\penalty0 109--122, 2020.

\bibitem[Eshragh et~al.(2022)Eshragh, Roosta, Nazari, and Mahoney]{eshragh2022lsar}
Eshragh, A., Roosta, F., Nazari, A., and Mahoney, M.~W.
\newblock Lsar: efficient leverage score sampling algorithm for the analysis of big time series data.
\newblock \emph{Journal of Machine Learning Research}, 23\penalty0 (22):\penalty0 1--36, 2022.

\bibitem[Fan et~al.(2011)Fan, Liao, and Mincheva]{fanhigh}
Fan, J., Liao, Y., and Mincheva, M.
\newblock High-dimensional covariance matrix estimation in approximate factor models1.
\newblock \emph{Annals of statistics}, 2011.

\bibitem[Feldman(2020)]{feldman2020core}
Feldman, D.
\newblock Core-sets: Updated survey.
\newblock \emph{Sampling techniques for supervised or unsupervised tasks}, pp.\  23--44, 2020.

\bibitem[Gajjar \& Musco(2021)Gajjar and Musco]{gajjar2021subspace}
Gajjar, A. and Musco, C.
\newblock Subspace embeddings under nonlinear transformations.
\newblock In \emph{Algorithmic Learning Theory}, pp.\  656--672. PMLR, 2021.

\bibitem[Gajjar et~al.(2023)Gajjar, Musco, and Hegde]{gajjar2023active}
Gajjar, A., Musco, C., and Hegde, C.
\newblock Active learning for single neuron models with lipschitz non-linearities.
\newblock In \emph{International Conference on Artificial Intelligence and Statistics}, pp.\  4101--4113. PMLR, 2023.

\bibitem[Gajjar et~al.(2024)Gajjar, Tai, Xingyu, Hegde, Musco, and Li]{gajjar2024agnostic}
Gajjar, A., Tai, W.~M., Xingyu, X., Hegde, C., Musco, C., and Li, Y.
\newblock Agnostic active learning of single index models with linear sample complexity.
\newblock In \emph{The Thirty Seventh Annual Conference on Learning Theory}, pp.\  1715--1754. PMLR, 2024.

\bibitem[Gittens \& Mahoney(2013)Gittens and Mahoney]{gittens2013revisiting}
Gittens, A. and Mahoney, M.
\newblock Revisiting the nystrom method for improved large-scale machine learning.
\newblock In \emph{International Conference on Machine Learning}, pp.\  567--575. PMLR, 2013.

\bibitem[Goel et~al.(2017)Goel, Kanade, Klivans, and Thaler]{pmlr-v65-goel17a}
Goel, S., Kanade, V., Klivans, A., and Thaler, J.
\newblock Reliably learning the relu in polynomial time.
\newblock In Kale, S. and Shamir, O. (eds.), \emph{Proceedings of the 2017 Conference on Learning Theory}, volume~65 of \emph{Proceedings of Machine Learning Research}, pp.\  1004--1042. PMLR, 07--10 Jul 2017.
\newblock URL \url{https://proceedings.mlr.press/v65/goel17a.html}.

\bibitem[Goodfellow et~al.(2013)Goodfellow, Erhan, Carrier, Courville, Mirza, Hamner, Cukierski, Tang, Thaler, Lee, et~al.]{goodfellow2013challenges}
Goodfellow, I.~J., Erhan, D., Carrier, P.~L., Courville, A., Mirza, M., Hamner, B., Cukierski, W., Tang, Y., Thaler, D., Lee, D.-H., et~al.
\newblock Challenges in representation learning: A report on three machine learning contests.
\newblock In \emph{Neural information processing: 20th international conference, ICONIP 2013, daegu, korea, november 3-7, 2013. Proceedings, Part III 20}, pp.\  117--124. Springer, 2013.

\bibitem[Ha \& Eck(2018)Ha and Eck]{ha2018neural}
Ha, D. and Eck, D.
\newblock A neural representation of sketch drawings.
\newblock In \emph{International Conference on Learning Representations}, 2018.

\bibitem[Har-Peled \& Mazumdar(2004)Har-Peled and Mazumdar]{har2004coresets}
Har-Peled, S. and Mazumdar, S.
\newblock On coresets for k-means and k-median clustering.
\newblock In \emph{Proceedings of the thirty-sixth annual ACM symposium on Theory of computing}, pp.\  291--300, 2004.

\bibitem[H{\"a}rdle et~al.(2004)H{\"a}rdle, M{\"u}ller, Sperlich, Werwatz, et~al.]{hardle2004nonparametric}
H{\"a}rdle, W., M{\"u}ller, M., Sperlich, S., Werwatz, A., et~al.
\newblock \emph{Nonparametric and semiparametric models}, volume~1.
\newblock Springer, 2004.

\bibitem[Hristache et~al.(2001)Hristache, Juditsky, and Spokoiny]{hristache2001direct}
Hristache, M., Juditsky, A., and Spokoiny, V.
\newblock Direct estimation of the index coefficient in a single-index model.
\newblock \emph{Annals of Statistics}, pp.\  595--623, 2001.

\bibitem[Huggins et~al.(2016)Huggins, Campbell, and Broderick]{huggins2016coresets}
Huggins, J., Campbell, T., and Broderick, T.
\newblock Coresets for scalable bayesian logistic regression.
\newblock \emph{Advances in neural information processing systems}, 29, 2016.

\bibitem[Iwen et~al.(2021)Iwen, Needell, Rebrova, and Zare]{iwen2021lower}
Iwen, M.~A., Needell, D., Rebrova, E., and Zare, A.
\newblock Lower memory oblivious (tensor) subspace embeddings with fewer random bits: modewise methods for least squares.
\newblock \emph{SIAM Journal on Matrix Analysis and Applications}, 42\penalty0 (1):\penalty0 376--416, 2021.

\bibitem[Johnson \& Schechtman(2001)Johnson and Schechtman]{johnson2001finite}
Johnson, W.~B. and Schechtman, G.
\newblock Finite dimensional subspaces of lp.
\newblock \emph{Handbook of the geometry of Banach spaces}, 1:\penalty0 837--870, 2001.

\bibitem[Kakade et~al.(2011)Kakade, Kanade, Shamir, and Kalai]{kakade2011efficient}
Kakade, S.~M., Kanade, V., Shamir, O., and Kalai, A.
\newblock Efficient learning of generalized linear and single index models with isotonic regression.
\newblock \emph{Advances in Neural Information Processing Systems}, 24, 2011.

\bibitem[Katharopoulos \& Fleuret(2018)Katharopoulos and Fleuret]{katharopoulos2018not}
Katharopoulos, A. and Fleuret, F.
\newblock Not all samples are created equal: Deep learning with importance sampling.
\newblock In \emph{International conference on machine learning}, pp.\  2525--2534. PMLR, 2018.

\bibitem[Langberg \& Schulman(2010)Langberg and Schulman]{langberg2010universal}
Langberg, M. and Schulman, L.~J.
\newblock Universal $\varepsilon$-approximators for integrals.
\newblock In \emph{Proceedings of the twenty-first annual ACM-SIAM symposium on Discrete Algorithms}, pp.\  598--607. SIAM, 2010.

\bibitem[Langley(2000)]{langley00}
Langley, P.
\newblock Crafting papers on machine learning.
\newblock In Langley, P. (ed.), \emph{Proceedings of the 17th International Conference on Machine Learning (ICML 2000)}, pp.\  1207--1216, Stanford, CA, 2000. Morgan Kaufmann.

\bibitem[Lantz(2019)]{lantz2019machine}
Lantz, B.
\newblock \emph{Machine learning with R: expert techniques for predictive modeling}.
\newblock Packt publishing ltd, 2019.

\bibitem[Liu \& Lee(2017)Liu and Lee]{liu2017black}
Liu, Q. and Lee, J.
\newblock Black-box importance sampling.
\newblock In \emph{Artificial Intelligence and Statistics}, pp.\  952--961. PMLR, 2017.

\bibitem[Liu et~al.(2024)Liu, Wang, Zhong, Xu, Zha, Tang, Jiang, Zhou, Chaudhary, Xu, et~al.]{liu2024winner}
Liu, Z., Wang, G., Zhong, S.~H., Xu, Z., Zha, D., Tang, R.~R., Jiang, Z.~S., Zhou, K., Chaudhary, V., Xu, S., et~al.
\newblock Winner-take-all column row sampling for memory efficient adaptation of language model.
\newblock \emph{Advances in Neural Information Processing Systems}, 36, 2024.

\bibitem[Lucic et~al.(2018)Lucic, Faulkner, Krause, and Feldman]{lucic2018training}
Lucic, M., Faulkner, M., Krause, A., and Feldman, D.
\newblock Training gaussian mixture models at scale via coresets.
\newblock \emph{Journal of Machine Learning Research}, 18\penalty0 (160):\penalty0 1--25, 2018.

\bibitem[Mahoney \& Drineas(2009)Mahoney and Drineas]{mahoney2009cur}
Mahoney, M.~W. and Drineas, P.
\newblock Cur matrix decompositions for improved data analysis.
\newblock \emph{Proceedings of the National Academy of Sciences}, 106\penalty0 (3):\penalty0 697--702, 2009.

\bibitem[Mahoney et~al.(2011)]{mahoney2011randomized}
Mahoney, M.~W. et~al.
\newblock Randomized algorithms for matrices and data.
\newblock \emph{Foundations and Trends{\textregistered} in Machine Learning}, 3\penalty0 (2):\penalty0 123--224, 2011.

\bibitem[Mai et~al.(2021)Mai, Musco, and Rao]{mai2021coresets}
Mai, T., Musco, C., and Rao, A.
\newblock Coresets for classification--simplified and strengthened.
\newblock \emph{Advances in Neural Information Processing Systems}, 34:\penalty0 11643--11654, 2021.

\bibitem[Martinsson \& Tropp(2020)Martinsson and Tropp]{martinsson2020randomized}
Martinsson, P.-G. and Tropp, J.~A.
\newblock Randomized numerical linear algebra: Foundations and algorithms.
\newblock \emph{Acta Numerica}, 29:\penalty0 403--572, 2020.

\bibitem[Meng et~al.(2022)Meng, Liu, Neiswanger, Song, Burke, Lobell, and Ermon]{meng2022count}
Meng, C., Liu, E., Neiswanger, W., Song, J., Burke, M., Lobell, D., and Ermon, S.
\newblock Is-count: Large-scale object counting from satellite images with covariate-based importance sampling.
\newblock In \emph{Proceedings of the AAAI Conference on Artificial Intelligence}, volume~36, pp.\  12034--12042, 2022.

\bibitem[Meng \& Mahoney(2013)Meng and Mahoney]{meng2013low}
Meng, X. and Mahoney, M.~W.
\newblock Low-distortion subspace embeddings in input-sparsity time and applications to robust linear regression.
\newblock In \emph{Proceedings of the forty-fifth annual ACM symposium on Theory of computing}, pp.\  91--100, 2013.

\bibitem[Mirzasoleiman et~al.(2020)Mirzasoleiman, Bilmes, and Leskovec]{mirzasoleiman2020coresets}
Mirzasoleiman, B., Bilmes, J., and Leskovec, J.
\newblock Coresets for data-efficient training of machine learning models.
\newblock In \emph{International Conference on Machine Learning}, pp.\  6950--6960. PMLR, 2020.

\bibitem[Mollahosseini et~al.(2017)Mollahosseini, Hasani, and Mahoor]{mollahosseini2017affectnet}
Mollahosseini, A., Hasani, B., and Mahoor, M.~H.
\newblock Affectnet: A database for facial expression, valence, and arousal computing in the wild.
\newblock \emph{IEEE Transactions on Affective Computing}, 10\penalty0 (1):\penalty0 18--31, 2017.

\bibitem[Munteanu et~al.(2018)Munteanu, Schwiegelshohn, Sohler, and Woodruff]{munteanu2018coresets}
Munteanu, A., Schwiegelshohn, C., Sohler, C., and Woodruff, D.
\newblock On coresets for logistic regression.
\newblock \emph{Advances in Neural Information Processing Systems}, 31, 2018.

\bibitem[Murray et~al.(2023)Murray, Demmel, Mahoney, Erichson, Melnichenko, Malik, Grigori, Luszczek, Derezi{\'n}ski, Lopes, et~al.]{murray2023randomized}
Murray, R., Demmel, J., Mahoney, M.~W., Erichson, N.~B., Melnichenko, M., Malik, O.~A., Grigori, L., Luszczek, P., Derezi{\'n}ski, M., Lopes, M.~E., et~al.
\newblock Randomized numerical linear algebra: A perspective on the field with an eye to software.
\newblock \emph{arXiv preprint arXiv:2302.11474}, 2023.

\bibitem[Musco et~al.(2022)Musco, Musco, Woodruff, and Yasuda]{musco2022active}
Musco, C., Musco, C., Woodruff, D.~P., and Yasuda, T.
\newblock Active linear regression for $\ll_p$ norms and beyond.
\newblock In \emph{2022 IEEE 63rd Annual Symposium on Foundations of Computer Science (FOCS)}, pp.\  744--753. IEEE, 2022.

\bibitem[Nabian et~al.(2021)Nabian, Gladstone, and Meidani]{nabian2021efficient}
Nabian, M.~A., Gladstone, R.~J., and Meidani, H.
\newblock Efficient training of physics-informed neural networks via importance sampling.
\newblock \emph{Computer-Aided Civil and Infrastructure Engineering}, 36\penalty0 (8):\penalty0 962--977, 2021.

\bibitem[Nelson \& Nguy{\^e}n(2013)Nelson and Nguy{\^e}n]{nelson2013osnap}
Nelson, J. and Nguy{\^e}n, H.~L.
\newblock Osnap: Faster numerical linear algebra algorithms via sparser subspace embeddings.
\newblock In \emph{2013 ieee 54th annual symposium on foundations of computer science}, pp.\  117--126. IEEE, 2013.

\bibitem[Netzer et~al.(2011)Netzer, Wang, Coates, Bissacco, Wu, Ng, et~al.]{netzer2011reading}
Netzer, Y., Wang, T., Coates, A., Bissacco, A., Wu, B., Ng, A.~Y., et~al.
\newblock Reading digits in natural images with unsupervised feature learning.
\newblock In \emph{NIPS workshop on deep learning and unsupervised feature learning}, volume 2011, pp.\ ~4. Granada, 2011.
\newblock URL \url{http://ufldl.stanford.edu/housenumbers}.

\bibitem[O’Leary-Roseberry et~al.(2022)O’Leary-Roseberry, Villa, Chen, and Ghattas]{o2022derivative}
O’Leary-Roseberry, T., Villa, U., Chen, P., and Ghattas, O.
\newblock Derivative-informed projected neural networks for high-dimensional parametric maps governed by pdes.
\newblock \emph{Computer Methods in Applied Mechanics and Engineering}, 388:\penalty0 114199, 2022.

\bibitem[Pace \& Barry(1997)Pace and Barry]{pace1997sparse}
Pace, R.~K. and Barry, R.
\newblock Sparse spatial autoregressions.
\newblock \emph{Statistics \& Probability Letters}, 33\penalty0 (3):\penalty0 291--297, 1997.

\bibitem[Paschou et~al.(2007)Paschou, Ziv, Burchard, Choudhry, Rodriguez-Cintron, Mahoney, and Drineas]{paschou2007pca}
Paschou, P., Ziv, E., Burchard, E.~G., Choudhry, S., Rodriguez-Cintron, W., Mahoney, M.~W., and Drineas, P.
\newblock Pca-correlated snps for structure identification in worldwide human populations.
\newblock \emph{PLoS genetics}, 3\penalty0 (9):\penalty0 e160, 2007.

\bibitem[Samadian et~al.(2020)Samadian, Pruhs, Moseley, Im, and Curtin]{samadian2020unconditional}
Samadian, A., Pruhs, K., Moseley, B., Im, S., and Curtin, R.
\newblock Unconditional coresets for regularized loss minimization.
\newblock In \emph{International Conference on Artificial Intelligence and Statistics}, pp.\  482--492. PMLR, 2020.

\bibitem[Sarlos(2006)]{sarlos2006improved}
Sarlos, T.
\newblock Improved approximation algorithms for large matrices via random projections.
\newblock In \emph{2006 47th annual IEEE symposium on foundations of computer science (FOCS'06)}, pp.\  143--152. IEEE, 2006.

\bibitem[Schechter(1996)]{schechter1996handbook}
Schechter, E.
\newblock \emph{Handbook of Analysis and its Foundations}.
\newblock Academic Press, 1996.

\bibitem[Scherpen \& Gray(2002)Scherpen and Gray]{scherpen2002nonlinear}
Scherpen, J.~M. and Gray, W.~S.
\newblock Nonlinear hilbert adjoints: Properties and applications to hankel singular value analysis.
\newblock \emph{Nonlinear Analysis: Theory, Methods \& Applications}, 51\penalty0 (5):\penalty0 883--901, 2002.

\bibitem[Stich et~al.(2017)Stich, Raj, and Jaggi]{stich2017safe}
Stich, S.~U., Raj, A., and Jaggi, M.
\newblock Safe adaptive importance sampling.
\newblock \emph{Advances in Neural Information Processing Systems}, 30, 2017.

\bibitem[Tolochinksy et~al.(2022)Tolochinksy, Jubran, and Feldman]{tolochinksy2022generic}
Tolochinksy, E., Jubran, I., and Feldman, D.
\newblock Generic coreset for scalable learning of monotonic kernels: Logistic regression, sigmoid and more.
\newblock In \emph{International Conference on Machine Learning}, pp.\  21520--21547. PMLR, 2022.

\bibitem[Tremblay et~al.(2019)Tremblay, Barthelm{\'e}, and Amblard]{tremblay2019determinantal}
Tremblay, N., Barthelm{\'e}, S., and Amblard, P.-O.
\newblock Determinantal point processes for coresets.
\newblock \emph{Journal of Machine Learning Research}, 20\penalty0 (168):\penalty0 1--70, 2019.

\bibitem[Vershynin(2018)]{vershynin2018high}
Vershynin, R.
\newblock \emph{High-dimensional probability: An introduction with applications in data science}, volume~47.
\newblock Cambridge university press, 2018.

\bibitem[Wickham \& Sievert(2009)Wickham and Sievert]{wickham2009ggplot2}
Wickham, H. and Sievert, C.
\newblock \emph{ggplot2: elegant graphics for data analysis}, volume~10.
\newblock springer New York, 2009.

\bibitem[Woodruff et~al.(2014)]{woodruff2014sketching}
Woodruff, D.~P. et~al.
\newblock Sketching as a tool for numerical linear algebra.
\newblock \emph{Foundations and Trends{\textregistered} in Theoretical Computer Science}, 10\penalty0 (1--2):\penalty0 1--157, 2014.

\bibitem[Xu et~al.(2016)Xu, Yang, Roosta, R{\'e}, and Mahoney]{xu2016sub}
Xu, P., Yang, J., Roosta, F., R{\'e}, C., and Mahoney, M.~W.
\newblock Sub-sampled newton methods with non-uniform sampling.
\newblock \emph{Advances in Neural Information Processing Systems}, 29, 2016.

\end{thebibliography}
\bibliographystyle{icml2025}

%%%%%%%%%%%%%%%%%%%%%%%%%%%%%%%%%%%%%%%%%%%%%%%%%%%%%%%%%%%%%%%%%%%%%%%%%%%%%%%
%%%%%%%%%%%%%%%%%%%%%%%%%%%%%%%%%%%%%%%%%%%%%%%%%%%%%%%%%%%%%%%%%%%%%%%%%%%%%%%
% APPENDIX
%%%%%%%%%%%%%%%%%%%%%%%%%%%%%%%%%%%%%%%%%%%%%%%%%%%%%%%%%%%%%%%%%%%%%%%%%%%%%%%
%%%%%%%%%%%%%%%%%%%%%%%%%%%%%%%%%%%%%%%%%%%%%%%%%%%%%%%%%%%%%%%%%%%%%%%%%%%%%%%
\newpage
\appendix
\onecolumn

\section{Appendix}

\subsection{Proof of \cref{prop:f*_composit}}
\label{sec:appendix:proofs}
\begin{proof}
We first note that by Euler's homogeneous function theorem, the gradient of $h$ is positively homogeneous of degree $ \alpha - 1$, i.e., 
\begin{align*}
    \frac{\partial}{\partial \btheta} h(t \btheta) = t^{\alpha-1} \frac{\partial}{\partial \btheta} h(\btheta), \quad \forall t > 0.
\end{align*}
%Further, assume that $g(0) = 0$ (since $h(\btheta)$ is positively homogeneous, it naturally follows that $ h(\zero) = 0 $. 
We have,
\begin{align*}
    \ff^{\star}(\btheta) &= \int_{0}^{1} \frac{\partial}{\partial \btheta} f(t \btheta) \df t =\int_{0}^{1} \frac{\partial}{\partial \btheta} h(t \btheta) g^{\prime}(h(t \btheta)) \df t \\
    &= \int_{0}^{1} t^{\alpha-1}\frac{\partial}{\partial \btheta} h(\btheta) g^{\prime}(t^{\alpha}h(\btheta)) \df t 
    = \left( \int_{0}^{1} t^{\alpha-1} g^{\prime}(t^{\alpha}h(\btheta)) \df t\right) \frac{\partial}{\partial \btheta} h(\btheta). \tageq\label{eq:zero}
\end{align*}
Now letting 
\begin{align*}
    t  = \left(\frac{s}{h(\btheta)}\right)^{1/\alpha},
\end{align*}
gives
\begin{align*}
    \df t = \frac{1}{\alpha h(\btheta)}  \left(\frac{s}{h(\btheta)}\right)^{(1-\alpha)/\alpha} \df s.
\end{align*}
It follows that
\begin{align*}
    \ff^{\star}(\btheta) &= \int_{0}^{h(\btheta)} \left(\frac{s}{h(\btheta)}\right)^{(\alpha-1)/\alpha} \frac{1}{\alpha h(\btheta)}  \left(\frac{s}{h(\btheta)}\right)^{(1-\alpha)/\alpha} g^{\prime}(s) \df s \hspace{0.2cm} \frac{\partial}{\partial \btheta} h(\btheta) \\
    &= \frac{\partial h(\btheta)/\partial \btheta }{\alpha h(\btheta)}\left( \int_{0}^{h(\btheta)} g^{\prime}(s) \df s\right),
    \end{align*}
and hence,
\begin{align*}
    \ff^{\star}(\btheta) = \left(\frac{g(h(\btheta)) - g(0)}{\alpha \left(h(\btheta)\right)} \right)\frac{\partial}{\partial \btheta} h(\btheta). 
\end{align*}
If $\btheta$ is such that $h(\btheta) = 0$, then from \cref{eq:zero}, we get
\begin{align*}
    \ff^{\star}(\btheta) = \left(\frac{g^{\prime}(0)}{\alpha} \right)\frac{\partial}{\partial \btheta} h(\btheta). 
\end{align*}
\end{proof}

\subsection{Norm Score Approximation for \cref{ex:nn_estimate}}
\label{sec:appendix:nn}
Consider \cref{ex:nn} with $\phi$  such that such that $c_1 \leq  (\phi(t)-\phi(0))^{2}/t^{2} \leq c_2$ for some $0 < c_1 \leq c_2 < \infty$ and for all $t\in \sT$ for some set of interest $\sT$, e.g., Swish-type or linear output layer. 
Recall that $ \btheta = [\btheta_{1},\ldots,\btheta_m] $ where $ \btheta_{j} = [a_{j},\bb_{j}] $. Also denote  $ \btheta^{\star} = [\btheta_{1}^{\star},\ldots,\btheta_{m}^{\star}] $ where $\btheta^{\star}_{j} = [a_{j}^{\star},\bb_{j}^{\star}]$. We get 
\begin{align*}
\vnorm{\ff^{\star}(\btheta)}^{2} &= \sum_{j=1}^{m} \vnorm{\rr^{\star}(\btheta_{j})}^{2} 
= \sum_{j=1}^{m} \gamma_{j}^{2} \left( 
        \left[\max\left\{\dotprod{\bb_{j},\xx},0\right\}\right]^{2} + 
        a_{j}^{2} \cdot \vnorm{\xx}^{2} \cdot \indic{\dotprod{\bb_{j},\xx} > 0}\right),
\end{align*}
where
\begin{align*}
\gamma_j \defeq 
	 \frac{\phi(a_{j} \cdot   \max\left\{\dotprod{\bb_{j},\xx},0\right\}) - \phi(0)}{2 a_{j} \cdot \max\left\{\dotprod{\bb_{j},\xx},0\right\}}.
\end{align*}
For any $\xx_{i}$, denote
\begin{align*}
    f_{i}(\btheta) = \sum_{j=1}^{m} r_{i}(\btheta_{j}) = \sum_{j=1}^{m} \phi_{i}(a_{j} \cdot   \max\left\{\dotprod{\bb_{j},\xx_{i}},0\right\}).
\end{align*}
Suppose for some $j\in\{1,2,\ldots,m\}$, $\dotprod{\bb_{j},\xx_{i}} > 0$, as otherwise $\ff^{\star}_{i}(\btheta) = \zero$. We have,
\begin{align*}
\gamma^2_j a_j^2 \vnorm{\xx_{i}}^{2} &\leq \sum_{j=1}^{m} \gamma_{j}^{2} \left( 
        \left[\max\left\{\dotprod{\bb_{j},\xx_{i}},0\right\}\right]^{2} + 
        a_{j}^{2} \cdot \vnorm{\xx_{i}}^{2} \cdot \indic{\dotprod{\bb_{j},\xx_{i}} > 0}
    \right) \leq \left( \sum_{j=1}^{m} \gamma^2_j \left( \vnorm{\bb_{j}}^{2} + a_j^2 \right) \right) \vnorm{\xx_{i}}^{2}.
\end{align*}
By assumption on $\phi$, it follows that 
\begin{align*}
c_1 a_j^2 \vnorm{\xx_{i}}^{2} &\leq \sum_{j=1}^{m} \gamma_{j}^{2} \left( 
        \left[\max\left\{\dotprod{\bb_{j},\xx_{i}},0\right\}\right]^{2} + 
        a_{j}^{2} \cdot \vnorm{\xx_{i}}^{2} \cdot \indic{\dotprod{\bb_{j},\xx_{i}} > 0}
    \right) \leq c_2 \left( \sum_{j=1}^{m} \left( \vnorm{\bb_{j}}^{2} + a_j^2 \right) \right) \vnorm{\xx_{i}}^{2}.
\end{align*}
Assume $a^{\star}_{j} \neq 0$ for all $j \in {1, 2, \ldots, m}$. 
%If $a^{\star}_{j} = 0$ for some $j$, we can simply eliminate all connections leading to $a^{\star}_{j}$ and consider a network with $m-1$ hidden units. 
Let $l > 0$ be such that $\min_{j} (a^{\star}_{j})^{2} \geq  l$ and  $0 < u < \infty $ be such that $\sum_{j=1}^{m} \left( \vnorm{\bb^{\star}_{j}}^{2} + (a^{\star}_j)^2 \right) \leq u$. Define the set 
\begin{align*}
\sC \defeq \left\{ [a_{1},\bb_{1},\ldots,a_{m},\bb_{m}] \mid  \min_{j  = 1,\ldots,m} \; (a^{\star}_{j})^{2} \geq  l, \quad \sum_{j=1}^{m} \left( \vnorm{\bb^{\star}_{j}}^{2} + (a^{\star}_j)^2 \right) \leq u  \right\}.
\end{align*}
Clearly, by construction, $\bthetas \in \sC $. It follows that that for any $\btheta \in \sC$, we have 
\begin{align*}
c_1 l  \vnorm{\xx_{i}}^{2} &\leq \vnorm{\ff^{\star}_{i}(\btheta)}^{2} \leq c_2 u \vnorm{\xx_{i}}^{2},
\end{align*}
which in turn gives
\begin{align*}
\min\{c_1 l,1\}  \left( \vnorm{\xx_{i}}^{2} + m^2\phi_{i}(0)^{2} \right) &\leq \vnorm{\widehat{\ff}^{\star}_{i}(\btheta)}^{2} \leq \max\{c_2 u,1\} \left( \vnorm{\xx_{i}}^{2} + m^2\phi_{i}(0)^{2}\right).
\end{align*}
This implies that 
\begin{align*}
\tau_{i}(\btheta) = \frac{\vnorm{\widehat{\ff}^{\star}_{i}(\btheta)}_{2}^{2}}{\vnorm{\widehat{\FF}^{\star}(\btheta)}^{2}_{\textnormal{F}}} \leq \left(\frac{\max\{c_2 u,1\}  }{\min\{c_1 l,1\}}\right)\frac{\vnorm{\xx_i}^{2} + m^2\phi_{i}^{2}(0)}{\vnorm{\XX}^{2}_{\textnormal{F}} + m^2 \sum_{j=1}^{n} \phi_{j}^{2}(0)} = \left(\frac{\max\{c_2 u,1\}  }{\min\{c_1 l,1\}}\right) \tau_{i},
\end{align*}
where $\tau_{i}$ is the norm score for the $i\th$ row of
\begin{align*}
    \widehat{\XX} = \begin{bmatrix}
        \xx^{\T}_{1} & m\phi_{1}(0) \\ \xx^{\T}_{2} & m\phi_{2}(0)  \\ \vdots & \vdots \\ \xx^{\T}_{n} & m\phi_{n}(0) 
    \end{bmatrix} \in \real^{n \times  (d+1)}.
\end{align*}
This allows us to pick $\beta = {\min\{c_1 l,1\}}/{\max\{c_2 u,1\}}$ in \cref{thm:param_indep_02}.

\subsection{Beyond Nonlinear Least-squares.} 
\label{sec:appendix:general}
Going beyond nonlinear least-squares settings, it turns out that an extension of our adjoint based approach still applies as long as the map $f$ has suitable homogeneity properties. More precisely, consider \cref{eq:loss} and suppose $\ell$ is a nonnegative loss and $ f $ is positively homogeneous\footnote{Recall that a function $\psi$ is positively homogeneous of degree $\alpha $  if $ \psi(t \btheta) = t^{\alpha} \psi(\btheta)  $ for any $ t > 0 $; see \cite{schechter1996handbook} for more details on positive homogeneity.} of degree $\alpha $, e.g., ReLU and varaints such as Leaky ReLU. Since $\ell$ is a nonnegative loss function, we define  
\begin{align*}
	h(\btheta) \defeq \sqrt{\ell(f(\btheta))}.
\end{align*}
Analogously to \cref{def:naop}, for any given $\xx$ and $\btheta$, we can define
\begin{align*}
	&\hh^{\star}(\btheta) \defeq \int_{0}^{1} \frac{\partial}{\partial \btheta} h(t\btheta) \df t \\
    &= \int_{0}^{1} \frac{1}{2} \left(\frac{\ell^{\prime}(f(t\btheta))}{\sqrt{\ell(f(t\btheta))}}\right)\frac{\partial}{\partial \btheta} f(t\btheta) \df t \\
	&= \frac{1}{2} \left(\int_{0}^{1} t^{\alpha-1} \left(\frac{\ell^{\prime}(t^{\alpha}f(\btheta))}{\sqrt{\ell(t^{\alpha}f(\btheta))}}\right) \df t \right) \frac{\partial}{\partial \btheta} f(\btheta, \xx),
\end{align*}
where the last equality follows from Euler's homogeneous function theorem. Now, letting $ s = t^{\alpha}f(\btheta) $, we have $ \df s = \alpha t^{\alpha-1} f(\btheta) \df t $, which gives
\begin{align*}
	\hh^{\star}(\btheta) &= \frac{1}{2 \cdot \alpha f(\btheta)} \left(\int_{0}^{f(\btheta)} \frac{\ell^{\prime}(s)}{\sqrt{\ell(s)}} \df s \right) \frac{\partial}{\partial \btheta} f(\btheta) \\
	&= \left(\frac{\sqrt{\ell(f(\btheta))} - \sqrt{\ell(0)}}{\alpha f(\btheta)}  \right) \frac{\partial}{\partial \btheta} f(\btheta).
\end{align*}
Since $h(\btheta) = h(\zero) + \dotprod{\hh^{\star}_{i}(\btheta),\btheta}$, we can define
\begin{align*}
	\widehat{\hh}^{\star}(\btheta) = \begin{bmatrix}
		\hh^{\star}(\btheta) \\ h(\zero)
	\end{bmatrix},
\end{align*}
to get
\begin{align*}
	\sL(\btheta) &= \sum_{i=1}^{n} \ell(f_{i}(\btheta)) = \sum_{i=1}^{n} \left(h_{i}(\btheta)\right)^{2} \\
	&= \sum_{i=1}^{n} \left(\dotprod{\widehat{\hh}^{\star}_{i}(\btheta),\widehat{\btheta}}\right)^{2} 
	 = \vnorm{\widehat{\HH}^{\star}(\btheta) \widehat{\btheta}}^{2},
\end{align*}
where $\widehat{\HH}^{\star}(\btheta) $ is defined analogously to \cref{{def:dual_X}}. Now, similar to the case of nonlinear least-squares, importance  sampling according to leverage scores or norms of $\{\widehat{\hh}^{\star}_{i}(\btheta)\}$ give sampling approximations of the form \cref{eq:approx}.

\subsection{Construction of the $\varepsilon$-Net in \cref{sec:lower_bound}}
\label{sec:appendix:e_net} 
Let $\btheta^{\star}$ denote a solution to \cref{eq:loss}, and consider a compact ball with radius $R$, chosen large enough to contain $\btheta^{\star}$. Let $\sB^{\star}_{R}$ denote this ball. For any $\varepsilon \in (0,1)$, we pick a discrete subset $\sN_{\varepsilon} \subseteq \sB^{\star}_{R}$ such that, for every $\btheta \in \sB^{\star}_{R}$, there exists at least one $\btheta^{\prime} \in \sN_{\varepsilon}$ satisfying
\begin{align*}
    \|\btheta - \btheta^{\prime}\|_{2} \leq \varepsilon R.
\end{align*}
The construction of an $\varepsilon$-net reduces an uncountable continuous set to a finite covering set; see \cref{fig:epsilon_net}. For completeness, we provide some details here. The reader is encouraged to consult references such as \citet{woodruff2014sketching,vershynin2018high} for further discussions and details. 

To find an upper bound on the cardinality of the set $|\sN_{\varepsilon}|$, one can use standard volume and covering number arguments. Recall that the volume of $\sB^{\star}_{R}$ is %$\text{Vol}(\sB^{\star}_{R}) = {\pi^{p/2} R^{p}}/{\Gamma\left({p}/{2}+1\right)}$,
\begin{align*}
\text{Vol}(\sB^{\star}_{R}) = \frac{\pi^{p/2} R^{p}}{\Gamma\left(\frac{p}{2}+1\right)},
\end{align*}
where $\Gamma$ is Euler's gamma function. A bound on $|\sN_{\varepsilon}|$ can be obtained as the number of small balls with radius $\varepsilon R/2$ that cover the larger ball of radius $(1 + \frac{\varepsilon}{2})R$, which is given by the ratio of their respective volumes:
\begin{align*}
|\sN_{\varepsilon}| \geq \frac{{\left(1 + \frac{\varepsilon}{2}\right)}^p R^p}{{\left(\frac{\varepsilon R}{2}\right)}^p} = \left(1 + \frac{2}{\varepsilon}\right)^p \in \Theta\left(\frac{1}{\varepsilon^p}\right).
\end{align*}

% \begin{figure}[htbp]
%     \centering
%     \begin{tikzpicture}
%         % Draw large enclosing set (X)
%         \draw[thick] (0,0) circle (3);
%         \node at (3.2,1) {$\sB^{\star}_{R}$};

%         % Define epsilon-net points (blue dots) and their covering epsilon-balls
%         \foreach \x/\y in {0.6/1.8, -1.2/1.4, -1.4/-0.6, 1.0/-1.4, 2.0/-0.2, -0.2/0.5} {
%             \fill[blue] (\x,\y) circle (2pt);  % Blue net points
%             \draw[dashed] (\x,\y) circle (1);  % Epsilon-balls
%         }

%         % Define arbitrary points in X (black dots)
%         \foreach \x/\y in {0.8/1.6, -0.8/1.2, -1.0/-0.4, 0.8/-1.6, 1.8/-0.4, 0.5/0.5} {
%             \fill[black] (\x,\y) circle (1.5pt);
%         }

%         % Labels
%         \node[blue] at (-1.4,1.8) {$\sN_{\varepsilon}$};
%         % \node[black] at (-2.0,-1.5) {Covered points};

%     \end{tikzpicture}
%     \caption{Illustration of an $\varepsilon$-net covering $\sB^{\star}_{R}$. Blue dots represent elements of $\sN_{\varepsilon}$, and dashed circles indicate the $\varepsilon$-balls covering all points in $\sB^{\star}_{R}$.}
%     \label{fig:epsilon_net}
% \end{figure}

\begin{figure}[htbp]
    \centering
    \begin{tikzpicture}

        % Define radii
        \def\R{2.5}  % Radius of B^*_R
        \def\e{0.5}  % Epsilon
        \def\OuterRadius{(1 + \e/2) * \R}  % Outer radius (1+ε/2)R
        \def\InnerRadius{\R}  % Inner radius R
        \def\CoverRadius{(\e/2) * \R}  % Covering epsilon-ball radius εR/2

        % Draw outer ball (radius (1+ε/2)R)
        \draw[thick, green!60!teal, dashed] (0,0) circle (\OuterRadius);
        % \node[green] at (3.7,1.5) {$(1+\varepsilon/2) R$};

        % Draw inner ball (B^*_R with radius R)
        \draw[thick, red] (0,0) circle (\InnerRadius);
        \node[red] at (-1,-2) {$\sB^{\star}_{R}$};

        % Draw epsilon-net covering circles (εR/2 balls)
        \foreach \x/\y in {0.8/1.4, -1.0/1.0, -1.2/-0.6, 1.2/-1.0, 1.8/-0.2, -0.2/2.2} {
            \draw[dashed,blue] (\x,\y) circle (\CoverRadius);
            \fill[blue] (\x,\y) circle (2pt);  % Blue net points
        }

        % Highlight one covering ball with label
        \draw[dashed,blue] (1.2,-1.0) circle (\CoverRadius);
        % \node at (2.0,-1.0) {$\varepsilon R/2$};

        % Draw sample points inside B^*_R
        \foreach \x/\y in {0.6/1.2, -0.8/0.8, -1.0/-0.3, 0.7/-1.2, 1.5/-0.3, 0.3/2.3} {
            \fill[black] (\x,\y) circle (1.5pt);
        }

        % Draw radius lines
        \draw[->,red] (0,0) -- (0,-\R) node[midway, left] {$R$}; % Inner radius
        \draw[->,green!60!teal,dashed] (0,0) -- ({\OuterRadius-0.23},1.2) node[pos = 1.25, above] {$(1+\varepsilon/2) R$}; % Outer radius
        % \draw[->,blue] (1.2,-1.0) -- ({1.2+\CoverRadius},-1.0) node[midway, below] {$\varepsilon R/2$}; % Small covering ball radius
        \draw[->,blue] (1.2,-1.0) -- ({1.2+\CoverRadius},-1.0) node[midway, below, xshift=-5pt] {\scriptsize $\varepsilon R/2$}; % Small covering ball radius

    \end{tikzpicture}
    \caption{Illustration of an $\varepsilon$-net covering $\sB^{\star}_{R}$. The larger circle has radius $(1+\varepsilon/2)R$, while $\sB^{\star}_{R}$ has radius $R$. Blue dots denote the $\varepsilon$-net, and dashed circles of radius $\varepsilon R/2$ cover all points.}
    \label{fig:epsilon_net}
\end{figure}
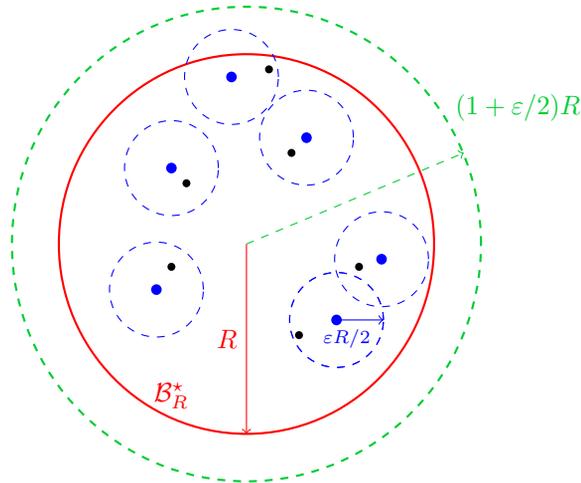

\subsection{Further Details for \cref{sec:experiments}}
\label{sec:appendix:exp_details}

\paragraph{Classification Experiments.}  % \paragraph{Qualitative Experiments.}
To carry out the experiment in an under-parameterized setting, the dataset was balanced, and the images were resized to $10 \times 10$ dimensions with a grayscale background. A fully connected MLP was trained with a linear 100-input layer connected to a hidden layer with 10 neurons and a ReLU activation unit, followed by a sigmoid output transformation function. The optimal weights were computed using PyTorch, with the Adam optimizer for 1000-5000 epochs (depending on the dataset) and BCEWithLogitsLoss(). These optimal weights were then used to calculate the nonlinear leverage scores in \cref{def:lev_score} for each data point.

\paragraph{Regression Experiments.} % \paragraph{Quantitative Experiments.}
Here, we train a single-index model using a bounded output transformation function described in \cref{ex:glm_estimate}, with $c_1 = 1$ and $c_2 = 2$. After training for 30,000 epochs, we obtain the optimal parameters and compute the nonlinear importance scores for each data point, as well as the classical linear leverage/row-norm scores using the original data matrix. We then sample training instances using a stratified strategy, proportional to these scores, and evaluate how well each sampling strategy preserves the training performance. Specifically, we solve the subproblem outlined in \cref{eq:bthetasS}. The experiment is repeated multiple times, and we measure the median log relative error between the MSE of the parameter $\bthetasS$ on the full dataset and the optimal MSE obtained by training on the full dataset, i.e., $\log (\sL(\bthetasS) - \sL(\bthetas))/\sL(\bthetas)$ as a function of the number of samples selected (the number of samples were chosen depending on the total size of the dataset). 
%In the case of classification, the parameters trained on the sampled daset was evaluated on the full dataset to obtain the accuracies.

\subsection{Quantitative performance for Qualitative datasets}
\label{sec:appendix:qualitative_numerical_exp}
In this section, we show that the datasets presented in \cref{fig:all_datasets} serve not only as qualitative illustrations but also as quantitative evidence that sampling based on nonlinear scores helps reduce the training mean squared error (MSE) loss more effectively than alternatives. This reduction is more clearly observed when the loss is plotted on a logarithmic scale. We support this result with experiments on two datasets, SVHN (digits 1 vs. 0) and QD, comparing nonlinear leverage score sampling to linear leverage score sampling and uniform sampling.

\begin{figure}[H]
    \centering
    
    \subfigure[SVHN (1 vs 0) Dataset]{
        \includegraphics[width=.4\linewidth]{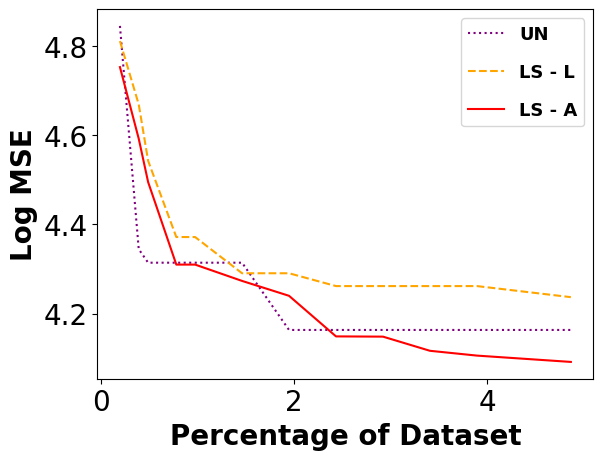}
    } \hspace{4mm}
    \subfigure[Quick Draw dataset]{
        \includegraphics[width=.4\linewidth]{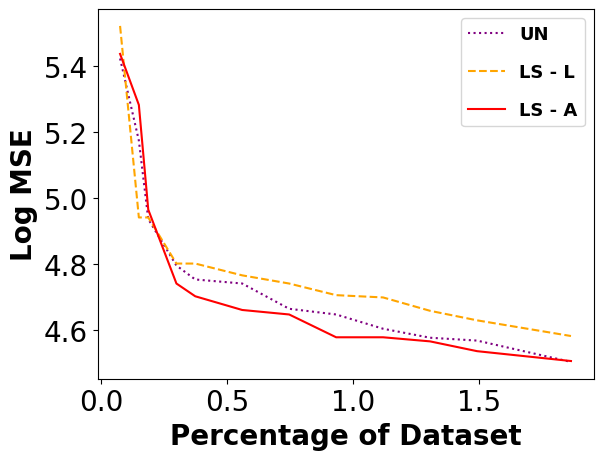}
    }
    \caption{
    Illustration of quantitative results on datasets used in \cref{fig:all_datasets}, SVHN \& QD.  The Y-axis shows Log({\text{MSE}}) on training data against sample size (in percentage of total data). ``LS' and ``UN'' denote Leverage Scores and Uniform Sampling schemas, respectively, with ``L'' and ``A'' indicating linear and adjoint-based nonlinear variants.
    \label{fig:MSE_qualitative_datasets}}
\end{figure}

\subsection{Additional images from \cref{sec:experiments}}
\label{sec:appendix:exp_results}
To facilitate further comparisons, we provide an additional 50 images with the highest and lowest leverage scores for each dataset from the classification experiments in \cref{sec:experiments}.
\begin{figure*}[!ht]
    \centering

    %----------------- Row 1: \texttt{SVHN} (1 vs 0) and (1 vs 7) -----------------%
    \subfigure[\texttt{SVHN} High (1 vs 0) Nonlinear Leverage Scores]{
        \includegraphics[width=0.48\linewidth]{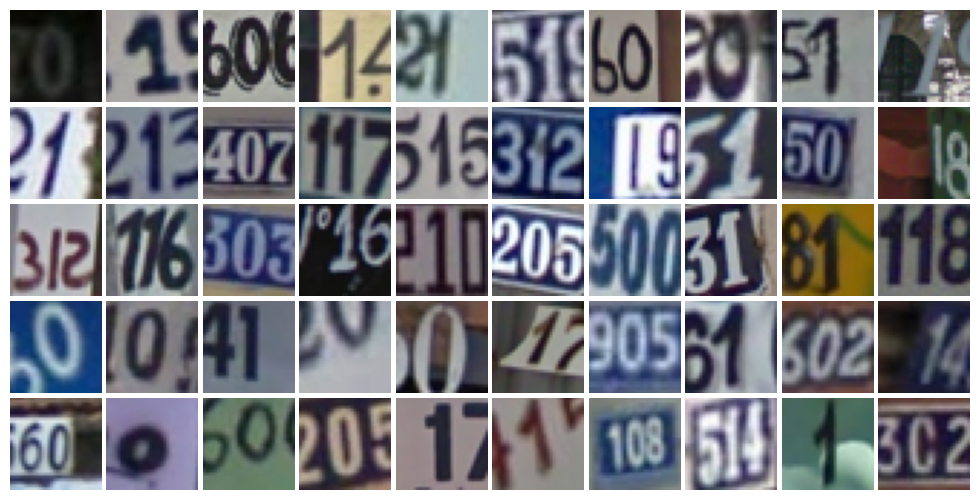}
    }
    \subfigure[\texttt{SVHN} Low   (1 vs 0) Nonlinear Leverage Scores]{
        \includegraphics[width=0.48\linewidth]{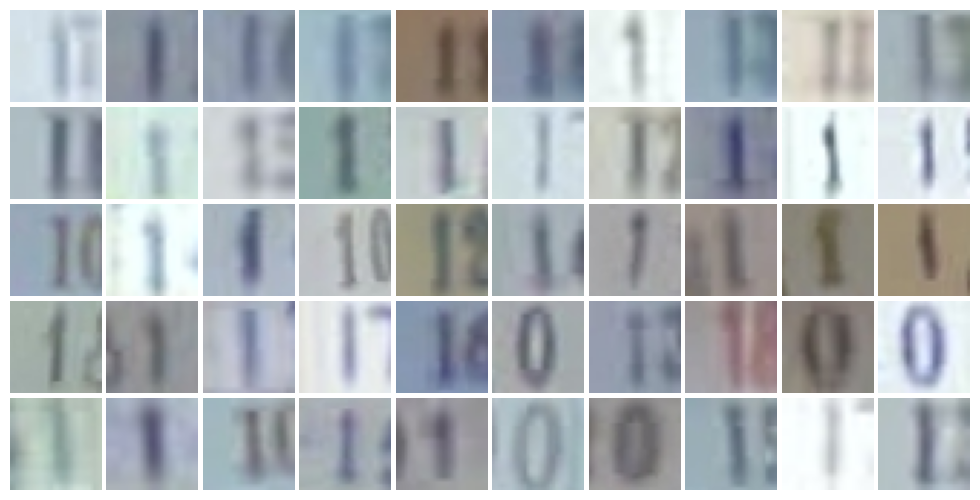}
    }
    \subfigure[\texttt{SVHN} High (1 vs 0) Linear Leverage Scores]{
        \includegraphics[width=0.48\linewidth]{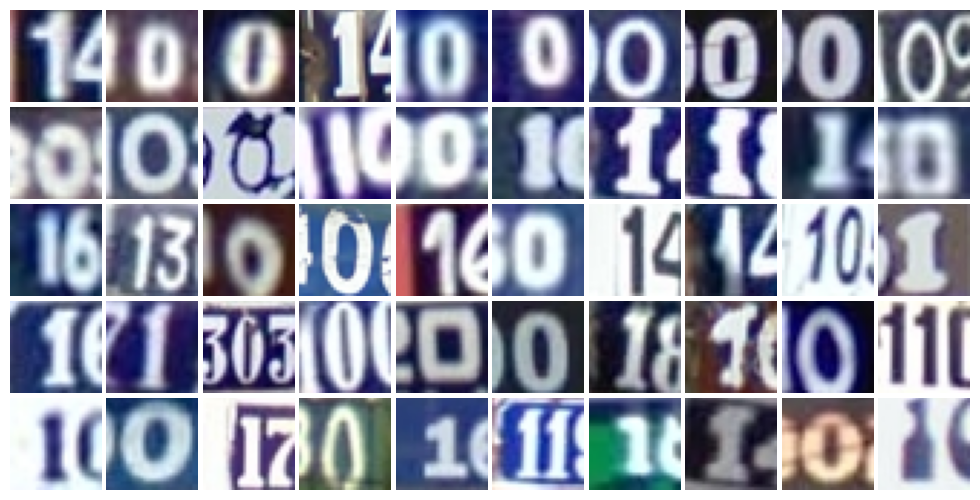}
    }
    \subfigure[\texttt{SVHN} Low   (1 vs 0) Linear Leverage Scores]{
        \includegraphics[width=0.48\linewidth]{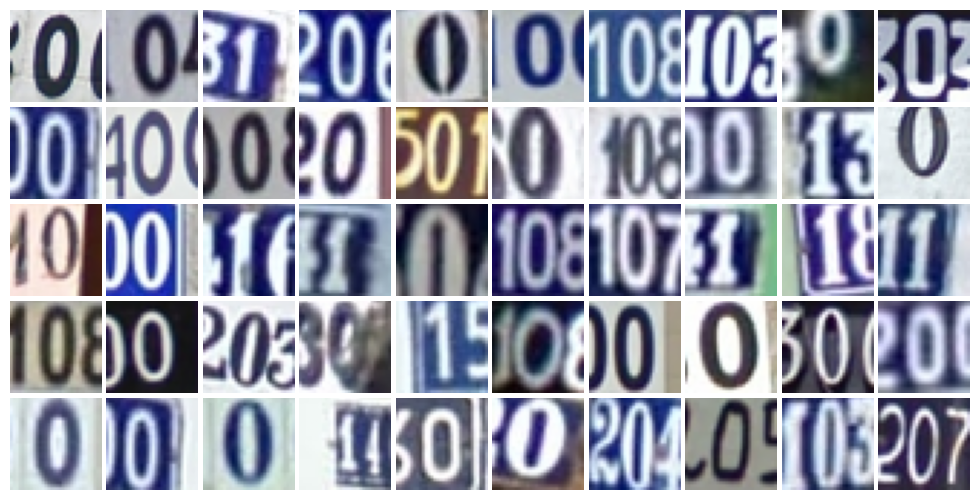}
    }
\end{figure*}
\begin{figure*}
     \subfigure[\texttt{SVHN} High (1 vs 7) Nonlinear Leverage Scores]{
        \includegraphics[width=0.48\linewidth]{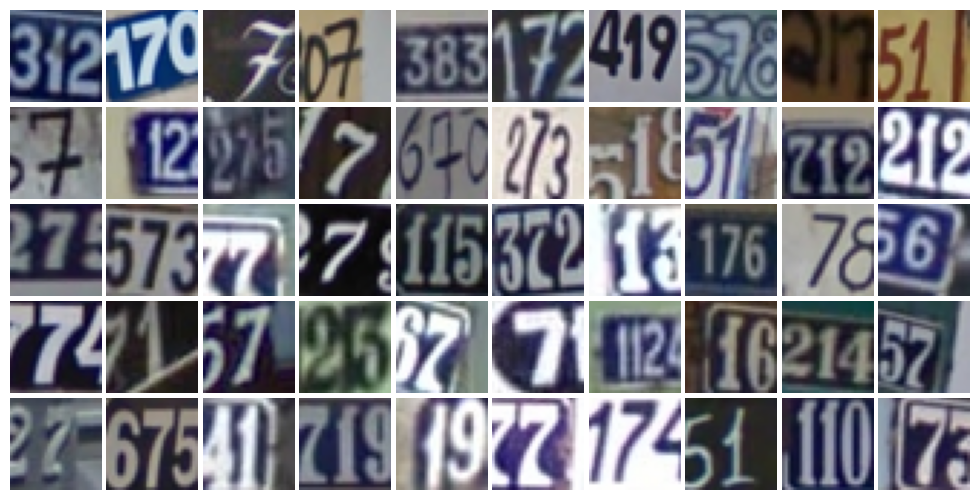}
    }
    \subfigure[\texttt{SVHN} Low   (1 vs 7) Nonlinear Leverage Scores]{
        \includegraphics[width=0.48\linewidth]{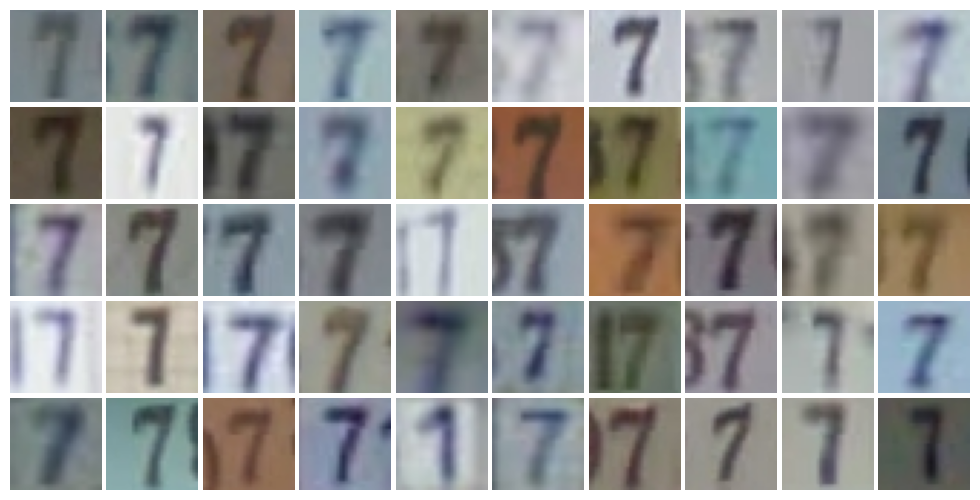}
    }
    \subfigure[\texttt{SVHN} High (1 vs 7) Linear Leverage Scores]{
        \includegraphics[width=0.48\linewidth]{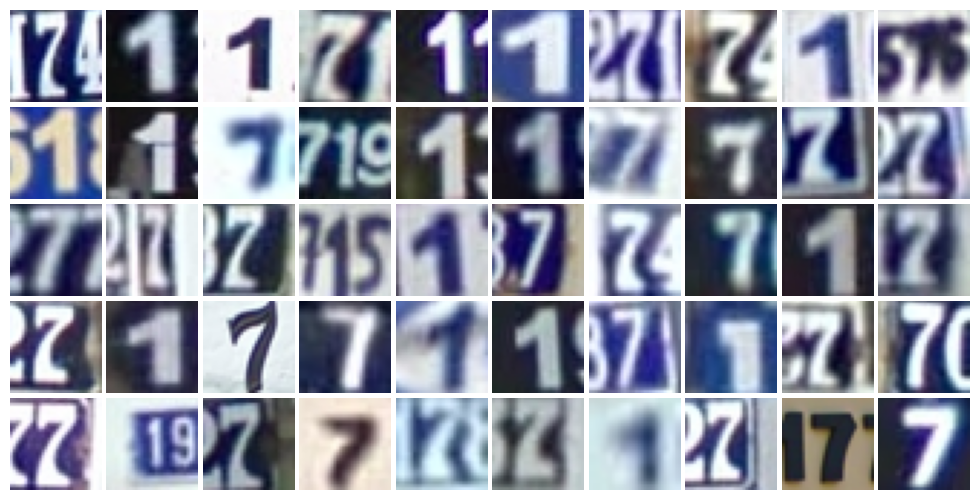}
    }
    \subfigure[\texttt{SVHN} Low   (1 vs 7) Linear Leverage Scores]{
        \includegraphics[width=0.48\linewidth]{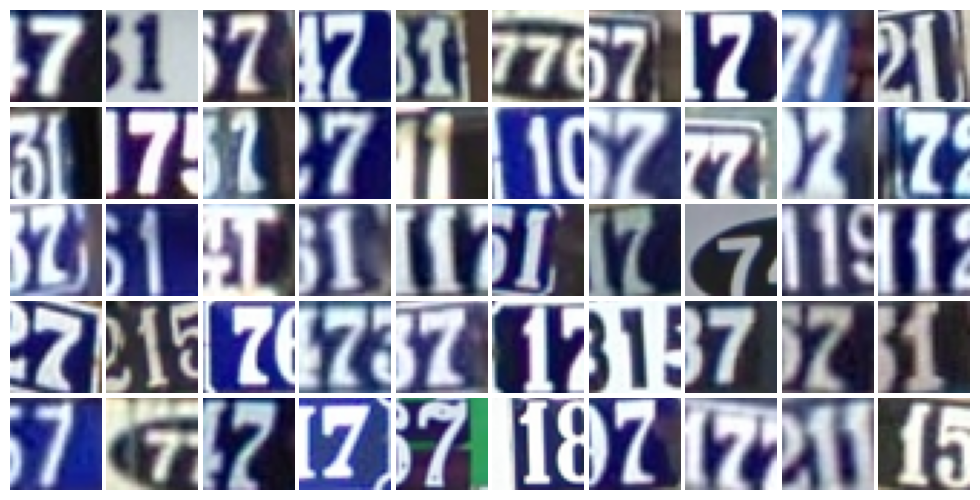}
    }
    \caption{Top 50 images with the highest and lowest nonlinear leverage scores in each grouping for the \texttt{SVHN} dataset.}
    \label{fig:all_datasets_50_SVHN}
\end{figure*}
    % \vspace{1em}
\begin{figure*}
     \subfigure[\texttt{NOTMNIST} High (A vs B) Nonlinear Leverage Scores]{
        \includegraphics[width=0.48\linewidth]{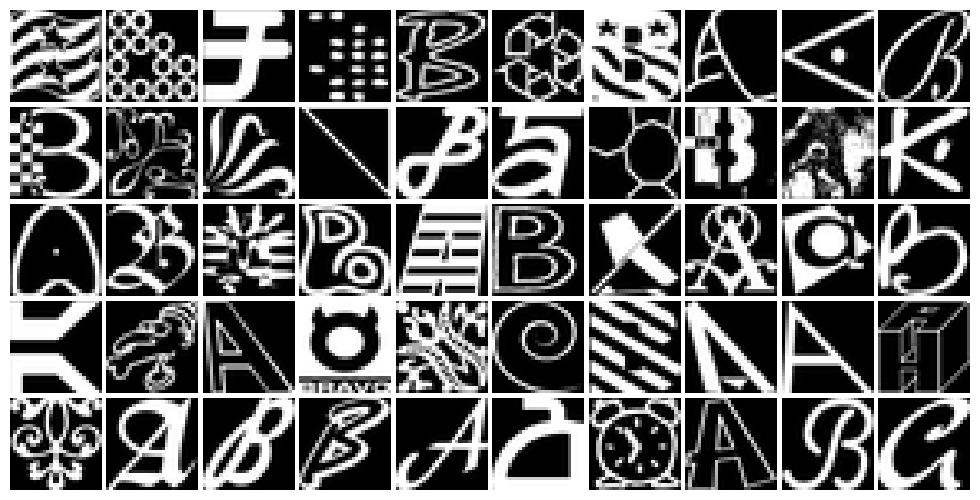}
    }
    \subfigure[\texttt{NOTMNIST} Low   (A vs B) Nonlinear Leverage Scores]{
        \includegraphics[width=0.48\linewidth]{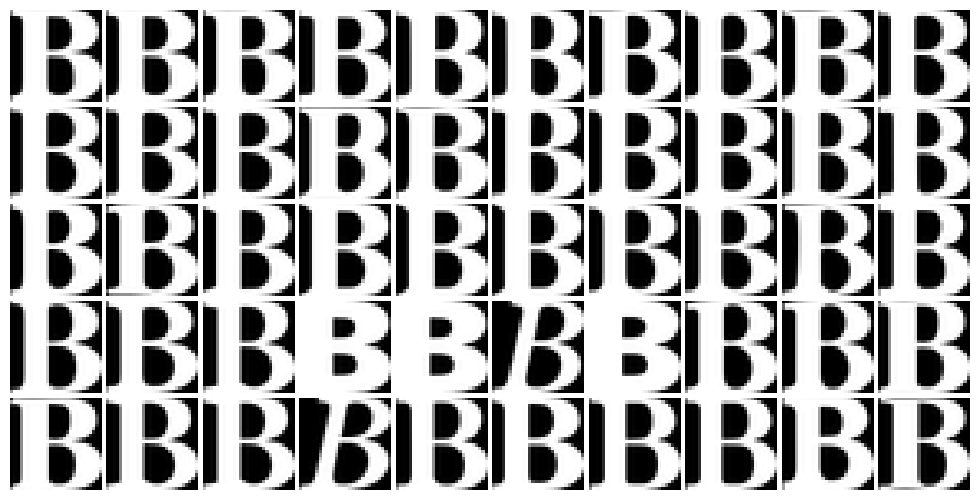}
    }
    \subfigure[\texttt{NOTMNIST} High (A vs B) Linear Leverage Scores]{
        \includegraphics[width=0.48\linewidth]{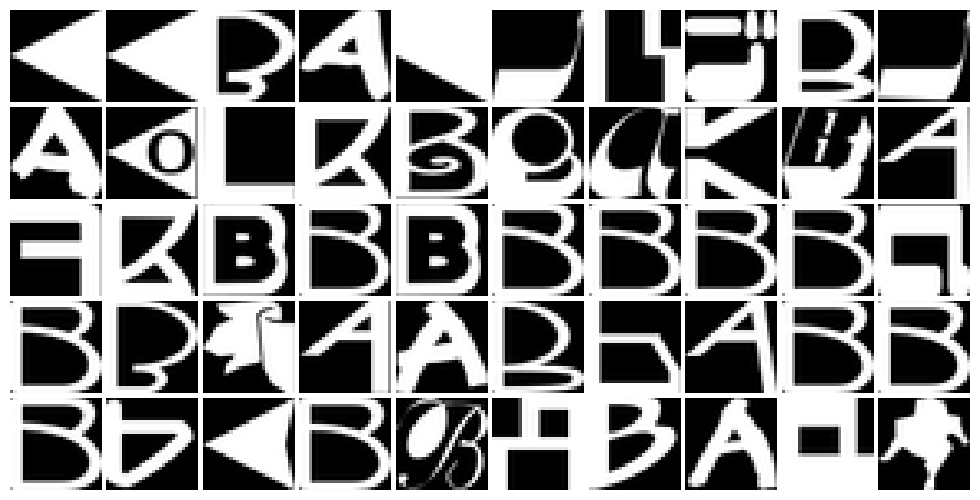}
    }
    \subfigure[\texttt{NOTMNIST} Low  (A vs B) Linear Leverage Scores]{
        \includegraphics[width=0.48\linewidth]{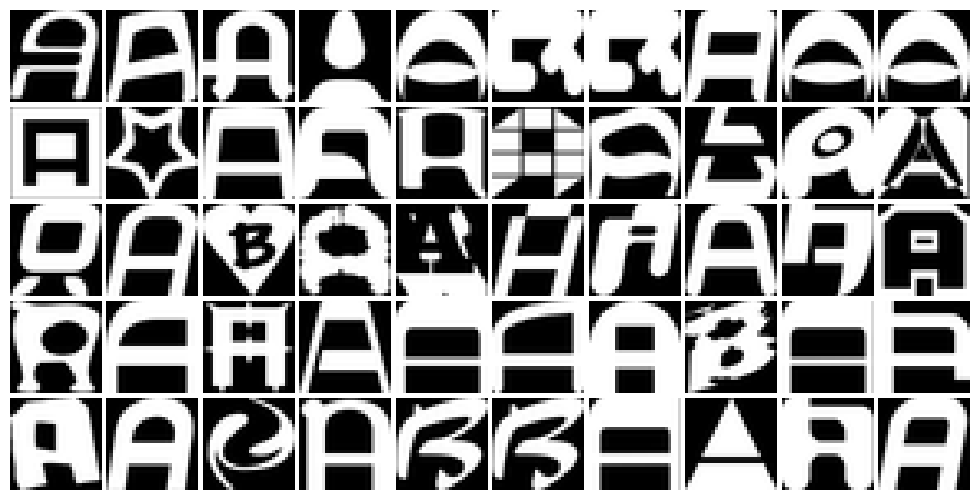}
    }
\end{figure*}
\begin{figure*}
     \subfigure[\texttt{NOTMNIST} High (B vs D) Nonlinear Leverage Scores]{
        \includegraphics[width=0.48\linewidth]{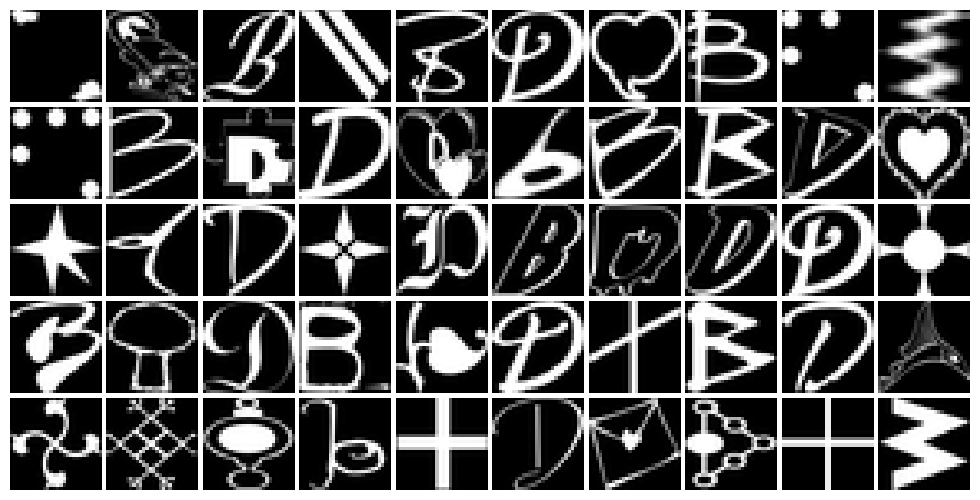}
    }
    \subfigure[\texttt{NOTMNIST} Low   (B vs D) Nonlinear Leverage Scores]{
        \includegraphics[width=0.48\linewidth]{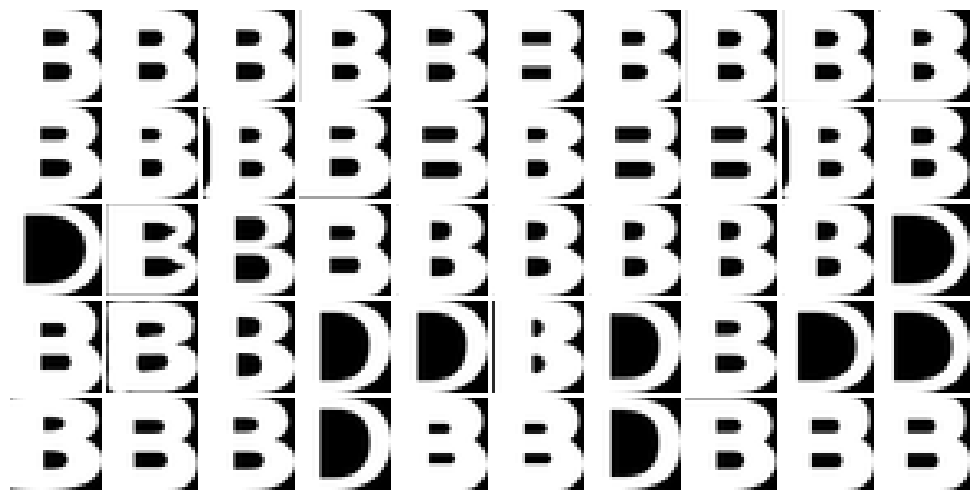}
    }
    \subfigure[\texttt{NOTMNIST} High (B vs D) Linear Leverage Scores]{
        \includegraphics[width=0.48\linewidth]{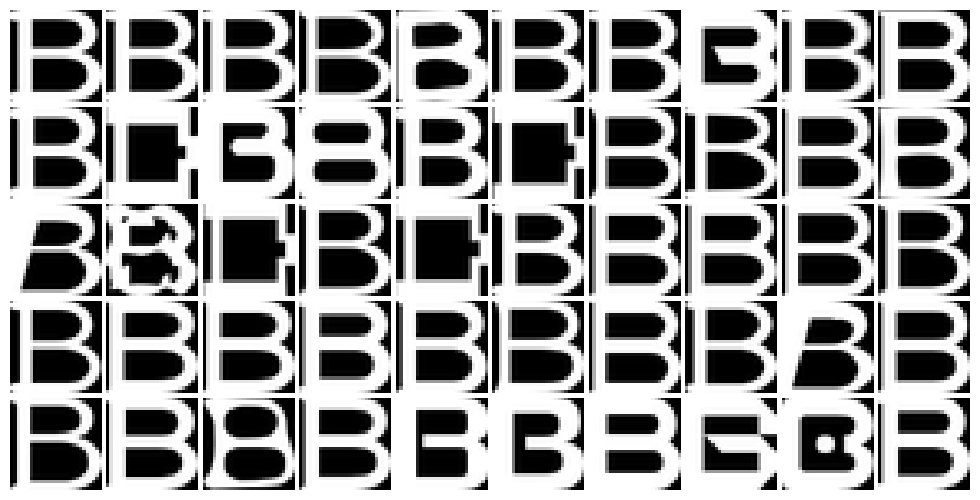}
    }
    \subfigure[\texttt{NOTMNIST} Low  (B vs D) Linear Leverage Scores]{
        \includegraphics[width=0.48\linewidth]{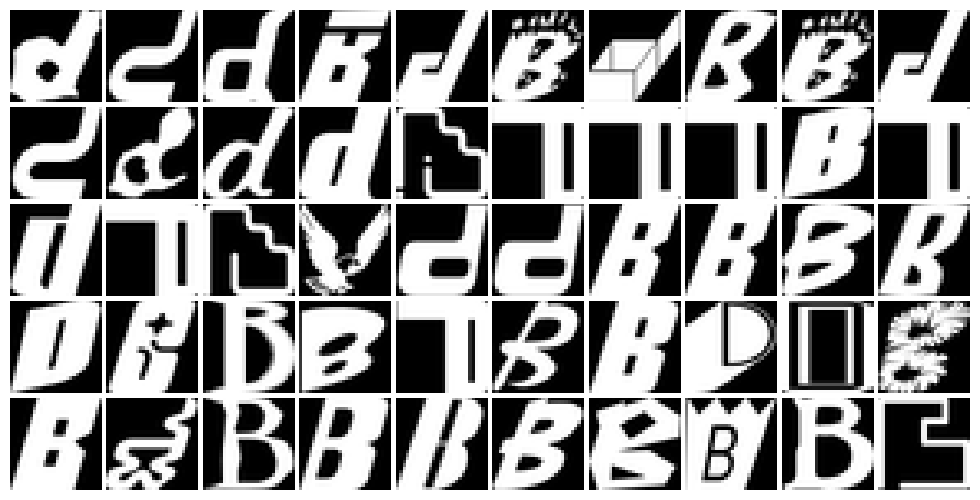}
    }
    \caption{Top 50 images with the highest and lowest nonlinear leverage scores in each grouping for the \texttt{NOTMNIST} dataset.}
    \label{fig:all_datasets_50_NOTMNIST}
\end{figure*}
\begin{figure*}
     \subfigure[\texttt{QD} High Nonlinear Leverage Scores]{
        \includegraphics[width=0.48\linewidth]{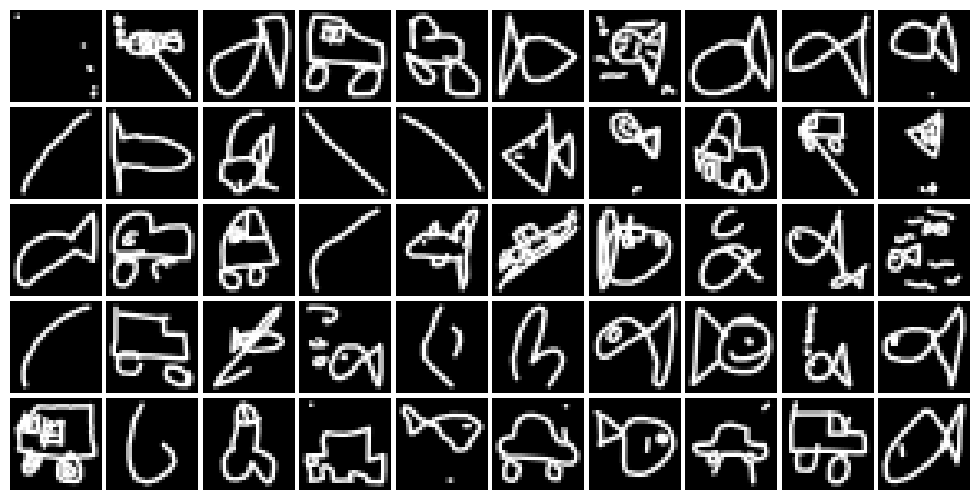}
    }
    \subfigure[\texttt{QD} Low   Nonlinear Leverage Scores]{
        \includegraphics[width=0.48\linewidth]{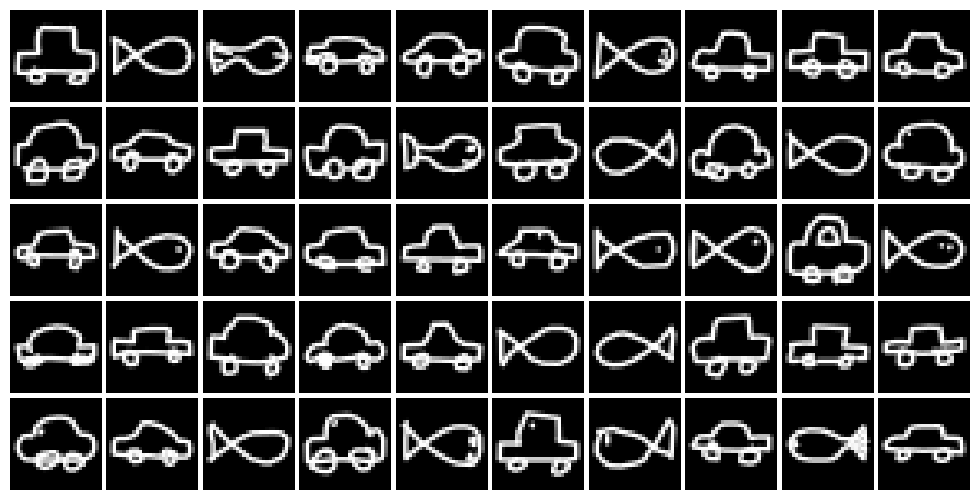}
    }
    \subfigure[\texttt{QD} High Linear Leverage Scores]{
        \includegraphics[width=0.48\linewidth]{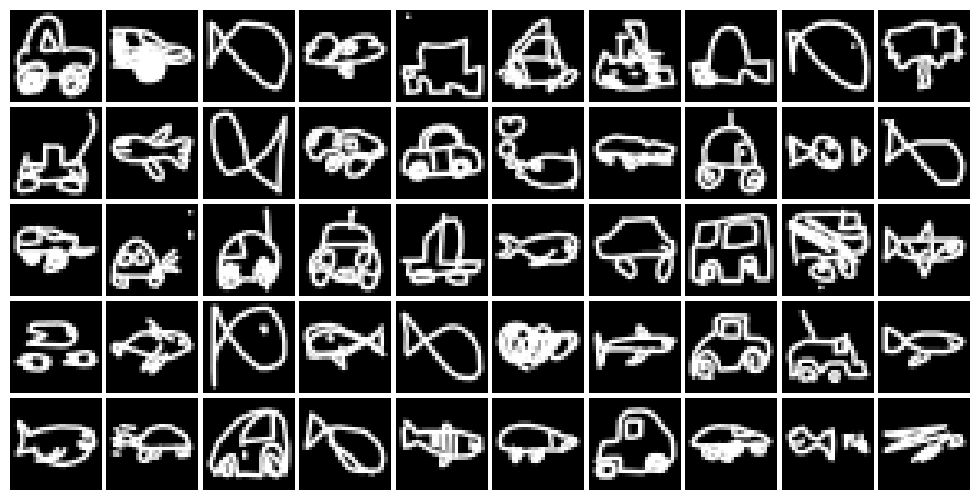}
    }
    \subfigure[\texttt{QD} Low  Linear Leverage Scores]{
        \includegraphics[width=0.48\linewidth]{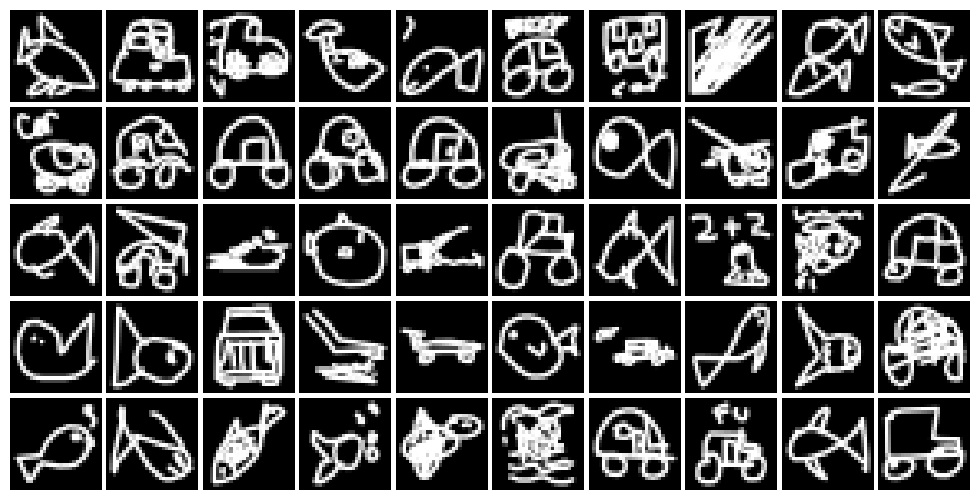}
    }
    \caption{Top 50 images with the highest and lowest nonlinear leverage scores for the \texttt{QD} dataset.}
    \label{fig:all_datasets_50_QD}
\end{figure*}
\begin{figure*}
    \subfigure[\texttt{FER - 75\%} High Nonlinear Leverage Scores]{
        \includegraphics[width=0.48\linewidth]{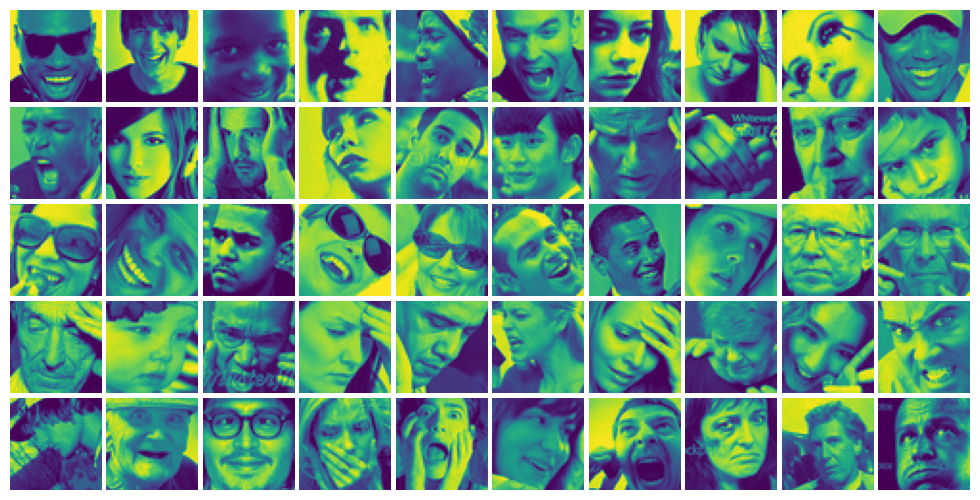}
    }
    \subfigure[\texttt{FER - 75\%} Low  Nonlinear Leverage Scores]{
        \includegraphics[width=0.48\linewidth]{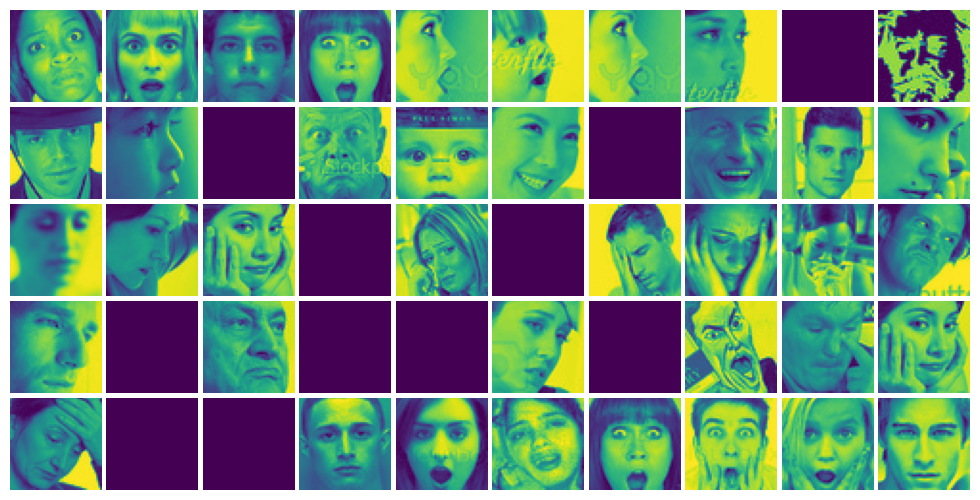}
    }
    \subfigure[\texttt{FER - 65\%} High Nonlinear Leverage Scores]{
        \includegraphics[width=0.48\linewidth]{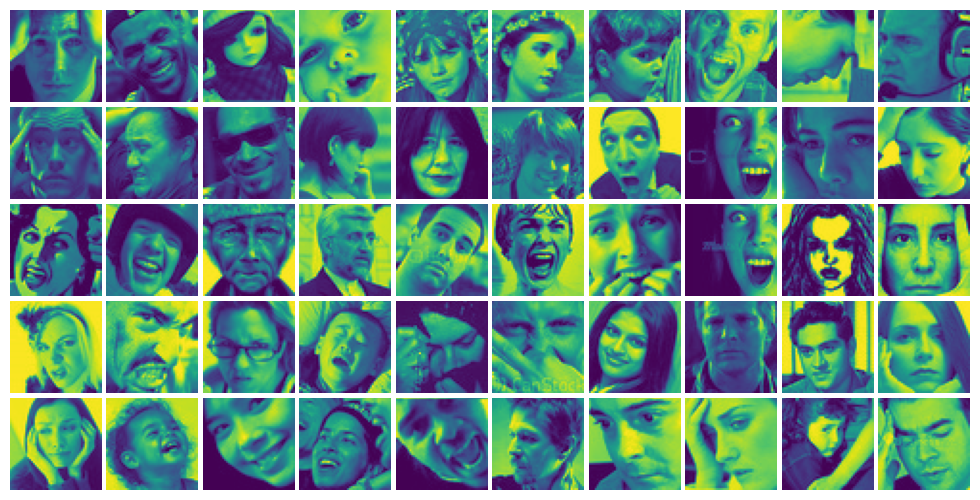}
    }
    \subfigure[\texttt{FER - 65\%} Low  Nonlinear Leverage Scores]{
        \includegraphics[width=0.48\linewidth]{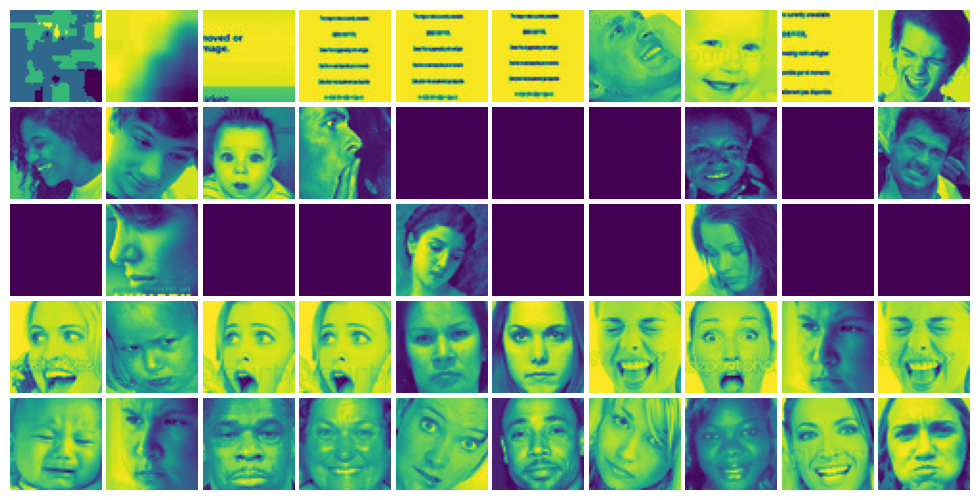}
    }
    \centering
    \subfigure[\texttt{FER - 55\%} Nonlinear Leverage Scores.]{
        \includegraphics[width=0.25\linewidth]{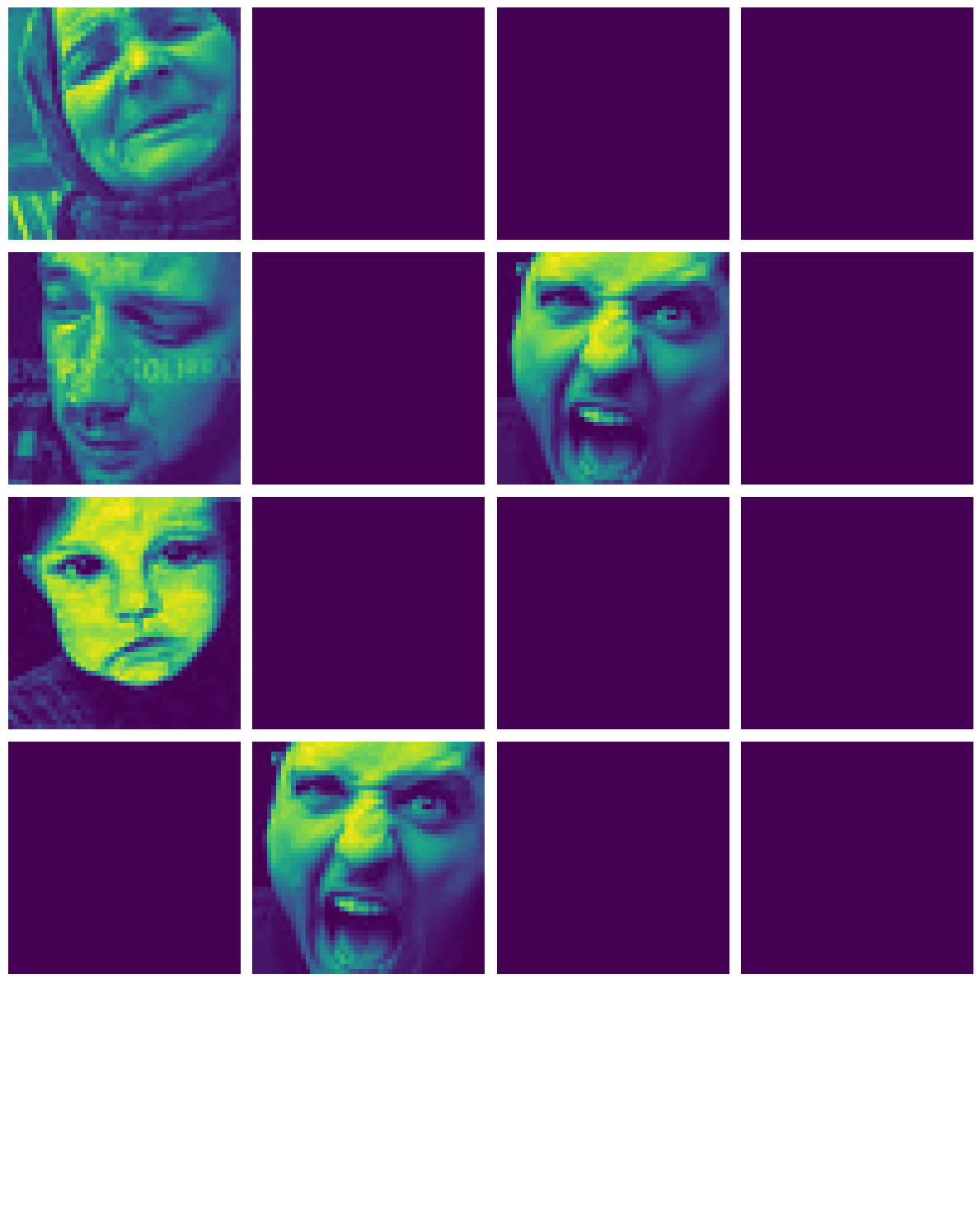}
    }

    \subfigure[\texttt{FER} High Linear Leverage Scores]{
        \includegraphics[width=0.47\linewidth]{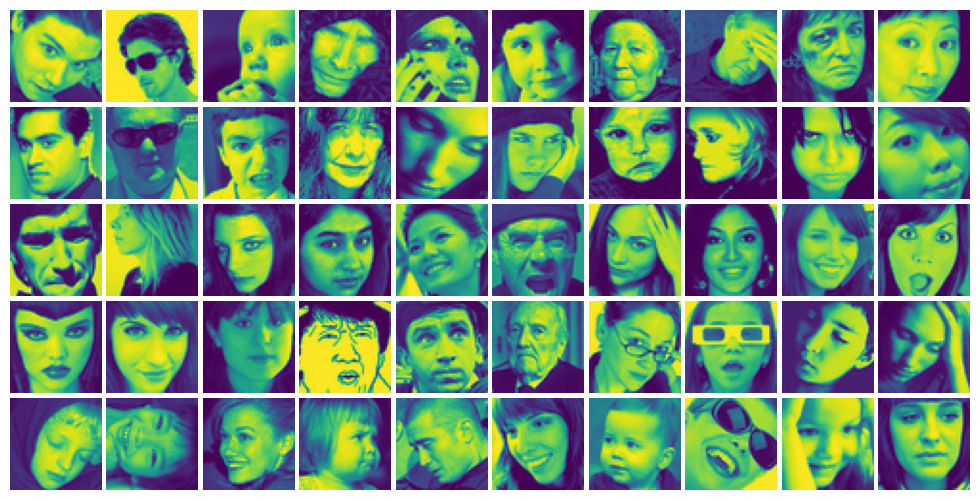}
    }
    \subfigure[\texttt{FER} Low  Linear Leverage Scores]{
        \includegraphics[width=0.47\linewidth]{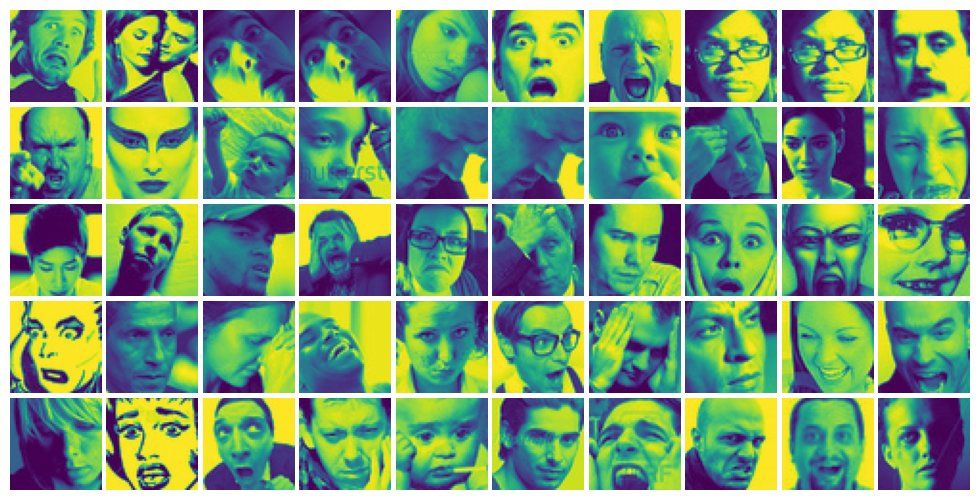}
    }
    \vspace{-0.2cm}
    \caption{Top 50 (trained and linear settings) and 16 (under-trained setting) images with the highest and lowest nonlinear leverage scores for valence expression on the \texttt{FER} dataset. In (e) only 16 images had non-zero scores at initialization.}
    \label{fig:all_datasets_50_FER}
\end{figure*}

\end{document}